\documentclass{article}

\usepackage{microtype}
\usepackage{graphicx}
\usepackage{subcaption}
\usepackage{booktabs} % for professional tables
\usepackage{mathtools}
\usepackage{appendix}
\usepackage{wrapfig}
\usepackage{xr-hyper}

\usepackage{bbm, amsthm, amsmath,amsfonts, amssymb}
\newtheorem{theorem}{Theorem}
\newtheorem{prop}{Proposition}
\newtheorem{corollary}{Corollary}
\newtheorem{lemma}[theorem]{Lemma}
\newtheorem{definition}{Definition}

\usepackage {graphicx, epigraph}
\usepackage{comment}
\usepackage{subcaption}
\usepackage{booktabs}
\usepackage{amsfonts, bm, multicol}
 \usepackage[ruled,vlined]{algorithm2e}
 \usepackage[super]{nth}
\DeclareMathOperator{\distance}{distance}
\DeclareMathOperator{\spn}{span}
  
\usepackage{color}
\def\isaccepted{1}
\DeclareMathOperator{\arctanh}{arctanh}

\usepackage{hyperref}

\usepackage{icml2020}

\icmltitlerunning{Explicit Gradient Learning}

\begin{document}

\twocolumn[
\icmltitle{Explicit Gradient Learning for Black-Box Optimization}




\begin{icmlauthorlist}
\icmlauthor{Mor Sinay*}{ed}
\icmlauthor{Elad Sarafian*}{ed}
\icmlauthor{Yoram Louzoun}{ed}
\icmlauthor{Noa Agmon}{ed}
\icmlauthor{Sarit Kraus}{ed}
\end{icmlauthorlist}

\icmlaffiliation{ed}{Department of Computer Science, Bar-Ilan University, Israel}

\icmlcorrespondingauthor{Mor Sinay}{ mor.sinay@gmail.com}
\icmlkeywords{Machine Learning, ICML}

\vskip 0.3in
]



\printAffiliationsAndNotice{\icmlEqualContribution}

\begin{abstract}

Black-Box Optimization (BBO) methods can find optimal policies for systems that interact with complex environments with no analytical representation. As such, they are of interest in many Artificial Intelligence (AI) domains. Yet classical BBO methods fall short in high-dimensional non-convex problems. They are thus often overlooked in real-world AI tasks. Here we present a BBO method, termed Explicit Gradient Learning (EGL), that is designed to optimize high-dimensional ill-behaved functions. We derive EGL by finding weak-spots in methods that fit the objective function with a parametric Neural Network (NN) model and obtain the gradient signal by calculating the parametric gradient. Instead of fitting the function, EGL trains a NN to estimate the objective gradient directly.
We prove the convergence of EGL in convex optimization and its robustness in the optimization of integrable functions. We evaluate EGL and achieve state-of-the-art results in two challenging problems: (1) the COCO test suite against an assortment of standard BBO methods; and (2) in a high-dimensional non-convex image generation task.

\end{abstract}

\section{Introduction}
Optimization problems are prevalent in many artificial intelligence applications, from search-and-rescue optimal deployment \cite{zhen2014simulation}, 
to triage policy in emergency rooms \cite{rosemarin2019emergency} to hyperparameter tuning in machine learning \cite{bardenet2013collaborative}. In these tasks, the objective is to find a policy that minimizes a cost or alternatively maximizes a reward.
Evaluating the cost of a single policy is a complex and often costly process that usually has no analytical representation, e.g., due to interaction with real-world physics or numerical simulation. {\em Black-Box Optimization} (BBO) algorithms \cite{audet2017derivative,golovin2017google} are designed to solve such problems, when the analytical formulation is missing, by repeatedly querying the Black-Box function and searching for an optimal solution while minimizing the number of queries (budget).

\par{\bf Related Work} BBO algorithms have been studied in multiple fields with diverse approaches. For example, {\em derivative–free} methods \cite{rios2013derivative}, from the classic Nelder–Mead algorithm \cite{nelder1965simplex} and Powell's method \cite{powell1964efficient} to more recent evolutionary algorithms such as CMA-ES \cite{hansen2006cma}. Another line of research is {\em derivative-based} algorithms, which first approximate the gradient and then apply line-search methods such as the Conjugate Gradient (CG) Method \cite{shewchuk1994introduction} and Quasi-Newton Methods, e.g. BFGS \cite{nocedal2006numerical}. Other model-based methods such as SLSQP \cite{bonnans2006numerical} and COBYLA \cite{powell2007view} iteratively solve quadratic or linear approximations of the objective function. Some variants apply {\em trust-region} methods \cite{conn2009introduction} and iteratively find an optimum within a trusted subset of the domain \cite{audet2017derivative}. Another line of research is more focused on stochastic discrete problems, e.g. Bayesian approaches \cite{snoek2015scalable}, and multi-armed bandit problems \cite{li2016hyperband}.

\par{\bf Our contribution} In this paper, we suggest a new derivative-based algorithm, {\em Explicit Gradient Learning} (EGL), that learns a surrogate function for the gradient by averaging the numerical directional derivatives over small volumes bounded by $\varepsilon$ radius. We control the accuracy of our model by controlling the $\varepsilon$ radius parameter. We then use trust-regions and dynamic scaling of the objective function to fine-tune the convergence process to the optimal solution. This results in a theoretically guaranteed convergence of EGL to a local minimum for convex problems. We compared the performance of EGL to eight different BBO algorithms in rigorous simulations in the COCO test suite \cite{nikolaus_hansen_2019_2594848} and show that EGL outperforms all others in terms of final accuracy. EGL was further evaluated in a high-dimensional non-convex domain, involving searching for an image with specific characteristics on a pre-trained Generative Adversarial Network (GAN) \cite{pan2019recent}. EGL again outperformed existing algorithms in terms of accuracy and appearance, where other BBO methods failed to converge to a solution in reasonable time.

The paper is organized as follows: BBO background and motivation for the EGL algorithm are presented in Sections \ref{sec:Background} and \ref{sec:Motivation}. The detailed description of the EGL algorithm and its theoretical analysis are shown in Section \ref{sec:Design_Analysis}. The empirical evaluation of EGL's performance and comparison to state-of-the-art BBO algorithms are described in Section \ref{sec:Experiment}, and we conclude in Section \ref{sec:Conclusions}. The proofs for all of the theoretical statements are found in the Appendix.

\section{Background}
\label{sec:Background}

A function $f:\Omega\to\mathbb{R}$, $\Omega \subseteq \mathbb{R}^n$ is considered a Black-Box if one can evaluate $y = f(x)$ at $x\in\Omega$, but has no prior knowledge of its analytical form. A Black-Box Optimization seeks to find $x^{*}=\arg\min_{x\in\Omega}f(x)$, typically, with as few evaluations as possible \cite{audet2017derivative}. Since, in general, it is not possible to converge to the optimal value with a finite number of evaluations, we define a budget $C$ and seek to find as good a solution as possible $x^{\star}$ with less than $C$ evaluations \cite{hansen2010comparing}.

Many BBO methods operate with a two-phase iterative algorithm: (1) search or collect data with some heuristic; and (2) update a model to obtain a new candidate solution and improve the heuristic. Traditionally, BBO methods are divided into derivative-free and derivative-based methods. The first group relies on statistical models \cite{balandat2019botorch}, physical models \cite{van1987simulated}, or Evolutionary Strategies \cite{back1996evolutionary} to define the search pattern and are not restricted to continuous domains, so they can also be applied to discrete variables. Our algorithm, EGL, falls under the category of the latter group, which relies on a gradient estimation to determine the search direction \cite{bertsekas2015convex}. Formally, derivative-based methods are only restricted to differentiable functions, but here we show that EGL can be applied successfully whenever the objective function is merely locally integrable.

Recently, with the Reinforcement Learning renaissance \cite{silver2017mastering,schulman2015trust}, new algorithms have been suggested for the problem of continuous action control (e.g., robotics control). Many of these can be viewed in the context of BBO as derivative-based methods applied with NN parametric models. One of the most prominent algorithms is DDPG \cite{lillicrap2015continuous}. DDPG iteratively collects data with some (often naive) exploration strategy and then fits a local NN parametric model $f_{\theta}$ around a candidate solution $x_k$. To update the candidate, it approximates the gradient $\nabla f$ at $x_k$ with the parametric gradient $\nabla f_{\theta}$ and then applies a gradient descent step \cite{ruder2016overview}.\footnote{Since $f_{\theta}$ is a differentiable analytical function, it can be differentiated with respect to its input $x$, much like what is done in adversarial examples \cite{yuan2019adversarial}.} Algorithm \ref{alg:igo} outlines the DDPG steps in the BBO formulation. We denote it as Indirect Gradient Learning (IGL) since it does not directly learn the gradient $\nabla f$. In the next section, we will develop arguments as to why one should learn the gradient explicitly instead of using the parametric gradient.

\begin{algorithm}[]
\small{
\DontPrintSemicolon
\KwIn{$x_0$, $\alpha$, $C$}
$k=0$ \\
\While{budget $C > 0$}{
\SetKwProg{Fn}{Build Local Model:}{}{end}
\Fn{}{
\ \ Collect data $\mathcal{D}_k = \{(x_k + \varepsilon n_i, y_i)\}_{i=1}^m, \ n_i \sim \mathcal{N}(0,\mathbf{I})$ \\
Fit a model $f_{\theta_k}$ with\\
$\theta_k = \arg\min_{\theta}\sum_{i=1}^m |f_{\theta}(x_k + \varepsilon_{i}) - y_i|^2$\\
}
\SetKwProg{Fn}{Gradient Descent:}{}{end}
\Fn{}{
\ \ $x_{k+1} \leftarrow x_k - \alpha \nabla f_{\theta_k}(x_k)$ \\
$k \leftarrow k + 1$ \\
}
}
}
\KwRet{$x_k$}
\caption{Indirect Gradient Learning}
\label{alg:igo}
\end{algorithm}

\section{Motivation}
\label{sec:Motivation}

In Algorithm \ref{alg:igo} only the gradient information $\nabla f(x_k)$ is required to update the next $x_k$ candidate. However, the gradient function is never learned directly and it is only inferred from the parametric model, without any clear guarantee of its veracity. Hence we seek a method that learns the gradient function $\nabla f$ explicitly. Clearly, directly learning the gradient is infeasible since the Black-Box only outputs the $f(x)$ values. Instead, our approach would be to learn a surrogate function for $\nabla f$, termed the \textit{mean-gradient}, by sampling pairs of observations $\{(x_i, y_i), (x_j, y_j)\}_{i,j}$ and averaging the numerical directional derivatives over small volumes. This section formally defines the mean-gradient and then formulates two arguments that motivate its use, instead of $\nabla f_{\theta}$. We then illustrate these arguments using 1D and 2D examples taken from the COCO test suite. 

\begin{figure*}[ht!]
  \centering
  \includegraphics[width=1\linewidth]{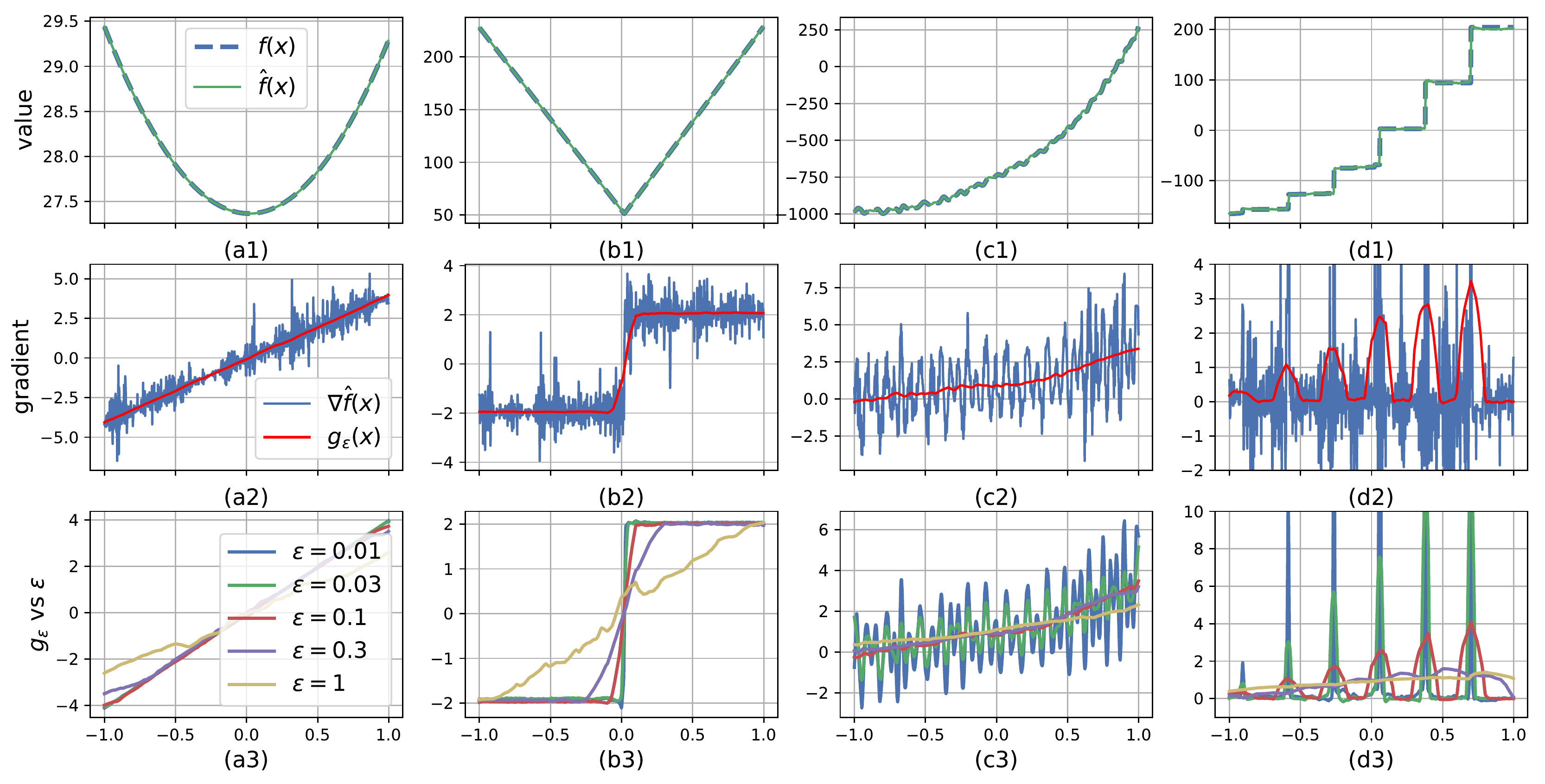}
  %\vspace{-20pt}
  \caption{Comparing indirect gradient learning and explicit gradient learning for 4 typical functions: (a) parabolic; (b) piecewise linear; (c) multiple local minima; (d) step function.}
\label{fig:egl_compare}
\end{figure*}

\begin{figure*}[ht!]
  \centering
  \includegraphics[width=1\linewidth]{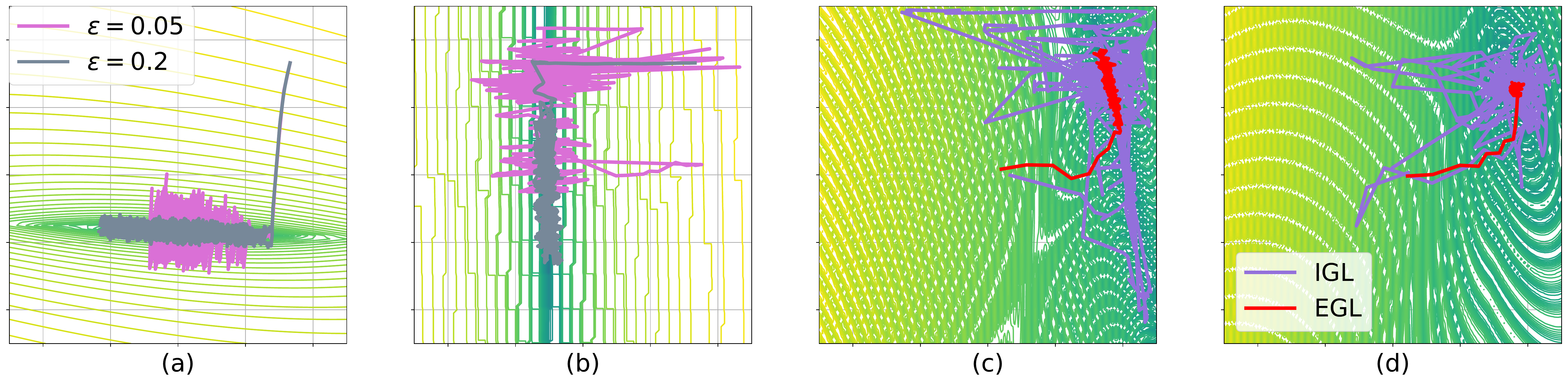}
  %\vspace{-20pt}
  \caption{Visualizing explicit gradient learning with different $\varepsilon$ for various 2D problems from COCO test suite: (a) sharp-ridge problem 194; (b) step-ellipsoid problem 97. Comparing EGL and IGL: (c) Schaffer F7 C1000 problem 255; (d) Schaffer F7 C10 problem 240.}
  
\label{fig:egl_2d_compare}
\end{figure*}

\subsection{The Mean-Gradient}
\label{mean_gradient_estimator}

For any differentiable function $f$ with a continuous gradient, the first order Taylor expression is
\begin{equation}
    f(x+\tau) = f(x) + \nabla f(x) \cdot \tau + O(\|\tau\|^2).
\end{equation}
Thus, locally around $x$, the directional derivative satisfies $\nabla f(x) \cdot \tau \approx f(x+\tau) - f(x)$. We define the mean-gradient as the function that minimizes the Mean-Square-Error (MSE) of such approximations in a vicinity of $x$.
\begin{definition}
\label{mg_definition}
The mean-gradient at $x$ with $\varepsilon > 0$ averaging radius is
\begin{equation}
\label{mean_gradient}
    g_{\varepsilon}(x) = \arg\min_{g\in\mathbb{R}^n} \int\displaylimits_{V_{\varepsilon}(x)}|g\cdot \tau - f(x+\tau) + f(x)|^2 d\tau
\end{equation}
where $V_{\varepsilon}(x)\subset\mathbb{R}^n$ is a convex subset s.t. $\|x' - x\|\leq \varepsilon$ for all $x'\in V_{\varepsilon}(x)$ and the integral domain is over $\tau$ s.t. $x+\tau\in V_{\varepsilon}(x)$.
\end{definition}

\begin{prop}[controllable accuracy]
For any differentiable function $f$ with a continuous gradient, there is $\kappa_g > 0$, so that for any $\varepsilon > 0$ the mean-gradient satisfies $\|g_{\varepsilon}(x) -\nabla f(x)\| \leq \kappa_g \varepsilon$ for all $x\in\Omega$.   
\end{prop}

In other words, the mean-gradient has a controllable accuracy parameter $\varepsilon$ s.t. reducing $\varepsilon$ improves the gradient approximation. As explained in Sec. \ref{Convergence}, this property is crucial in obtaining the convergence of EGL in convex problems. Unlike the mean-gradient, the parametric gradient has no such parameter and the gradient accuracy is not directly controlled. Even for zero MSE error s.t. $f_{\theta}(x_i)\equiv y_i$ for all of the samples in the replay buffer \cite{mnih2015human}, there is no guarantee of the parametric gradient accuracy. On the contrary, overfitting may severely hurt this approximation. 

Moreover, the parametric gradient can be discontinuous even when the parametric model has a Lipschitz continuous \cite{hansen1995lipschitz} gradient. For example, a commonly used NN with ReLU activation is only piecewise differentiable. This leads to very erratic gradients even for smooth objective functions. If the objective function is not smooth (e.g. for noise-like functions or in singular points), the gradient noise is exacerbated. On the other hand, the next proposition suggests that, due to the integral over $V_{\varepsilon}(x)$ that smooths the gradient, the mean gradient is smooth whenever the objective function is continuous.

\begin{prop}[continuity]
If $f(x)$ is continuous in $V$ s.t. $V_{\varepsilon}(x)\subset V$ then the mean-gradient is a continuous function at $x$.
\end{prop}

The mean-gradient is not necessarily continuous in discontinuity points of $f$. One can obtain even a smoother surrogate for $\nabla f$ by slightly modifying the mean-gradient definition
\begin{equation*}
\label{perturbed_mean_gradient}
    g_{\varepsilon}^p(x) = \arg\min_{g\in\mathbb{R}^n} \iint\displaylimits_{V_{\varepsilon}(x)B_{p}(x)}|g\cdot (\tau - s) - f(\tau) + f(s)|^2  ds d\tau,
\end{equation*}

\vspace{-0.7em}

where the integral domains are $s\in B_{p}(x)$ and $\tau\in V_{\varepsilon}(x)$. Here, $B_{p}(x)\subset V_{\varepsilon}(x)$ is an $n$-ball perturbation set with $p < \varepsilon$ radius that dithers the reference point (which is fixed at $x$ in Definition \ref{mg_definition}). We term this modified version as the {\em perturbed mean-gradient}. Remarkably, while $g_{\varepsilon}^p$ is still a controllably accurate model for Lipschitz continuous gradients, as the integrand is $x$ independent, $g_{\varepsilon}^p$ is continuous whenever $f$ is merely integrable. In practice, as $g_{\varepsilon}$ is learnt with a Lipschitz continuous model, we find that both forms are continuous for integrable functions. Nevertheless, Seq. \ref{sec:Experiment} shows that $g_{\varepsilon}^p$ adds a small gain to the EGL performance. Next, we demonstrate how the EGL smoothness (as opposed to $\nabla f_{\theta}$) leads to more stable and efficient trajectories in both continuous and discontinuous objective functions.

\subsection{Illustrative Examples}

To demonstrate these properties of the mean-gradient, we consider 1D \footnote{Since COCO does not have built-in 1D problems; we generated 1D problems based on 2D problems with $f_{1D}(x):f_{2D}(x,x)$} and 2D problems from the COCO test suite. We start by examining 4 typical 1D functions: (a) parabolic; (b) piecewise linear; (c) multiple local minima; and (d) step function. We fit $f$ with a NN and compare its parametric gradient to the mean-gradient, learnt with another NN. The NN model is identical for both functions and is based on our Spline Embedding architecture (see description in Appendix Sec. \ref{sec:SplineEmb}). The results are presented in Fig. \ref{fig:egl_compare}. The \nth{1} row shows that the fit $\hat{f}$ is very strong s.t. the error is almost indistinguishable to the naked eye. Nevertheless, the calculated parametric gradient (\nth{2} row) is very noisy, even for smooth functions, as the NN architecture is piecewise linear. Occasionally, there are even spikes that change the gradient sign. Traversing the function manifold with such a function is very unstable and inefficient. On the other hand, due to the smoothing parameter $\varepsilon=0.1$ and since the NN is Lipschitz continuous, the mean-gradient is always smooth, even for singularity points and discontinuous gradients. 

In the \nth{3} row we evaluate $g_{\varepsilon}$ for different size $\varepsilon$ parameters. We see that by setting $\varepsilon$ sufficiently high, the gradient becomes smooth enough so there is no problem descending over steps and multiple local minima functions. In practice, we may use this property and start the descent trajectory with a high $\varepsilon$. This way, the optimization process does not commit too early to a local minimum and searches for regions with lower valleys. After refining $\varepsilon$ the process will settle in a local minimum, which in practice would be much lower than minima found around the initial point.

The next experiment  (Fig. \ref{fig:egl_2d_compare}) is executed on 2D problems from the COCO test suite. In Fig. \ref{fig:egl_2d_compare}(a-b) we present the gradient-descent steps with the mean-gradient and a constant $\varepsilon$. In \ref{fig:egl_2d_compare}(a) the objective has a singular minimum, similar to the $|x|$ function's minimum (Fig.\ref{fig:egl_compare}(a)). As $\varepsilon$ gets smaller, the gradients near the minimum get larger and there is a need to reduce the learning-rate ($\alpha$) in order to converge.  \ref{fig:egl_2d_compare}(b) presents a step function. Again, we observe that for a smaller $\varepsilon$ the gradient has high spikes in the discontinuity points and this leads to a noisier trajectory. For that purpose, the EGL algorithm decays both $\alpha$ and $\varepsilon$ during the learning process (see Seq. \ref{sec:Design_Analysis} for details). This lends much smoother trajectories, as can be seen in Fig. \ref{fig:egl_2d_compare}(c-d). Here we executed both EGL and IGL with the same $\alpha$, $\varepsilon$ decay pattern. We observe that the IGL trajectories are very noisy and inefficient since the parametric gradient always contains some noise. On the other hand, EGL smoothly travels through a ravine (Fig. \ref{fig:egl_2d_compare}(c)) and converges to a global minimum (Fig. \ref{fig:egl_2d_compare}(d)).

\section{Design \& Analysis}\label{sec:Design_Analysis}

In this section, we lay out the practical EGL algorithm and analyze its convergence properties.

\subsection{Monte-Carlo Approximation}
\label{mc_approx}

To learn the mean-gradient, one may evaluate the integral in Eq. (\ref{mean_gradient}) with  Monte-Carlo samples and learn a model that minimizes this term. Formally, for a model $g_{\theta}:\Omega\to\mathbb{R}^n$ and a dataset $\mathcal{D}_k=\{(x_i, y_i)\}_{i=1}^m$, define the loss function
\begin{equation}
\label{mc_objective}
    \mathcal{L}_{k,\varepsilon}(\theta) = \sum_{i=1}^m\sum_{x_j\in V_{\varepsilon}(x_i)} |(x_j - x_i) \cdot g_{\theta}(x_i) - y_j + y_i |^2
    % \mathcal{L}_{k,\varepsilon}(\theta) = \sum_{x_i,x_j\in \mathcal{D}_t}\sum_{x_j\in V_{\varepsilon}(x_i)} |(x_j - x_i) \cdot g_{\theta}(x_i) - y_j + y_i |^2
\end{equation}
and learn $\theta^{*}_k=\arg\min_{\theta}\mathcal{L}_{k,\varepsilon}(\theta)$, e.g. with gradient descent. This formulation can be used to estimate the mean-gradient for any $x$. Yet, practically, in each optimization iteration, we only care about estimating it in a close proximity to the current candidate solution $x_k$. Therefore, we assume that the dataset $\mathcal{D}_k$ holds samples only from $V_{\varepsilon}(x_k)$.

The accuracy of the learnt model $g_{\theta^*_k}$  heavily depends on the number and locations of the evaluation points and the specific parameterization for $g_{\theta}$. Still, for $f$ with a Lipschitz continuous gradient, i.e. $f\in\mathcal{C}^{+1}$, we can set bounds for the model accuracy in $V_{\varepsilon}(x_k)$ with respect to $\varepsilon$. For that purpose,  we require  a set of at least $m \geq n+1$ evaluation points in $\mathcal{D}_k$ which satisfy the following poised set definition.

\begin{definition}[poised set for regression]
Let $\mathcal{D}_k = \{(x_i, y_i)\}_{1}^{m}$, $m\geq n+1$ s.t. $x_i \in V_{\varepsilon}(x_k)$ for all $i$. Define the matrix $\tilde{X}_{i}\in\mathbb{M}^{m\times n}$ s.t. the $j$-th row is $x_i-x_j$. Now define $\tilde{X}=\begin{psmallmatrix}\tilde{X}_1^T & \cdots & \tilde{X}_m^T\end{psmallmatrix}^T$. The set $\mathcal{D}_k$ is a poised set for regression in $x_k$ if the matrix $\tilde{X}$ has rank $n$.
\end{definition}
Intuitively, a set is poised if its difference vectors $x_i-x_j$ span $\mathbb{R}^n$. For the poised set, and a constant parameterization, the solution of Eq. (\ref{mc_objective}) is unique and it is equal to the Least- Squares (LS) minimizer. If $f$ has a Lipschitz continuous gradient, then, with an admissible parametric model, the error between $g_{\theta}(x)$ and $\nabla f(x)$ can be proportional to $\varepsilon$. We formalize this argument in the following theorem.

\begin{theorem}
\label{th:regression_problem}
Let $\mathcal{D}_k$ be a poised set in $V_{\varepsilon}(x_k)$. The regression problem
\begin{equation}
    g^{MSE} = \arg\min_{g} \sum_{i,j \in \mathcal{D}_k} |(x_j - x_i) \cdot g - y_j + y_i |^2
\end{equation}
has the unique solution $g^{MSE}=(\tilde{X}^T \tilde{X})^{-1} \tilde{X}^T \delta$, where $\delta\in\mathbb{R}^{m^2}$ s.t. $\delta_{i \cdot (m-1) + j} = y_j - y_i$. Further, if $f\in\mathcal{C}^{1+}$ and $g_{\theta}\in\mathcal{C}^0$ is a parameterization with (equal or) lower regression loss than $g^{MSE}$, the following holds for all $x\in V_{\varepsilon}(x_k)$:
\begin{equation}
    \|\nabla f(x) - g_{\theta}(x)\| \leq \kappa_g \varepsilon
\end{equation}
\end{theorem}

\begin{corollary}
\label{lipschitz_nn}
For the $\mathcal{D}_k$ poised set, any Lipschitz continuous parameterization of the form $g_{\theta}(x)=F(Wx)+b$ is a controllably accurate model \cite{audet2017derivative} in $V_{\varepsilon}(x_k)$ for the optimal set of parameters $\theta^*_k$.  
\end{corollary}

This is obvious as we can simply set $W=0$ and $b=g^{MSE}$. In this work, we are interested in using NNs which are much stronger parameterizations than constant models. If the NN satisfies the Lipschitz continuity property (e.g. with spectral normalization) and has at least a biased output layer, then its optimal set of parameters $\theta^*_k$ has lower regression loss than $g^{MSE}$ and it is therefore a controllably accurate model.\footnote{Provided that the learning process recovered $\theta^*_k$, which is not necessarily true in practice.}

Finally, to incorporate learning of the perturbed mean gradient, we slightly modify Eq. (\ref{mc_objective}). Note that we cannot directly dither the reference point $x_i$ as this means that we need to collect more samples. Instead, we may dither the $g_{\theta}$ argument by evaluating it in $\bar{x}_i=x_i+n_i$ where $n_i$ is uniformly sampled in an $n$-ball with radius $p$. Thus, the perturbed loss function for small $p$ s.t. $p \ll \varepsilon$ is
\begin{equation}
    \mathcal{L}_{k,\varepsilon,p}(\theta) = \sum_{i,j \in \mathcal{D}_k} |(x_j - x_i) \cdot g_{\theta}(\bar{x}_i) - y_j + y_i |^2
\end{equation}

\subsection{Convergence of EGL}
\label{Convergence}

In this part we assume that the function $f$ is convex with a Lipschitz continuous gradient s.t. $\|\nabla f(x) - \nabla f(x')\| \leq \kappa_f \|x-x'\|$. The classical gradient descent theorem states that the update rule $x_{k+1}=x_{k} - \alpha \nabla f(x_k)$ with learning parameter $\alpha \leq \frac{1}{\kappa_f}$ yields a solution $x_k\to x^*$ for $k\to\infty$.  

In EGL, we descend over a surrogate function of the gradient that contains some amount of error. Far from $x^*$, where $\|\nabla f\|$ is large, the error in $g_{\varepsilon}$ is small enough s.t. every new candidate improves the solution, i.e. $f(x_{k+1}) \leq f(x_k)$. As $x_k$ gets closer to $x^*$, the gradient $\|\nabla f\|$ decreases (as $f$ is convex), and eventually the error, i.e. $g_{\varepsilon}(x_k) - \nabla f(x_k)$, becomes so significant that improvement is  no longer guaranteed. Nevertheless, our next theorem shows that for a fixed $\varepsilon$ after a finite number of descent steps, the magnitude of the gradient is in the order of $\|\nabla f (x_k)\| \propto \frac{\varepsilon}{\alpha}$. 

\begin{theorem}
Let $f:\Omega\to\mathbb{R}$ be a convex function with a Lipschitz continuous gradient and a Lipschitz constant $\kappa_f$. Suppose a controllable mean-gradient model $g_{\varepsilon}$ with error constant $\kappa_g$, the gradient descent iteration $x_{k+1} = x_{k} - \alpha g_{\varepsilon}(x_k)$ with a sufficiently small $\alpha$ s.t. $\alpha \leq \min(\frac{1}{\kappa_g}, \frac{1}{\kappa_f})$ guarantees:
\begin{enumerate}
    \item For $\varepsilon \leq \frac{\|\nabla f(x)\|}{5\alpha}$, monotonically decreasing steps s.t. $f(x_{k+1}) \leq f(x_k) -  2.25 \frac{\varepsilon^2}{\alpha}$.
    \item After a finite number of iterations, the descent process yields $x^{\star}$ s.t. $\|\nabla f(x^{\star})\| \leq \frac{5\varepsilon}{\alpha}$.
\end{enumerate}
\end{theorem}

Utilizing Theorem \ref{Convergence} we can design an algorithm that converges to a local minimum. For that purpose we must make sure that the learning rate abides the requirement $\alpha \leq \min(\frac{1}{\kappa_g}, \frac{1}{\kappa_f})$. Since this factor cannot be easily estimated, a practical solution is to decay $\alpha$ during the optimization process. However, to converge, the ratio $\frac{\varepsilon}{\alpha}$ must also decay to zero. This means that $\varepsilon$ must decay to zero faster than $\alpha$. Algorithm \ref{alg:convergent_egl} provides both of these conditions, so it is guaranteed to converge to a local minimum for $\bar{\varepsilon}\to 0$. 

\begin{algorithm}[]
\small{
\DontPrintSemicolon
\KwIn{$x_0$, $\alpha$, $\varepsilon$, $\gamma_{\alpha} < 1$, $\gamma_{\varepsilon} < 1$, $\bar{\varepsilon}$}
% \SetAlgoLined
$k=0$ \\
\While{$\varepsilon \leq \bar{\varepsilon}$}{

\SetKwProg{Fn}{Build Model:}{}{end}
\Fn{}{
\ \ Collect data $\{(x_i, y_i)\}_{1}^m$, $x_i\in V_{\varepsilon}(x_k)$ \\
Learn a local model $g_{\varepsilon}(x_k)$\\
}

\SetKwProg{Fn}{Gradient Descent:}{}{end}
\Fn{}{
\ \ $x_{k+1} \leftarrow x_k - \alpha g_{\varepsilon}(x_k)$ \\
\If{$f(x_{k+1}) > f(x_k) - 2.25 \frac{\varepsilon^2}{\alpha}$}{
\ \ $\alpha \leftarrow \gamma_{\alpha}\alpha$ \\
$\varepsilon \leftarrow \gamma_{\alpha}\gamma_{\varepsilon}\varepsilon$ \\
}
}
$k \leftarrow k + 1$ \\
}
}
\KwRet{$x_k$}
\caption{Convergent EGL}
\label{alg:convergent_egl}
\end{algorithm}

This analysis assumes a convex problem, but since EGL is also designed for non-convex optimization, our practical algorithm alleviates the requirement for strictly monotonically decreasing steps. Instead, it calculates a running mean over the last candidates in order to determine when to decrease $\alpha$ and $\varepsilon$. The complete practical EGL algorithm is found in the Appendix Sec. \ref{sec:egl_alg}. It also includes input and output mapping and scaling as described in the next section.

\subsection{Dynamic Mappings}\label{subsec:DynamicScaling}

Unlike supervised learning with a constant dataset, Black-Box optimization is a dynamic problem. The input and output statistics change over time as the optimization progresses. There are several sources for this drift. First, traversing via $g_{\varepsilon}$ changes the input's first moment by updating the center of the samples $x_i$ and the output's first moment by collecting smaller costs $y_i$. At the same time, squeezing $\varepsilon$ and $\alpha$ over time decreases the second moment statistics by reducing the variety in $\{(x_i, y_i)\}$. NNs are sensitive to such distribution changes \cite{brownlee2018better} and require tweaking hyperparameters such as learning rate and initial weight distribution to maintain high performance. Default numbers (e.g. learning rate of $10^{-3}$) usually work best when the input data is normalized. To regulate the statistics over the entire optimization process, we apply a method of double dynamic mappings for both input and output values. 

From the input perspective, our dynamic mapping resembles {\it Trust-Region} methods \cite{nocedal2006numerical}. Instead of searching for $x^*$ in the entire $\Omega$ domain by reducing $\varepsilon$ and $\alpha$ over time, we fix $\varepsilon$ and $\alpha$ and search for $x^{\star}_j$ in a sub-region $\Omega_j$. After finding the best candidate solution in this sub-region we shrink $\Omega_j$ by a factor of $\gamma_{\alpha}>0$ and $\varepsilon$ by a factor of $\gamma_{\varepsilon} > 0$. To keep the input statistics regulated, we maintain a bijective mapping $h_j:\Omega_j\to\mathbb{R}^n:$ that scales the $x$ values to an unconstrained domain $\tilde{x}$ with approximately constant first and second moments. 

In this work, $\Omega$ is assumed to be a rectangular box that denotes the upper and lower bounds for each entry in $x$. Thus we consider element-wise bijective mappings of the form $h_j(x) = (h_j^1(x^1),...,h_j^n(x^n))$ and each new sub-region, $\Omega_{j+1} \subset \Omega_j$, is a smaller rectangular box centered at $x^{\star}_j$. In the unconstrained domain $\tilde{x}$, we can use the original learning rate, yet, due to the squeezing factor, the effective learning rate is decayed by the $\gamma_{\alpha}$ factor. Equivalently, the effective accuracy parameter $\varepsilon$ is reduced by a factor of $\gamma_{\alpha}\times \gamma_{\varepsilon}$. In other words, we use the same NN parameters to learn a zoomed-in problem of the original objective function.

In order to regulate the output statistics, we define a scalar, monotonically increasing, invertible mapping $r_k(y)$ which maps the $y_i$ samples in the dataset $\mathcal{D}_k$ to approximately constant statistics. Since traversing with $g_{\varepsilon}$ can significantly change the statistics even in the same sub-region, $r_k$ must be dynamic and cannot be held fixed for the entire $j$ sub-problem. Combining both input and output mappings, we obtain a modified set of samples $\{(\tilde{x}_i, \tilde{y}_i)\}_i = \{(h_j(x_i), r_k(y_i))\}_i$, so effectively we learn a modified mean-gradient, denoted as $\tilde{g}_{\varepsilon_{jk}}$, of a compressed and scaled function $\tilde{y} = \tilde{f}_{jk}(\tilde{x})=r_k \circ f \circ h_j^{-1} (\tilde{x})$.

If both input and output mappings are linear, then the true mean-gradient is proportional to the modified mean-gradient

\begin{equation}
\label{lin_map}
    g_{\varepsilon}(x) = \left(\frac{\partial r_k}{\partial y}\right)^{ -1} \nabla h_j(x) \odot \tilde{g}_{\varepsilon_{jk}}(\tilde{x})
\end{equation}
Even when the mappings are approximately linear inside $V_{\varepsilon_{jk}}(\tilde{x}_k)$, $g_{\varepsilon}(x_k)$ can be estimated according to Eq. (\ref{lin_map}). Hence, to maintain a controllably accurate model, we seek mappings that preserve the linearity as much as possible. On the other hand, strictly linear mappings may be insufficient since they are very sensitive to outliers. After experimenting with several functions, we found a sweet spot: a composition of a linear mapping followed by a squash function that compresses only the outliers (see details in Appendix Sec. \ref{sec:Mappings}).  With such mappings, we make sure that both the model is controllably accurate and the learning hyperparameters are adequate for the entire optimization process.

\section{Empirical Evaluation}\label{sec:Experiment}
In this section we describe the rigorous empirical analysis of EGL against various BBO methods in the COCO test suite and in a search task over the latent space of a GAN model. {\em The code is available in the supplementary material.} % Git link will be added upon acceptance.)}

\begin{figure*}[ht!]
  \centering
  \includegraphics[width=1\linewidth]{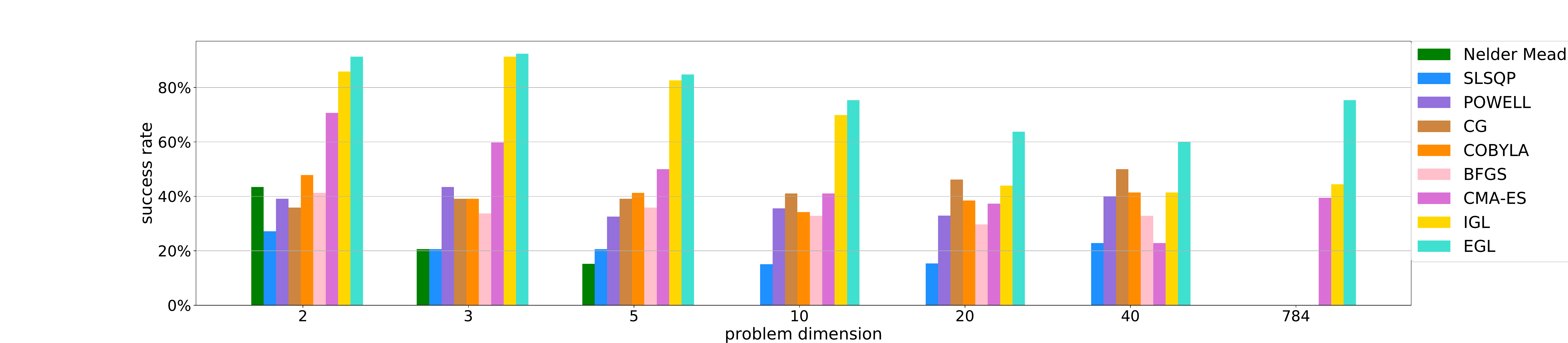}
  %\vspace{-20pt}
  \caption{Comparing the success rate of EGL and other BBO algorithms for a budget $C=150\cdot10^3$.}
\label{fig:egl_baseline_cmp_success}
\vspace{-0.5em}
\end{figure*}

\begin{figure*}[ht!]
  \centering
  \includegraphics[width=1\linewidth]{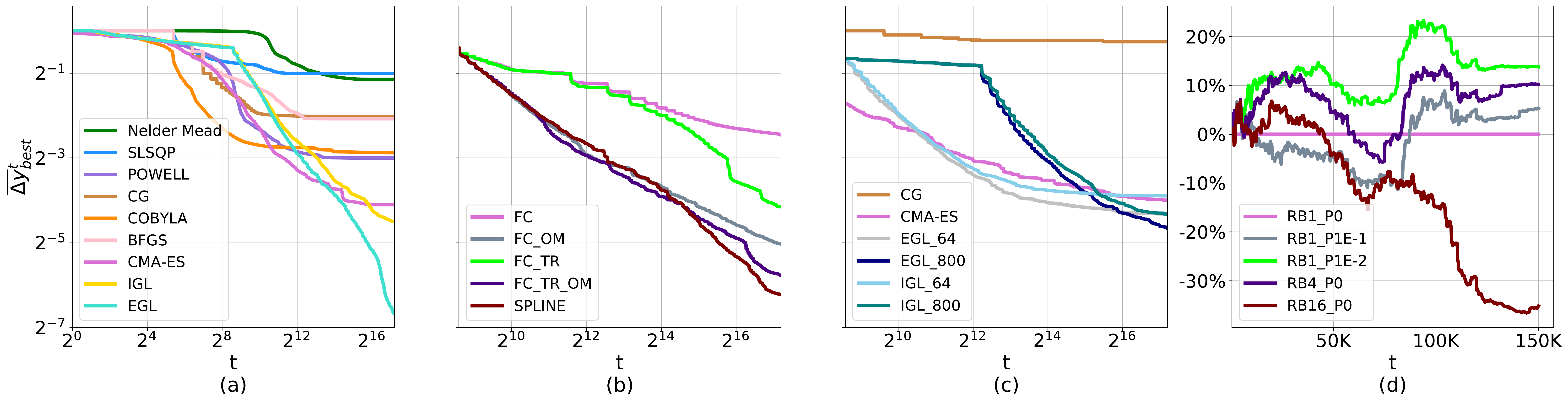}
  %\vspace{-20pt}
  \caption{The scaled distance $\Delta y_{best}^t$ as a function of $t\in[1,..,C]$ for: (a) EGL and baselines on 40D, (b) trust-region and output mapping ablation test, (c) EGL with different $m$ samples and baselines on 784D, (d) the perturbed mean-gradient and long replay buffer on 40D.}
\label{fig:egl_dim_avg}
\vspace{-0.5em}
\end{figure*}

\subsection{The COCO test suite}

We tested EGL on the COCO test suite, a platform for systematic comparison of real-parameter global optimizers. COCO provides Black-Box functions in several dimensions (2,3,5,10,20,40), where each dimension comprises 360 distinct problems. To test higher dimensions, we created an extra set of 784D problems by composing a pre-trained encoder \cite{kingma2013auto} of FashionMnist images, each with 784 pixels \cite{xiao2017/online}, with a 10D COCO problem as follows: $f_{784D}(x):f_{10D}(Encoder(x))$.

We compared EGL with seven baselines, implemented on \href{https://docs.scipy.org/doc/}{Scipy} and \href{http://cma.gforge.inria.fr/apidocs-pycma/cma.evolution_strategy.html}{Cma} Python packages: Nedler Mead, SLSQP, POWELL, CG,  COBYLA, BFGS, CMA-ES and with our IGL implementation based on the DDPG approach. For the 784D problems we evaluated only CG, CMA-ES and IGL, which yielded results in reasonable time. We examined two network models. To compare against other baselines (Fig. \ref{fig:egl_baseline_cmp_success} and Fig. \ref{fig:egl_dim_avg}(a)), we used our Spline Embedding model (details in Appendix Sec. \ref{sec:SplineEmb}). Since Spline net is more time consuming, for the ablation tests in Fig. \ref{fig:egl_dim_avg}(b-d), we used a lighter Fully Connected (FC) net. For additional information including hyperparameters list, refer to Appendix Sec. \ref{sec:coco_exp}.

Fig. \ref{fig:egl_baseline_cmp_success} presents the success rate of each algorithm with respect to the dimension number for a budget $C=150\cdot10^3$. A single test-run is considered successful if: (1) $y_{best} - y^* \leq 1$; and (2) $\frac{y_{best}-y^*}{y_0 - y^*} \leq 10^{-2}$. Here, $y_0$ is the initial score $f(x_0)$,  $y_{best}=\min y_k$ is the best observed value for that run, and $y^*$ is the minimal value obtained from all the baselines' test-runs. This definition guarantees that a run is successful only if its best observed value is near $y^*$, both in terms of absolute distance and in terms of relative improvement with respect to the initial score. The results show that EGL outperforms all other baselines for any dimension, and most importantly, as the dimension number increases, the performance gap between EGL and all other baselines grows larger.

To visualize the average convergence rate of each method, we first calculate a scaled distance between the best value at time $t$ and the optimal value $\Delta y_{best}^t = \frac{ \min_{k\leq t} y_k - y^*}{y_0 - y^*}$. We then average this number, for each $t$, over all runs in the same dimension problem set. This distance is now scaled from zero to one and the results are presented on a log-log scale. In Fig. \ref{fig:egl_dim_avg}(a) we present $\overline{\Delta y}_{best}^t$ on the 40D problem set. Both EGL and IGL have a cold start behavior which is a result of a warm-up phase where we evaluate $384$ points around $x_0$ and train the networks before executing the first gradient descent step. The elbow pattern in step $384$ marks the switch from $x_0$ to $x_1$. Unlike other baselines that settle on a local minimum, EGL monotonically decreases for the entire optimization process. It finally overtakes the best baseline (CMA-ES) after roughly $10^4$ steps.

\ifx\hidepics\undefined

\begin{figure*}[ht!]
  \centering
  \includegraphics[width=1\linewidth]{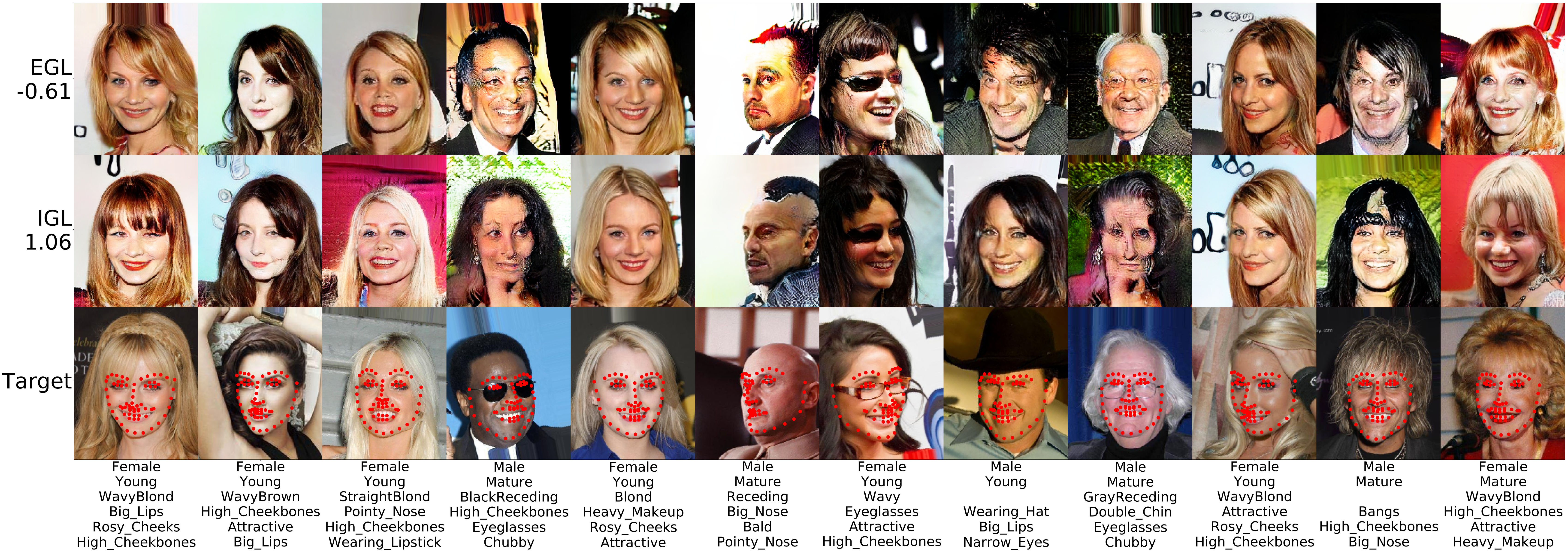}
  %\vspace{-20pt}
  \caption{Searching latent space of generative models with EGL and IGL. Note that {\em the target image is not revealed to the optimizer}, only the face attributes and landmark points. The left-hand number is the average minimal value for each algorithm over 64 different problems.}
\label{fig:latent}
% \vspace{-1em}
\end{figure*}

\else

\begin{figure*}[ht!]
  \centering
  \includegraphics[width=1\linewidth]{example-image-a}
\end{figure*}

\fi

Fig. \ref{fig:egl_dim_avg}(b) demonstrates the advantage of output-mapping (OM) and trust-region (TR). Here, we compared an FC net, trained with OM and TR (\texttt{FC\_TR\_OM}) against the variations \texttt{FC}, \texttt{FC\_TR} and \texttt{FC\_OM}. The results show a clear advantage of using both OM and TR, yet, while OM is crucial for the entire optimization process, the TR advantage materializes only near the minimal value. In addition, Fig. \ref{fig:egl_dim_avg}(b) manifests the gain of the Spline net (\texttt{SPLINE}) on top of \texttt{FC\_TR\_OM}.

The next experiment is executed on the high dimensional 784D problem set. In Fig. \ref{fig:egl_dim_avg}(c) we compare the performance of different numbers of exploration points $m = \{64,800\}$ against CMA-ES, CG and the IGL baselines. While, Theorem \ref{th:regression_problem} guarantees a controllably accurate model when sampling $m\geq n+1$ points, remarkably, EGL generalized to an outstanding performance and outperformed all other baselines even for $m =64 \ll n$. Nevertheless, as expected, more exploration points converged to a better final value. In practice, the choice of $m$ should correspond to the allocated budget size. More exploration around each candidate comes with the cost of fewer gradient descent steps, thus, for low budgets, one should typically choose a small $m$.

Next, we tested the gain of two possible modifications to EGL: (1) perturbations with $g_{\varepsilon}^p$; and (2) training $g_{\theta}$ with a larger replay buffer (RB) with exploration points from the last $L$ candidates. These tests were all executed with an identical random seed. Fig. \ref{fig:egl_dim_avg}(d) presents the performance gain $(1 - \overline{\Delta y}_{best}^t/\overline{\Delta y}_{RB1\_P0}^t)$ as a function of $t$ for several different runs. Small perturbations $p=0.01\varepsilon$ improved the performance ($14\%$), possibly since they also regulate the NN training, yet, too large perturbations of $p=0.1\varepsilon$ yielded inconsistent results. We also observed that a too long RB of $L=16$ largely hurt the performance, probably since it adds high values to the RB which leads to a more compressed output mapping. However, moderate RB of $L=4$ had a positive impact ($10\%$), probably since more exploration points near the current candidate reduce the controllable accuracy factor $\kappa_g$ and thus the accuracy of $g_{\varepsilon}$ improves.

\subsection{Searching the latent space of generative models}

To examine EGL in a high-dimensional, complex, non-convex and noisy domain, we experimented with the task of searching the latent space of an image generative model \cite{volz2018evolving}. Generative models learn to map between a latent predefined distribution $z$ to a complex real world distribution $x$ (e.g. images, audio, etc.). Given a trained Black-Box generator, while it is easy to sample from the distribution of $x$ by sampling from $z$, it is not straightforward to generate an image with some desired characteristics. For that purpose, one may apply a BBO to search the latent space for a hidden representation $z^*$ that generates an image with the desired traits $x^*$. Here, we used a face generative model and optimized $z^*$ to generate an image with a required set of face attributes, landmark points and quality.

We trained generator \& discriminator for the CelebA dataset \cite{liu2015faceattributes} based on the BigGAN \cite{brock2018large} architecture and a classifier for the CelebA attributes. For the face landmark points, we used a pre-trained model \cite{kazemi2014one}. The BBO was trained to minimize the following objective:
\begin{equation*}
    f_{al}(z) = \lambda_a \mathcal{L}_a(G(z)) + \lambda_l \mathcal{L}_l(G(z)) + \lambda_g  \tanh(D(G(z)))
\end{equation*}
Where: (1) $\mathcal{L}_a$ is the Cross-Entropy loss between the generated face attributes as measured by the classifier and the desired set of attributes $a$; (2) $\mathcal{L}_l$ is the MSE between the generated landmark points and the desired set of landmarks $l$ ; and (3) $D(G(z))$ is the discriminator output, positive for low-quality images and negative for high-equality images.

Since each evaluation of $z$ is costly, we limited the budget $C$ to only $10^4$ evaluations and the number of exploration points in each step was only $m=32 \ll 512$. Due to the high-dimensional problem, classic methods such as CG and even CMA-ES completely fail to generate satisfying faces (see Appendix Sec. \ref{sec:latent_space_search} for sample images generated by CG and CMA-ES, and additional results). In Fig. \ref{fig:latent} we compare the images generated by EGL and IGL. Generally, the quality of the results depends on the image target style. Some face traits which are more frequent in the CelebA dataset lead to a better face quality. Some faces, specifically with attributes such as a beard or a hat, are much harder to find, probably due to the suboptimality of the generator which usually reduces the variety of images found in the dataset. Nevertheless, the results show that EGL produces better images both visually and according to the final cost value.
We observed that, for hard targets, IGL more frequently fails to find any plausible candidates while EGL finds images that at least resemble some of the required characteristics.

\section{Conclusions}
\label{sec:Conclusions}

We presented EGL, a derivative-based BBO algorithm that achieves state-of-the-art results on a wide range of optimization problems. The essence of its success is a learnable function that estimates the mean gradient with a controllable smoothness factor. Starting with a high smoothness factor, let EGL find global areas in the function manifold with low valleys. Gradually decreasing it, lets EGL converge to a local minimum. The concept of EGL can be generalized to other related fields, such as sequential decision-making problems (i.e. Reinforcement Learning), by directly learning the gradient of the $Q$-function. We also demonstrated success in an interesting applicative high-dimensional Black-Box problem, searching the latent space of generative models.

\clearpage
\newpage

\bibliography{mybib}
\bibliographystyle{icml2020}

\onecolumn
\appendix

\appendix

\section{Theoretical Analysis}
\label{sec:Theoretical_Analysis}

\subsection{The Mean-Gradient}

\begin{definition}
The mean-gradient in a region around $x$ of radius $\varepsilon > 0$ is
\begin{equation}
\label{appendix_mean_gradient}
    g_{\varepsilon}(x) = \arg\min_{g\in\mathbb{R}^n} \int\displaylimits_{V_{\varepsilon}(x)}|g\cdot \tau - f(x+\tau) + f(x)|^2 d\tau
\end{equation}
where $V_{\varepsilon}(x)\subset\mathbb{R}^n$ is a convex subset s.t. $\|x' - x\|\leq \varepsilon$ for all $x'\in V_{\varepsilon}(x)$ and the integral domain is over $\tau$ s.t. $x+\tau\in V_{\varepsilon}(x)$.
\end{definition}

\begin{prop}[controllable accuracy]
\label{appendix_controllable_accuracy}
For any twice differentiable function $f\in\mathcal{C}^1$, there is $\kappa_g(x) > 0$, so that for any $\varepsilon > 0$ the mean-gradient satisfies $\|g_{\varepsilon}(x) -\nabla f(x)\| \leq \kappa_g(x) \varepsilon$ for all $x\in\Omega$.   
\end{prop}

\begin{proof}

Recall the Taylor theorem for a twice differentiable function $f(x+\tau)=f(x) + \nabla f(x) \cdot \tau + R_{x}(\tau)$, where $R_{x}(\tau)$ is the remainder. Since the gradient is continuous,by the fundamental theorem for line integrals
\begin{equation}
    f(x+\tau) = f(x) + \nabla f(x) \cdot \tau + \int_0^1 (\nabla f(x + t\tau) - \nabla f(x)) \cdot \tau dt
\end{equation}
Since $f\in\mathcal{C}^{1+}$, we also have $|\nabla f(x) - \nabla  f(x + \tau)|\leq \kappa_{f}\|\tau\|$. We can use this property to bound the remainder in the Taylor expression.
\begin{equation}
\begin{aligned}
    R_{x}(\tau) &= \int_0^1 (\nabla f(x + t\tau) - \nabla f(x)) \cdot \tau dt \\ 
    &\leq \kappa_f \int_0^1 \|x + t\tau - x\| \cdot \|\tau\| dt = \kappa_f  \|\tau\|^2 \int_0^1 t dt = \frac{1}{2}\kappa_f \|\tau\|^2
\end{aligned}
\end{equation}

Now, by the definition of $g_{\varepsilon}$, an upper bound for $\mathcal{L}(g_{\varepsilon}(x))$ is

\begin{align*}
    \mathcal{L}(g_{\varepsilon}(x)) \leq \mathcal{L}(\nabla f(x)) &= \int\displaylimits_{V_{\varepsilon}(x)}|\nabla f(x) \cdot \tau  - f(\tau) + f(s)|^2 d\tau = \int\displaylimits_{V_{\varepsilon}(x)}|R_{x}(\tau)|^2 d\tau\\ 
    & \leq \frac{1}{4}\kappa_f^2 \int\displaylimits_{V_{\varepsilon}(x)}|\|\tau\|^2|^2  d\tau \leq \kappa_f \varepsilon^4 |V_{\varepsilon}(x)| = \frac{1}{4}\kappa_f^2 \varepsilon^{n+4} |V_{1}(x)|
\end{align*}

To develop the lower bound we will assume that $\dim(\spn(V_{\varepsilon}(x)))=n$ and we will use the following definition
\begin{equation}
    M_{\varepsilon}(x) = \min_{\hat{\mathbf{n}}} \int_{V_{\varepsilon}(x)\setminus V_{\frac{\varepsilon}{2}}(x)}\left|\frac{\tau}{\|\tau\|}\cdot \hat{\mathbf{n}}\right|^2 d\tau
\end{equation}
where $\hat{\mathbf{n}}\in \mathbb{R}^n$ s.t. $\|\hat{\mathbf{n}}\|=1$ and we assumed that $V_{\frac{\varepsilon}{2}}\subset V_{\varepsilon}$. As the dimension of $V_{\varepsilon}(x)$ is $n$, it is obvious that $M_{\varepsilon}(x) > 0$.

The lower bound is

\begin{align*}
    \mathcal{L}(g_{\varepsilon}(x)) &= \int_{V_{\varepsilon}(x)}\left|g_{\varepsilon}(x)\cdot\tau-\nabla f(x)\cdot\tau+\nabla f(x)\cdot\tau-f(x+\tau)+f(x)\right|^2d\tau \\
&\geq \int_{V_{\varepsilon}(x)}\left(\left|g_{\varepsilon}(x)\cdot\tau-\nabla f(x)\cdot\tau\right|-\left|\nabla f(x)\cdot\tau-f(x+\tau)+f(x)\right|\right)^2d\tau \\
&= \int_{V_{\varepsilon}(x)}\left|g_{\varepsilon}(x)\cdot\tau-\nabla f(x)\cdot\tau\right|^2d\tau + \int_{V_{\varepsilon}(x)}\left|\nabla f(x)\cdot\tau-f(x+\tau)+f(x)\right|^2d\tau \\&-2\int_{V_{\varepsilon}(x)}\left|g_{\varepsilon}(x)\cdot\tau-\nabla f(x)\cdot\tau\right|\cdot\left|\nabla f(x)\cdot\tau-f(x+\tau)+f(x)\right|\tau \\
& \geq \int_{V_{\varepsilon}(x)}\left|g_{\varepsilon}(x)\cdot\tau-\nabla f(x)\cdot\tau\right|^2d\tau -2\int_{V_{\varepsilon}(x)}\left|g_{\varepsilon}(x)\cdot\tau-\nabla f(x)\cdot\tau\right|\cdot\left|\nabla f(x)\cdot\tau-f(x+\tau)+f(x)\right|\tau  \\
& \geq \int_{V_{\varepsilon}(x)}\left|g_{\varepsilon}(x)\cdot\tau-\nabla f(x)\cdot\tau\right|^2d\tau -\kappa_f \int_{V_{\varepsilon}(x)}\left|g_{\varepsilon}(x)\cdot\tau-\nabla f(x)\cdot\tau\right|\cdot\|\tau\|^2\tau \\
& \geq \left\|g_{\varepsilon}(x)-\nabla f(x)\right\|^2 \int_{V_{\varepsilon}(x)}\left|\hat{\mathbf{n}}(x)\cdot\tau\right|^2d\tau -\kappa_f \left\|g_{\varepsilon}(x)-\nabla f(x)\right\| \int_{V_{\varepsilon}(x)}\cdot\|\tau\|^3\tau \\
& \geq \left\|g_{\varepsilon}(x)-\nabla f(x)\right\|^2 \int_{V_{\varepsilon}(x)\setminus V_{\frac{\varepsilon}{2}}(x)}\left|\hat{\mathbf{n}}(x)\cdot\tau\right|^2d\tau -\kappa_f \left\|g_{\varepsilon}(x)-\nabla f(x)\right\| \varepsilon^{n+3}|V_{1}(x)| \\
& \geq \left\|g_{\varepsilon}(x)-\nabla f(x)\right\|^2 \left(\frac{\varepsilon}{2}\right)^2\int_{V_{\varepsilon}(x)\setminus V_{\frac{\varepsilon}{2}}(x)}\left|\hat{\mathbf{n}}(x)\cdot\frac{\tau}{\|\tau\|}\right|^2d\tau -\kappa_f \left\|g_{\varepsilon}(x)-\nabla f(x)\right\| \varepsilon^{n+3}|V_{1}(x)| \\
& \geq \left\|g_{\varepsilon}(x)-\nabla f(x)\right\|^2 \left(\frac{\varepsilon}{2}\right)^2\varepsilon^n \int_{V_{1}(x)\setminus V_{\frac{1}{2}}(x)}\left|\hat{\mathbf{n}}(x)\cdot\frac{\tau}{\|\tau\|}\right|^2d\tau -\kappa_f \left\|g_{\varepsilon}(x)-\nabla f(x)\right\| \varepsilon^{n+3}|V_{1}(x)| \\
& \geq \frac{1}{4}\left\|g_{\varepsilon}(x)-\nabla f(x)\right\|^2 \varepsilon^{n+2} M_1(x) - \kappa_f \left\|g_{\varepsilon}(x)-\nabla f(x)\right\| \varepsilon^{n+3}|V_{1}(x)| \\
\end{align*}

Combining the upper and lower bound we obtain

\begin{align*}
    &\frac{1}{4}\left\|g_{\varepsilon}(x)-\nabla f(x)\right\|^2 \varepsilon^{n+2} M_1(x) - \kappa_f \left\|g_{\varepsilon}(x)-\nabla f(x)\right\| \varepsilon^{n+3}|V_{1}(x)| \leq  \frac{1}{4}\kappa_f^2 \varepsilon^{n+4} |V_{1}(x)|\\
    \Rightarrow \quad & M_1(x) \left\|g_{\varepsilon}(x)-\nabla f(x)\right\|^2 - 4 \kappa_f \varepsilon |V_{1}(x)| \left\|g_{\varepsilon}(x)-\nabla f(x)\right\| - \kappa_f^2 \varepsilon^2 |V_{1}(x)| \leq 0 \\
    \Rightarrow \quad & \left\|g_{\varepsilon}(x)-\nabla f(x)\right\| \leq \varepsilon\kappa_f\frac{2|V_{1}(x)| + \sqrt{4|V_{1}(x)|^2 + |V_{1}(x)|M_1(x)}}{M_1(x)} \\
\end{align*}
\end{proof}

\begin{prop}[continuity]
\label{prop_continuity}
If $f(x)$ is continuous in $V$ s.t. $V_{\varepsilon}(x)\subset V$, then the mean-gradient is a continuous function at $x$.
\end{prop}

\begin{proof}

Let us define the two-variable function $\mathcal{L}(x,g)$:
\begin{equation}
    \mathcal{L}(x,g) = \int_{V_{\varepsilon}(x)}| g \cdot \tau - f(x+\tau) + f(x)|^2  d\tau
\end{equation}
The mean-gradient is the global minimum of this function for each $x$. Notice that $\mathcal{L}(x,g)$ is a polynomial function in $g$ since if $f$ is an integrable function, then we can write
\begin{equation}
\begin{aligned}
    \mathcal{L}(x,g) &= \int_{V_{\varepsilon}(x)}| g \cdot \tau - f(x+\tau) + f(x)|^2  d\tau \\
    &= \int_{V_{\varepsilon}(x)}|g \cdot \tau|^2 d\tau - 2 \int_{V_{\varepsilon}(x)}g \cdot \tau (f(x+\tau) - f(x)) d\tau + \int_{V_{\varepsilon}(x)}|f(x+\tau) - f(x)|^2 d\tau \\
    &= g\cdot \mathbf{A}(x) g + g\cdot b(x) + c(x)
\end{aligned}
\end{equation}
where $\mathbf{A}(x)=\int_{V_{\varepsilon}(x)}\tau\tau^T d\tau$, $b(x) = -2 \int_{V_{\varepsilon}(x)}\tau (f(x+\tau) - f(x)) d\tau$ and $c(x)=\int_{V_{\varepsilon}(x)}|f(x+\tau) - f(x)|^2 d\tau$. Without loss of generality for this proof, we can ignore the constant $c(x)$ as it does not change the minimum point. In addition, note that $A(x)$ is constant for all $x$ since the domain $V_{\varepsilon}(x)$ is invariant for $x$ and the integrand does not depend on $f$. 

Since $\mathcal{L}(x,g)$ is bounded from below, $\mathcal{L}(x,g) \geq 0$, it must have a minimum s.t. $\mathbf{A} \geq 0$. Assume that $\mathbf{A} > 0$ (e.g. when $V_{\varepsilon}$ is an $n$-ball of radius $\varepsilon$), then for each $x$ there is a unique minimum for $\mathcal{L}(x,g)$ and this minimum is the mean-gradient $g_{\varepsilon}(x)$.

To show the continuity of $g_{\varepsilon}(x)$, define $F(x,g)=\nabla_g \mathcal{L}(x,g)$, $F$ maps $\mathbb{R}^{2n}\to\mathbb{R}^n$. Assume that for $x_0$, $g_{\varepsilon}(x_0)$ is the mean-gradient. Therefore, $F(x_0,g_{\varepsilon}(x_0))=0$. Since $A>0$, this means that the derivative $\nabla_g F(x_0, g_{\varepsilon}(x_0))$ is invertible. We will apply a version of the implicit function theorem \cite{loomis1968advanced} (Theorem 9.3 pp. 230-231) to show that there exists a unique and continuous mapping $h(x)$ s.t. $F(x, h(x))=0$. 

To apply the implicit function theorem, we need to show that $F(x,g)$ is continuous and the derivative $\nabla_g F$ is continuous and invertible. The latter is obvious since $\nabla_g F = \mathbf{A} > 0$ is a constant positive definite matrix. $F(x,g)$ is also continuous with respect to $g$, therefore it is left to verify that $F(x,g)$ is continuous with respect to $x$.

\begin{lemma}
If $f(x)$ is continuous in $V$ s.t. $V_{\varepsilon}(x)\subset V$, then $\mathcal{L}(x,g)$ is continuous in $x$.
\end{lemma}

\begin{proof}

\begin{equation}
\begin{aligned}
        \left|\mathcal{L}(x,g) - \mathcal{L}(x',g)\right| &= \left|g\cdot \mathbf{A}(x) g + g\cdot b(x) - g\cdot \mathbf{A}(x') g - g\cdot b(x)\right| \\
        &= \left|g\cdot b(x) -  g\cdot b(x)\right| \leq \|g\| \left\| \int_{V_{\varepsilon}(x)} \tau (f(x + \tau) - f(x)) d\tau - \int_{V_{\varepsilon}(x')} \tau (f(x' + \tau) - f(x')) d\tau \right\|
\end{aligned}
\end{equation}
To write both integrals with the same variable, we change variables to $\tau = \tilde{\tau} - x$ in the first integrand and $\tau = \tilde{\tau} - x'$ in the second integrand.
\begin{equation}
\begin{aligned}
        \left|\mathcal{L}(x,g) - \mathcal{L}(x',g)\right|
        & \leq \|g\| \left\| \int_{V_{\varepsilon}(x)} (\tilde{\tau} - x) (f(\tilde{\tau}) - f(x)) d\tilde{\tau} - \int_{V_{\varepsilon}(x')} (\tilde{\tau} - x') (f(\tilde{\tau}) - f(x')) d\tilde{\tau} \right\| \\
        & \leq \|g\| C_1 + \|g\| C_2 + \|g\| C_3 + \|g\| C_4 + \|g\| C_5
\end{aligned}
\end{equation}
Where
\begin{align}
    & C_1 = |f(x') - f(x)| \int_{V_{\varepsilon}(x)\cap V_{\varepsilon}(x')} \|\tilde{\tau}\|  d\tilde{\tau} \\
    & C_2 = \|x - x'\| \int_{V_{\varepsilon}(x)\cap V_{\varepsilon}(x')} |f(\tilde{\tau})| d\tilde{\tau} \\
    & C_3 = \|x f(x) - x' f(x') \| \int_{V_{\varepsilon}(x)\cap V_{\varepsilon}(x')} d\tilde{\tau} \\
    & C_4 = \int_{V_{\varepsilon}(x)\setminus V_{\varepsilon}(x')} \|\tilde{\tau} - x\|\cdot |f(\tilde{\tau}) - f(x)| d\tilde{\tau} \\
    & C_5 = \int_{V_{\varepsilon}(x')\setminus V_{\varepsilon}(x)} \|\tilde{\tau} - x'\|\cdot |f(\tilde{\tau}) - f(x')| d\tilde{\tau} \\
\end{align}

Taking $x'\to x$, $C_1$, $C_2$, $C_3$ all go to zero as the integral is finite but $x'\to x$ and $f(x') \to f(x)$. For $C_4$ and $C_5$, note that the integrand is bounded but the domain size goes to zero as $x' \to x$. To  see that we will show that $|V_{\varepsilon}(x)\setminus V_{\varepsilon}(x')| \leq |A_{\varepsilon}(x)| \cdot \|x-x'\|$, where $|A_{\varepsilon}(x)|$ is the surface area of $V_{\varepsilon}$.
\begin{lemma}
\label{ap_a_eps}
$|V_{\varepsilon}(x)\setminus V_{\varepsilon}(x')| \leq |A_{\varepsilon}(x)| \cdot \|x-x'\|$
\end{lemma}
\begin{proof}
First, note that if $u\in V_{\varepsilon}(x)$, then $u + x' - x \in V_{\varepsilon}(x')$.
Take $P\subset V_{\varepsilon}$ s.t. $p\in P$ if and only if $\distance(A_{\varepsilon}(x),p) \geq \|x-x'\|$ and $p\in V_{\varepsilon}(x)$. For any $p\in P$, $p - x + x' \in V_{\varepsilon}(x)$, thus, following our first argument $p\in V_{\varepsilon}(x')$. 

We obtain that $P \cap V_{\varepsilon}(x)\setminus V_{\varepsilon}(x') = \Phi$, thus $|V_{\varepsilon}(x)\setminus V_{\varepsilon}(x')|\leq |V_{\varepsilon}(x)\setminus P|$. However, all points $q\in V_{\varepsilon}(x)\setminus P$ satisfy $\distance(A_{\varepsilon}(x),q) \leq \|x-x'\|$, therefore $|V_{\varepsilon}(x)\setminus P| \leq |A_{\varepsilon}(x)|\cdot \|x-x'\|$.
\end{proof}

Following Lemma \ref{ap_a_eps} we obtain that the integral in $C_4$ and $C_5$ goes to zero and therefore the distance $\left|\mathcal{L}(x,g) - \mathcal{L}(x',g)\right|\to 0$ as $x\to x'$.
\end{proof}

$\mathcal{L}(x,g)$ continuous in $x$ and $g$ with a continuous derivative in $g$ implies that $\nabla_g \mathcal{L}(x,g)$ is continuous in $x$. We can now apply Theorem 9.3 pp. 230-231 in \cite{loomis1968advanced} and conclude that there is a unique continuous mapping $h(x)$ s.t. $F(x, h(x))=0$. Since $A>0$, this means that such a mapping defines a local minimum for $\mathcal{L}(x,g)$ in $g$. Further, since $\mathcal{L}(x,g)$ is a second degree polynomial in $g$, this is a unique global mapping.  Therefore, it must be equal to $g_{\varepsilon}(x)$ and hence $g_{\varepsilon}$ is continuous in $x$.

\end{proof}

\subsection{Parametric approximation of the mean-gradient} 

In this section we analyze the Monte-Carlo learning of the mean-gradient with a parametric model. Generally, we define a parametric model $g_{\theta}$ and learn $\theta^*$ by minimizing the term
\begin{equation}
\label{eq:mc_appendix}
    \mathcal{L}(g_{\theta}, \varepsilon) = \sum_{i=1}^N\sum_{x_j\in V_{\varepsilon}(x_i)} |(x_j - x_i) \cdot g_{\theta}(x_i) - y_j + y_i |^2
\end{equation}

We start by analyzing constant parameterization of the mean-gradient around a candidate $x_k$. We consider two cases: (1) interpolation, where there are exactly $n+1$ evaluation points; and (2) regression where there are $m > n+1$ evaluation points. This line of arguments follows the same approach taken in \cite{audet2017derivative}, Chapter 9.

\subsubsection{Constant parameterization with $n+1$ interpolation points}
\label{linear_interpolation}

\begin{definition}
A set of $n+1$ points $\{x_i\}_{0}^{n}$, s.t. every subset of $n$ points spans $\mathbb{R}^n$, is a poised set for constant interpolation.
\end{definition}
\begin{prop}
\label{poised_uniqueness}
For a constant paramterization $g(x) = g$,  a poised set has a unique solution with zero regression error.
\begin{equation}
\label{min_g_const}
    \min_{g}\sum_{i,j \in \mathcal{D}}|(x_j-x_i)\cdot g - y_j + y_i|^2 = 0
\end{equation}
\end{prop}
\begin{proof}
Define the matrix $\tilde{X}_{i}\in\mathbb{M}^{n\times n}$ s.t. the $j$-th row is $x_i-x_j$ and $\delta_i\in\mathbb{R}^{n}$ s.t. $\delta_{i,j} = y_j - y_i$. We may transform Eq. (\ref{min_g_const}) into $n+1$ sets of linear equations:
\begin{equation}
    \forall \ i \ \ \tilde{X}_i g = \delta_i
\end{equation}
While there are $n+1$ different linear systems of equations, they all have the same solution $g_{\min}$. To see that, define the system of equation $\tilde{X}\tilde{g}=r$ where
\begin{equation}
    \tilde{X} = \begin{pmatrix} x_0 & 1 \\ x_1 & 1 \\ \vdotswithin{x_1} & \vdotswithin{1} \\ x_n & 1 \end{pmatrix}, \ \ \ \ \ g = \begin{pmatrix} g_0 \\ g_1 \\ \vdotswithin{r_1} \\ g_{n-1} \\ s \end{pmatrix}, \ \ \ \ r = \begin{pmatrix} y_0 \\ y_1 \\ \vdotswithin{y_1} \\ y_n \end{pmatrix}
\end{equation}
and $s$ is an additional slack variable. For all $i$ we can apply an elementary row operation of subtracting the $i$-th row s.t. the updated system is
\begin{equation}
    \left(\begin{matrix} x_0 - x_i & 0 \\ \vdotswithin{x_1 - x_i} & \vdotswithin{0} \\ x_{i-1}  - x_i & 0 \\ 0 & 0 \\ x_{i+1}  - x_i & 0 \\ \vdotswithin{x_1  - x_i} & \vdotswithin{0} \\ x_n  - x_i & 0 \end{matrix} \ \ \middle| \ \ \begin{matrix} y_0 - y_i \\ \vdotswithin{y_1 - y_i} \\ y_{i-1} - y_i \\ 0 \\ y_{i+1} - y_i \\ \vdotswithin{y_1 - y_i} \\ y_n - y_i \end{matrix}\right)
\end{equation}
Reducing the zeroed $i$-th row we get the system of equation $\tilde{X}_i g = \delta_i$ which has a unique solution since the set $\{x_j\} \setminus x_i$ spans $\mathbb{R}^n$. 
\end{proof}
\begin{corollary}
For any parameterization of the form $g_{\theta} = f(Wx) + b$ and a poised set $\{x_j\}_{0}^n$ we have an optimal solution where $W^*=0$ and $b^*=g_{\min}$. Specifically it also holds for a Neural Network with a biased output layer.
\end{corollary}

\begin{lemma}
For a poised set $\{x_j\}_{0}^n$ s.t. $\|x_i-x_j\|\leq\varepsilon$ and a mean-gradient estimator $g_{\theta}\in\mathcal{C}^{0}$ with zero interpolation error, the following holds
\begin{equation}
    \|\nabla f(x) - g_{\theta}(x)\| \leq \kappa_g \varepsilon
\end{equation}
\end{lemma}
\begin{proof}
$f\in\mathcal{C}^{1+}$, hence for any $x_i$ in the poised set and $x$ s.t. $\|x-x_i\|\leq\varepsilon$ we have
\begin{multline}
    \|\nabla f(x) - g_{\theta}(x) \| \leq \\ \|\nabla f(x) - \nabla f(x_i) \| + \|\nabla f(x_i) - g_{\theta}(x_i) \| + \|g_{\theta}(x_i) -g_{\theta}(x) \| \leq (\kappa_f + \kappa_{g_{\theta}})\varepsilon + \|\nabla f(x_i) - g_{\theta}(x_i) \|
\end{multline}
Where $\kappa_{g_{\theta}}$ is the Lipschitz constant of $g_{\theta}$ it is left to bound the last term. First, note that for all $x_j$ in the poised set we have that
\begin{equation}
    (x_j-x_i) \cdot g_{\theta}(x_i) = f(x_j) - f(x_i) \leq (x_j-x_i) \cdot \nabla f(x_i) + \frac{1}{2}\kappa_f \varepsilon^2
\end{equation}
where the last equation comes from the second error term in the Taylor series expansion in $x_i$ (see Proposition \ref{appendix_controllable_accuracy}).
Returning to our definition of $\tilde{X}_i$ (see proposition \ref{poised_uniqueness}) we can write
\begin{equation}
    \|\tilde{X}_i (\nabla f(x_i) - g_{\theta}(x_i))\| = \sqrt{\sum_{i} \left|(x_j-x_i) \cdot (g_{\theta}(x_i) - \nabla f(x_i))\right|^2} \leq \frac{1}{2} \sqrt{n} \kappa_f \varepsilon^2
\end{equation}
Using that property we have
\begin{equation}
    \|\nabla f(x_i) - g_{\theta}(x_i)\| = \|\tilde{X}_i^{-1} \tilde{X}_i (\nabla f(x_i) - g_{\theta}(x_i))\| \leq \|\tilde{X}_i^{-1}\|  \|\tilde{X}_i (\nabla f(x_i) - g_{\theta}(x_i))\| \leq  \frac{1}{2} \sqrt{n} \kappa_f \|\tilde{X}_i^{-1}\| \varepsilon^2.
\end{equation}
$\|\tilde{X}_i^{-1}\|=\frac{1}{\min\sigma(\tilde{X}_i)}$, where $\sigma$ is the singular values. Notice that the rows of $\tilde{X}_i$ are $x_j-x_i \propto \varepsilon$, thus we can scale them by $\varepsilon$. In this case, since the poised set spans $\mathbb{R}^n$, the minimal singular value of $\frac{1}{\varepsilon}\tilde{X}_i$ is finite and does not depend on $\varepsilon$. Therefore, we obtain
\begin{equation}
    \|\nabla f(x_i) - g_{\theta}(x_i)\| \leq \frac{1}{2} \sqrt{n} \|(\frac{1}{\varepsilon}\tilde{X}_i)^{-1}\| \kappa_f  \varepsilon  =  O(n \varepsilon)
\end{equation}

Therefore,
\begin{equation}
        \|\nabla f(x) - g_{\theta}(x) \| \leq  \left(\kappa_f + \frac{1}{2} \sqrt{n} \kappa_{g_{\theta}} + \|(\frac{1}{\varepsilon}\tilde{X}_i)^{-1}\| \kappa_f \right) \varepsilon
\end{equation}

\end{proof}
Notice that we only required $g_{\theta}$ to be a zero-order Lipschitz continuous and we do not set any restrictions on its gradient. For Neural Networks, having the  $\mathcal{C}^0$ property is relatively easy, e.g. with spectral normalization \cite{miyato2018spectral}. However, many NNs are not $\mathcal{C}^1$, e.g. NN with ReLU activations. 

\subsubsection{Constant parameterization with $m > n+1$ regression points}

We can extend the results of Sec. \ref{linear_interpolation} to the regression problem where we have access to $m > n+1$ points $\{x_i\}_0^{m-1}$. We wish to show that the bounds for a constant mean-gradient solution for the regression problem in Eq. (\ref{eq:mc_appendix}) are also controllably accurate, i.e. $\|\nabla f(x) - g\| \leq \kappa_g \varepsilon$. As in the interpolation case, we start with the definition of the poised set for regression.

\begin{definition}[poised set for regression]
Let $\mathcal{D}_k = \{(x_i, y_i)\}_{1}^{m}$, $m\geq n+1$ s.t. $x_i \in V_{\varepsilon}(x_k)$ for all $i$. Define the matrix $\tilde{X}_{i}\in\mathbb{M}^{m\times n}$ s.t. the $j$-th row is $x_i-x_j$. Now define $\tilde{X}=\begin{psmallmatrix}\tilde{X}_1^T & \cdots & \tilde{X}_m^T\end{psmallmatrix}^T$. The set $\mathcal{D}_k$ is a poised set for regression in $x_k$ if the matrix $\tilde{X}$ has rank $n$.
\end{definition}
Intuitively, a set is poised if its difference vectors $x_i-x_j$ span $\mathbb{R}^n$. For the poised set, and a constant parameterization, the solution of Eq. (\ref{mc_objective}) is unique and it equals to the Least- Squares (LS) minimizer. If $f$ has a Lipschitz continuous gradient, then the error between $g_{\varepsilon}(x)$ and $\nabla f(x)$ is proportional to $\varepsilon$. We formalize this argument in the next proposition.

\begin{prop}
Let $\mathcal{D}_k$ be a poised set in $V_{\varepsilon}(x_k)$. The regression problem
\begin{equation}
\label{regression_problem}
    g^{MSE} = \arg\min_{g} \sum_{i,j \in \mathcal{D}_k} |(x_j - x_i) \cdot g - y_j + y_i |^2
\end{equation}
has the unique solution $g^{MSE}=(\tilde{X}^T \tilde{X})^{-1} \tilde{X}^T \delta$, where $\delta\in\mathbb{R}^{m^2}$ s.t. $\delta_{i \cdot (m-1) + j} = y_j - y_i$. Further, if $f\in\mathcal{C}^{1+}$ and $g_{\theta}\in\mathcal{C}^0$ is a parameterization with lower regression loss $g^{MSE}$, the following holds
\begin{equation}
    \|\nabla f(x) - g_{\theta}(x)\| \leq \kappa_g \varepsilon
\end{equation}
\end{prop}

\begin{proof}

The regression problem can be written as 
\begin{equation}
\label{regression_problem}
    g = \arg\min_{g} \|\tilde{X}g - \delta\|^2
\end{equation}
This is the formulation for the mean-square error problem with matrix $\tilde{X}$ and target $\delta$. The minimizer of this function is the standard mean-square error minimizer which is unique as $\tilde{X}$ has rank $n$.
\begin{equation}
    g = (\tilde{X}^T \tilde{X})^{-1} \tilde{X}^T \delta
\end{equation}

\begin{lemma}
Let $f \in \mathcal{C}^{1+}$ on $B_{\varepsilon}(x_j)$ with a Lipschitz constant $\kappa_f$. For any triplet $x_i$, $x_j$, $x_k$ s.t. $\|x_i-x_j\|\leq \varepsilon$ and $\|x_k-x_j\|\leq \varepsilon$
\begin{equation}
    |f(x_k) - f(x_j) - (x_k - x_j) \cdot \nabla f(x_i)| \leq \frac{3}{2} \kappa_f \varepsilon^2
\end{equation}
\end{lemma}
\begin{proof}
\begin{equation}
    \begin{aligned}
    |f(x_k) - f(x_j) - (x_k - x_j) \cdot \nabla f(x_i)| &= \left|\int_0^1 (x_k - x_j) \cdot \nabla f(x_j + \tau (x_k - x_j))d\tau - (x_k - x_j) \cdot \nabla f(x_i)\right| \\
    & = \left|\int_0^1 (x_k - x_j) \cdot \left(\nabla f(x_j + \tau (x_k - x_j)) - \nabla f(x_i)\right)d\tau\right| \\ 
    & \leq \left|\int_0^1 \|x_k - x_j\| \cdot \|\nabla f(x_j + \tau (x_k - x_j)) - \nabla f(x_i)\|d\tau\right| \\ 
     & \leq \kappa_f \varepsilon \left|\int_0^1 \|x_j + \tau (x_k - x_j) - x_i\|d\tau\right| \\ 
     & \leq \kappa_f \varepsilon \left|\int_0^1  \|x_j - x_i\| + \|\tau (x_k - x_j)\| d\tau\right| \\ 
     & \leq \kappa_f \varepsilon \left|\varepsilon + \varepsilon\int_0^1  \tau d\tau\right| \\ 
     & = \frac{3}{2} \kappa_f \varepsilon^2
\end{aligned}
\end{equation}
\end{proof}

Applying the previous Lemma, for all $x\in V_{\varepsilon}(x_k)$
\begin{equation}
    \|\tilde{X} \nabla f(x) - \delta\|^2 = \sum_{k=0}^{m-1}\sum_{j=0}^{m-1} |(x_k-x_j)^T\nabla f(x) - (f(x_k) - f(x_j))|^2 \leq 
     m^2 \left(\frac{3}{2} \kappa_f \varepsilon^2\right)^2
\end{equation}
hence $\|\tilde{X} \nabla f(x) - \delta\| \leq \frac{3m}{2} \kappa_f \varepsilon^2$.

Notice also that if $\mathcal{D}_k$ is a poised set s.t. the matrix $\tilde{X}$ has rank $n$ s.t. $\tilde{X}^{\dagger}=(\tilde{X}^T\tilde{X})^{-1}\tilde{X}^T$ exists and we have that
\begin{equation}
    \|\nabla f(x) - \tilde{X}^{\dagger}\delta\| = \left\|\tilde{X}^{\dagger}\left(\tilde{X}\nabla f(x) - \delta\right)\right\| \leq \|\tilde{X}^{\dagger}\|\cdot\|\tilde{X}\nabla f(x_i) - \delta\|  \leq \|\tilde{X}^{\dagger}\| \frac{3m}{2} \kappa_f \varepsilon^2
\end{equation}

As in the interpolation case, we can multiply $\tilde{X}^{\dagger}$ by $\varepsilon$ to obtain a matrix which is invariant to the size of $\varepsilon$. Denote the scaled pseudo-inverse as $\tilde{X}^{\ddagger} = \varepsilon(\tilde{X}^T \tilde{X})^{-1} \tilde{X}^T$. Therefore,
\begin{equation}
    \|\nabla f(x) - g^{MSE}\| \leq \|\tilde{X}^{\ddagger}\| \frac{3m}{2} \kappa_f \varepsilon
\end{equation}

If $g_{\theta}$ has a lower regression error than $g^{MSE}$, then there exists at least one point $x_i$ s.t. $x_i\in\mathcal{D}_k$ and $\|\nabla f(x_i) - g_{\theta}(x_i)\| \leq \|\nabla f(x_i) - g^{MSE}\|$. 
In this case we have

\begin{align*}
    \|\nabla f(x) - g_{\theta}(x)\| & \leq \|\nabla f(x) - \nabla f(x_i) \| + \|\nabla f(x_i) - g_{\theta}(x_i) \| + \|g_{\theta}(x_i)  -g_{\theta}(x) \| \\
    & \leq (\kappa_f + \kappa_{g_{\theta}})\varepsilon + \|\nabla f(x_i) - g_{\theta}(x_i) \| \\
    & \leq (\kappa_f + \kappa_{g_{\theta}})\varepsilon + \|\nabla f(x_i) - g^{MSE}\| \\
    & \leq (\kappa_f + \kappa_{g_{\theta}})\varepsilon + \|\tilde{X}^{\ddagger}\| \frac{3m}{2} \kappa_f \varepsilon \\ 
    &= \left(\kappa_f + \kappa_{g_{\theta}} + \|\tilde{X}^{\ddagger}\| \frac{3m}{2} \kappa_f \right) \varepsilon
\end{align*}

\end{proof}

\begin{corollary}
\label{app_lipschitz_nn}
For the $\mathcal{D}_k$ poised set, any Lipschitz continuous parameterization of the form $g_{\theta}(x)=F(Wx)+b$, specifically NNs, is a controllably accurate model in $V_{\varepsilon}(x_k)$ for the optimal set of parameters $\theta^*$.  
\end{corollary}

\subsection{Convergence Analysis}

For clarity, we replace the subscript $\theta$ in $g_{\theta}$ and write $g_{\varepsilon}$ to emphasize that our model for the mean-gradient is controllably accurate.

\begin{theorem}
Let $f:\Omega\to\mathbb{R}$ be a convex function with Lipschitz continuous gradient, i.e. $f\in\mathcal{C}^{+1}$ and a Lipschitz constant $\kappa_f$ and let $f(x^{*})$ be its optimal value. Suppose a controllable mean-gradient model $g_{\varepsilon}$ with error constant $\kappa_g$, the gradient descent iteration $x_{k+1} = x_{k} - \alpha g_{\varepsilon}(x_k)$ with a sufficiently small $\alpha$ s.t. $\alpha \leq \min(\frac{1}{\kappa_g}, \frac{1}{\kappa_f})$ guarantees:
\begin{enumerate}
    \item For $\varepsilon \leq \frac{\|\nabla f(x)\|}{5\alpha}$, monotonically decreasing steps s.t. $f(x_{k+1}) \leq f(x_k) -  2.25 \frac{\varepsilon^2}{\alpha}$.
    \item After a finite number of iteration, the descent process yields $x^{\star}$ s.t. $\|\nabla f(x^{\star})\| \leq \frac{5\varepsilon}{\alpha}$.
\end{enumerate}
\end{theorem}

\begin{proof}
For convex function with Lipschitz continuous gradient the following inequality holds for all $x_k$
\begin{equation}
    f(x) \leq f(x_k) + (x - x_k) \cdot \nabla f(x_k) + \frac{1}{2}\kappa_f\|x - x_k\|^2
\end{equation}
Plugging in the iteration update $x_{k+1} = x_k - \alpha g_{\varepsilon}(x_k) $ we get
\begin{equation}
    f(x_{k+1}) \leq f(x_k) - \alpha g_{\varepsilon}(x_k) \cdot \nabla f(x_k) + \alpha^2 \frac{1}{2}\kappa_f\|g_{\varepsilon}(x_k)\|^2
\end{equation}
For a controllable mean-gradient we can write $\|g_{\varepsilon}(x) - \nabla f(x)\| \leq \varepsilon \kappa_g $, therefore we can write $g_{\varepsilon}(x) = \nabla f(x) + \varepsilon \kappa_g \xi(x)$ s.t. $\|\xi(x)\| \leq 1$ so the inequality is
\begin{equation}
    f(x_{k+1}) \leq f(x_k) - \alpha \|\nabla f(x_k)\|^2 - \alpha \varepsilon \kappa_g \xi(x_k) \cdot \nabla f(x_k) + \alpha^2 \frac{1}{2}\kappa_f\|\nabla f(x) + \varepsilon \kappa_g \xi(x)\|^2
\end{equation}
Using the equality $\|a+b\|^2 = \|a\|^2 + 2 a \cdot b + \|b\|^2$ and the Cauchy-Schwartz inequality inequality $a \cdot b\leq\|a\|\|b\|$ we can write
\begin{equation}
\begin{aligned}
    f(x_{k+1}) & \leq f(x_k) - \alpha \|\nabla f(x_k)\|^2 + \alpha \varepsilon \kappa_g \|\xi(x_k)\|\cdot\|\nabla f(x_k)\| + \alpha^2 \frac{1}{2}\kappa_f\left(\|\nabla f(x)\|^2 + 2 \varepsilon \kappa_g \|\xi(x)\|\cdot\|\nabla f(x_k)\| + \varepsilon^2 \kappa_g^2 \|\xi(x)\|^2 \right) \\ & \leq f(x_k) - \alpha \|\nabla f(x_k)\|^2 + \alpha \varepsilon \kappa_g \|\nabla f(x_k)\|  +  \frac{\alpha^2}{2}\kappa_f \|\nabla f(x)\|^2 + \alpha^2 \kappa_f \varepsilon \kappa_g \|\nabla f(x_k)\| +  \frac{\alpha^2 \varepsilon^2}{2} \kappa_f  \kappa_g^2 
\end{aligned}
\end{equation}
Using the requirement  $\alpha \leq \min(\frac{1}{\kappa_g}, \frac{1}{\kappa_f})$ it follows that $\alpha \kappa_g \leq 1$ and $\alpha \kappa_f \leq 1$ so
\begin{equation}
\begin{aligned}
    f(x_{k+1})  & \leq f(x_k) - \alpha \|\nabla f(x_k)\|^2 + \varepsilon \|\nabla f(x_k)\|  +  \frac{\alpha}{2} \|\nabla f(x)\|^2 + \varepsilon \|\nabla f(x_k)\| +  \frac{\varepsilon^2}{2}\kappa_g \\
    & = f(x_k) -  \frac{\alpha}{2} \|\nabla f(x_k)\|^2 + 2 \varepsilon \|\nabla f(x_k)\| +  \frac{\varepsilon^2}{2}\kappa_g
    \end{aligned}
\end{equation}
Now, for $x$ s.t. $\|\nabla f(x)\| \geq \frac{5\varepsilon}{\alpha}$ then $\varepsilon \leq \|\nabla f(x)\|\frac{\alpha}{5}$. Plugging it to our inequality we obtain
\begin{equation}
\begin{aligned}
    f(x_{k+1})  & \leq f(x_k) -  \frac{\alpha}{2} \|\nabla f(x_k)\|^2 + \frac{2\alpha}{5} \|\nabla f(x_k)\|^2 +  \frac{\alpha^2}{100}\kappa_g \|\nabla f(x_k)\|^2 \\
    & \leq f(x_k) -  \frac{\alpha}{2} \|\nabla f(x_k)\|^2 + \frac{2\alpha}{5} \|\nabla f(x_k)\|^2 +  \frac{\alpha}{100}\|\nabla f(x_k)\|^2 \\
    & = f(x_k) -  0.09 \alpha \|\nabla f(x_k)\|^2 \\
    & \leq f(x_k) -  2.25 \frac{\varepsilon^2}{\alpha}
    \end{aligned}
\end{equation}
Therefore, for all $x$ s.t. $\|\nabla f(x)\| \geq \frac{5\varepsilon}{\alpha}$ we have a monotonically decreasing step with finite size improvement, hence after a finite number of steps we obtain $x^{\star}$ for which $\|\nabla f(x^{\star})\| \leq \frac{5\varepsilon}{\alpha}$.
\end{proof}

\newpage

\section{The Perturbed Mean-Gradient}
\label{sec:perturbed}

\begin{definition}
The perturbed mean-gradient in $x$ with averaging radius $\varepsilon > 0$ and perturbation radius $p < \varepsilon$ is
\begin{equation}
\label{mean_gradient}
    g_{\varepsilon}^p(x) = \arg\min_{g\in\mathbb{R}^n} \iint\displaylimits_{V_{\varepsilon}(x)B_{p}(x)}|g\cdot (\tau - s) - f(\tau) + f(s)|^2  ds d\tau
\end{equation}
where $B_{p}(x)\subset V_{\varepsilon}(x)\subset\mathbb{R}^n$ are convex subsets s.t. $\|x' - x\|\leq \varepsilon$ for all $x'\in V_{\varepsilon}(x)$ and the integral domain is over $\tau\in V_{\varepsilon}(x)$ and $s\in B_{p}(x)$.
\end{definition}

We denote $V_{\varepsilon}(x)$ as the averaging domain and $B_{p}(x)$ as the perturbation domain and usually set $|V_{\varepsilon}| \gg |B_{p}|$. The purpose of $V_{\varepsilon}$ is to average the gradient in a region of radius $\varepsilon$ and the perturbation is required to obtain smooth gradients around discontinuity points.

\begin{prop}[controllable accuracy]
For any function $f\in\mathcal{C}^1$, there is $\kappa_g > 0$, so that for any $\varepsilon > 0$ the perturbed mean-gradient satisfies $\|g_{\varepsilon}^p -\nabla f(x)\| \leq \kappa_g \varepsilon$ for all $x\in\Omega$.   
\end{prop}

\begin{proof}

Recall the Taylor theorem for a twice differentiable function $f(\tau)=f(s) + \nabla f(s) \cdot (\tau - s) + R_{s}(\tau)$, where $R_{s}(\tau)$ is the reminder. Since the gradient is continuous, we can write
\begin{equation}
    f(\tau) = f(s) + \nabla f(s) \cdot (\tau - s) + \int_0^1 (\nabla f(s + t(\tau-s)) - \nabla f(s)) \cdot (\tau - s) dt
\end{equation}
Since $f\in\mathcal{C}^1$, we also have $|\nabla f(x) - \nabla  f(s)|\leq \kappa_{f}\|x-s\|$. We can use this property to bound the reminder in the Taylor expression.
\begin{equation}
\begin{aligned}
    R_{s}(\tau) &= \int_0^1 (\nabla f(s + t(\tau-s)) - \nabla f(s)) \cdot (\tau - s) dt \\ 
    &\leq \kappa_f \int_0^1 \|s + t(\tau-s) - s\| \cdot \|\tau - s\| dt \leq \frac{\kappa_f}{2}  \|\tau - s\|^2
\end{aligned}
\end{equation}

By the definition of $g_{\varepsilon}^p$, an upper bound for $\mathcal{L}(g_{\varepsilon}^p(x))$ is

\begin{align*}
    \mathcal{L}(g_{\varepsilon}(x)) \leq \mathcal{L}(\nabla f(x)) &= \iint\displaylimits_{V_{\varepsilon}(x)B_{p}(x)}|\nabla f(x) \cdot (\tau -s) - f(\tau) + f(s)|^2 ds d\tau \\
    &= \iint\displaylimits_{V_{\varepsilon}(x)B_{p}(x)}| (\nabla f(x) - \nabla f(s) + \nabla f(s)) \cdot (\tau -s) - f(\tau) + f(s)|^2 ds d\tau \\
    &\leq \iint\displaylimits_{V_{\varepsilon}(x)B_{p}(x)}|\nabla f(s) \cdot (\tau -s) - f(\tau) + f(s)|^2 ds d\tau \\
    & + 2\iint\displaylimits_{V_{\varepsilon}(x)B_{p}(x)}\|\nabla f(x) - \nabla f(s)\|\cdot\|\tau - s\|\cdot|\nabla f(s) \cdot (\tau -s) - f(\tau) + f(s)| ds d\tau \\
    & + \iint\displaylimits_{V_{\varepsilon}(x)B_{p}(x)}\|\nabla f(x) - \nabla f(s)\|^2\cdot \|\tau - s\|^2  ds d\tau \\
    & \leq \iint\displaylimits_{V_{\varepsilon}(x)B_{p}(x)} \frac{\kappa_f^2}{4} \| \tau - s\|^4 + \kappa_f^2 \|x-s\| \cdot \| \tau - s\|^3 + \kappa_f^2 \|x-s\| \cdot \| \tau - s\|^2 ds d\tau \\
    & \leq 16\kappa_f^2 \varepsilon^4 |V_{\varepsilon}(x)||B_{p}(x)| = 16\kappa_f^2 \varepsilon^{n+4} p^n |V_{1}(x)||B_{1}(x)| 
\end{align*}

Noticed that we used the inequalities: (1) $\|\nabla f(x) - \nabla  f(s)\|\leq \kappa_{f}\|x-s\|\leq \kappa_{f}\varepsilon$, (2) $\|x-s\|\leq \varepsilon$; and (3) $\|\tau-s\|\leq 2\varepsilon$.

For the lower bound we assume that $p=\varepsilon\bar{p}$ and $\bar{p} < \frac{1}{4}$. Note that for any other upper bound on $\bar{p}$ we can derive an alternative bound.

The lower bound is

\begin{align*}
    \mathcal{L}(g_{\varepsilon}(x)) &= \iint\displaylimits_{V_{\varepsilon}(x)B_{p}(x)}|g_{\varepsilon}(x) \cdot (\tau -s) - f(\tau) + f(s)|^2 ds d\tau \\
    &= \iint\displaylimits_{V_{\varepsilon}(x)B_{p}(x)}|(g_{\varepsilon}(x) - \nabla f(x) + \nabla f(x))  \cdot (\tau -s) - f(\tau) + f(s)|^2 ds d\tau \\
    &\geq \iint\displaylimits_{V_{\varepsilon}(x)B_{p}(x)}|(g_{\varepsilon}(x) - \nabla f(x)) \cdot (\tau -s)|^2 -2 |(g_{\varepsilon}(x) - \nabla f(x)) \cdot (\tau -s)| \cdot |\nabla f(x)   - f(\tau) + f(s)| ds d\tau \\
    & \geq \iint\displaylimits_{V_{\varepsilon}(x)\setminus V_{\frac{3\varepsilon}{4}}(x) B_{p}(x)}|(g_{\varepsilon}(x) - \nabla f(x)) \cdot (\tau -s)|^2 s d\tau\\
    &- 4\varepsilon \|(g_{\varepsilon}(x) - \nabla f(x))\| \iint\displaylimits_{V_{\varepsilon}(x)B_{p}(x)} \left(|\nabla f(s) - f(\tau) + f(s)| + |\nabla f(x) - \nabla f(s)|\right) ds d\tau \\
    & \geq \|g_{\varepsilon}(x) - \nabla f(x)\|^2 \left(\frac{\varepsilon}{2}\right)^2\iint\displaylimits_{V_{\varepsilon}(x)\setminus V_{\frac{3\varepsilon}{4}}(x) B_{p}(x)} \left|\hat{\mathbf{n}}(x) \cdot \frac{\tau -s}{\|\tau -s\|}\right|^2 s d\tau\\
    &- 4\varepsilon \|(g_{\varepsilon}(x) - \nabla f(x))\| \iint\displaylimits_{V_{\varepsilon}(x)B_{p}(x)} \frac{1}{2}\kappa_f\|\tau-s\|^2 + \kappa_f \|x-s\|^2 ds d\tau \\
    & \geq \|g_{\varepsilon}(x) - \nabla f(x)\|^2 \varepsilon^n p^n \left(\frac{\varepsilon}{2}\right)^2\iint\displaylimits_{V_{1}(x)\setminus V_{\frac{3}{4}}(x) B_{1}(x)} \left|\hat{\mathbf{n}}(x) \cdot \frac{\tau -s}{\|\tau -s\|}\right|^2 s d\tau - 2.5 \varepsilon \|(g_{\varepsilon}(x) - \nabla f(x))\| \varepsilon^n p^n \kappa_f \frac{27}{32}\varepsilon^2 |V_{1}(x)||B_{1}(x)|  \\
    & =  \frac{1}{4} \varepsilon^{n+2} p^n  M_{1} \|(g_{\varepsilon}(x) - \nabla f(x))\|^2 - 2.5 \varepsilon^{n+3} p^n \kappa_f \|(g_{\varepsilon}(x) - \nabla f(x))\| |V_{1}(x)||B_{1}(x)|  \\
\end{align*}

Combining the lower and upper bound we obtain
\begin{align*}
    & M_{1} \|(g_{\varepsilon}(x) - \nabla f(x))\|^2 - 10 \varepsilon \kappa_f \|(g_{\varepsilon}(x) - \nabla f(x))\| |V_{1}(x)||B_{1}(x)| - 64 \kappa_f^2 \varepsilon^2 |V_{1}(x)||B_{1}(x)| \leq 0  \\
    \Rightarrow \ \ & \|(g_{\varepsilon}(x) - \nabla f(x))\| \leq \kappa_g \varepsilon
\end{align*}

where 
\begin{equation*}
    \kappa_g = \kappa_f \frac{10|V_{1}(x)||B_{1}(x)| + \sqrt{100|V_{1}(x)|^2|B_{1}(x)|^2 + 256 |V_{1}(x)||B_{1}(x)| M_1(x)}}{M_1(x)}
\end{equation*}

\end{proof}

\begin{prop}[continuity]
If $f(x)$ is Riemann integrable in  $V_{\varepsilon}(x)\subset V$ then the perturbed mean-gradient is a continuous function at $x$.
\end{prop}

\begin{proof}

We follow the same line of arguments as in Proposition \ref{prop_continuity}, yet here we need to show that $\mathcal{L}(x,g)$ is continuous for any interable function $f$.

\begin{equation*}
\begin{aligned}
        \left|\mathcal{L}(x,g) - \mathcal{L}(x',g)\right| &= \left|g\cdot \mathbf{A}(x) g + g\cdot b(x) - g\cdot \mathbf{A}(x') g - g\cdot b(x)\right| = \left|g\cdot b(x) -  g\cdot b(x)\right| \\
        & \leq \|g\| \left\| \quad \iint\displaylimits_{V_{\varepsilon}(x)B_{p}(x)}|g\cdot (\tau - s) - f(\tau) + f(s)|^2  ds d\tau - \iint\displaylimits_{V_{\varepsilon}(x')B_{p}(x')}|g\cdot (\tau - s) - f(\tau) + f(s)|^2  ds d\tau \right\| \\
        & \leq \|g\| \iint\displaylimits_{V_{\varepsilon}(x)B_{p}(x)\setminus V_{\varepsilon}(x')B_{p}(x')}|g\cdot (\tau - s) - f(\tau) + f(s)|^2  ds d\tau \\
        & + \|g\| \iint\displaylimits_{V_{\varepsilon}(x')B_{p}(x') \setminus V_{\varepsilon}(x)B_{p}(x)}|g\cdot (\tau - s) - f(\tau) + f(s)|^2  ds d\tau \\
        & \leq M \|g\| \cdot |V_{\varepsilon}(x)B_{p}(x)\setminus V_{\varepsilon}(x')B_{p}(x')| + M \|g\| \cdot  |V_{\varepsilon}(x')B_{p}(x')\setminus V_{\varepsilon}(x)B_{p}(x)|
\end{aligned}
\end{equation*}
Applying the same arguments in Lemma \ref{ap_a_eps}, we have $|V_{\varepsilon}(x)B_{p}(x)\setminus V_{\varepsilon}(x')B_{p}(x')|\leq |A_{V_{\varepsilon}}(x)| \cdot |A_{B_p}(x)| \|x-x'\|^2$, where $A_{V_{\varepsilon}}(x)$ is the surface of $V_{\varepsilon}(x)$ and $A_{B_p}(x)$ is the surface of $B_{p}(x)$. Therefore, $\mathcal{L}(x,g)$ is continuous.

The rest of the proof, again, is identical to Proposition \ref{prop_continuity}.

\end{proof}

\subsection{Monte-Carlo approximation of the perturb mean-gradient}

For a parameterization $g_{\theta}$, we may learn the perturb mean-gradient by sampling a reference point $x_r$ and then uniformly sampling two evaluation points $x_i\sim U(B_{p}(x_r))$ $x_j\sim U(V_{\varepsilon}(x_r))$. With the tuples $(x_r,x_i,x_j)$ we minimize the following loss
\begin{equation*}
    \mathcal{L}_{\varepsilon,p}(\theta)=\sum_{x_r}\sum_{x_i}\sum_{x_j} |g_{\theta}(x_r)\cdot (x_j-x_i) - f(x_j) + f(x_i)|^2
\end{equation*}
Since $x_i\sim U(B_{p}(x_r))$, we can write $x_i = x_r + n_i$ where $n_i$ is uniformly sampled in an $n$-ball with $p$ radius. To reduce the number of evaluation points, we may choose to fix $x_i$ and sample $x_r = x_i + n_r$. If we assume that $\varepsilon \gg p$ then for a sample  $x_j\sim U(V_{\varepsilon}(x_i))$ with very high probability we have that $\|x_r - x_j\| \leq \varepsilon$. So we can approximate $\mathcal{L}_{\varepsilon,p}$ with
\begin{equation*}
    \mathcal{L}_{\varepsilon,p}(\theta)=\sum_{n_r}\sum_{x_i}\sum_{x_j} |g_{\theta}(x_i + n_r)\cdot (x_j-x_i) - f(x_j) + f(x_i)|^2
\end{equation*}

\newpage

\section{Spline Embedding}
\label{sec:SplineEmb}

When fitting $f$ with a NN, we found out that feeding the input vector $x$ directly into a Fully Connected NN provides unsatisfactory results when the dimension of the data is too small or when the target function is too complex. Specifically, gradient descent (with Adam optimizer \cite{kingma2014adam}) falls short in finding the global optimum. We did not investigate theoretically into this phenomena, but we designed an alternative architecture  that significantly improves the learning process. This method adds a preceding embedding layer \cite{zhang2016deep} before the NN. These embeddings represent a set of learnable Spline functions \cite{reinsch1967smoothing}.

Categorical Feature embedding \cite{howard2020fastai} is a strong, common practice, method to learn representations of multi-categorical information. It is equivalent to replacing the features  with their corresponding one-hot vector representation and concatenating the one-hot vectors into a single vector which is then fed to the input of a NN. An important advantage of categorical embedding is the ability to expand the input dimension into an arbitrary large vector size. In practice, this expansion can help in representing complex non-linear problems. 

For ordinal data, however, embedding may be viewed as an unnecessary step as one can feed the data directly into a NN input layer. Moreover, categorical feature embeddings do not preserve ordinality within each categorical variable as each class is assigned a different independent set of learnable embeddings. Nevertheless, motivated by the ability to expand the input dimension into an arbitrary large number, we designed an ordinal variable embedding that is Lipschitz continuous s.t. for two relatively close inputs $x_1$ and $x_2$ the embedding layer outputs $s(x_1)$ and $s(x_2)$ s.t. $\|s(x_1)- s(x_2)\|\leq \kappa_s \|x_1 - x_2\|$. To that end, for a given input vector $x\in\mathbb{R}^n$, we define the representation as $s_{\theta}:\mathbb{R}^n\to\mathbb{R}^{n_s}$, $y=s_{\theta}(x)$, where each entry $s^{j}_{\theta}(x^l)$ is a one-dimensional learnable Spline transformation. A Spline \cite{reinsch1967smoothing} is a piecewise polynomial with some degree of smoothness in the connection points. Spline is usually used to approximate smooth functions but here we use it to represent a learnable function.

To define a learnable spline, we need to determine the intersection points and the spline degree. Specifically, for a domain $x^i\in[a,b]$ we equally divided the domain into $k$ intersection points, where each point is also termed as knot (in this work $[a,b]=[-1,1]$ and $k=21$, s.t. each segment is 0.1 long). Our next step is to define the spline degree and smoothness. We experimented with three options: (1) continuous piecewise linear splines (2) \nth{3} degree polynomials with continuous second derivative, termed $C^2$ Cubic spline and; (3) continuous $C^0$ Cubic splines. We found out that for the purpose of EGL, continuous piecewise linear splines yield the best performance and requires less computational effort. The explicit definition of a piecewise linear spline is
\begin{equation}
    s(x,\theta) = \frac{\theta_i}{h_i}(x-t_{i-1}) + \frac{\theta_i}{h_i}(t_i-x)
\end{equation}
where $\theta$ is a $k$ elements ($k$ is the number of knots), $t_i$ is the location of the $i$-th knot and $h_i=t_i-t_{i-1}$.

We can learn more than a single spline for each element in the $x$ vector. In this work we learned $e$ different splines for each entry in $x$ s.t. the output shape of the embedding block is $n\times e$. It is also possible to learn two or more dimensional splines but the number of free parameters grows to the power of the splines dimensions. Therefore, it is non practical to calculate these high degree splines. To calculate interactions between different entries of $x$ we tested two different methods: (1) aggregation functions and; (2) attention aggregation after a non-local blocks \cite{wang2018non}. 

In the first option, given a spline representation $s(x)\in\mathbb{R}^{n\times e}$ an average pooling aggregation is executed along the \nth{1} dimension s.t. we end up with a $\bar{s}(x)\in\mathbb{R}^e$ representation vector. In the second option, the aggregation takes place after a non-local blocks which calculates interactions between each pairs of entries in the \nth{1} dimension of $s(x)$ (i.e. the input dimension). To preserve the information of the input data, we concatenated $x$ to the output of the aggregation layer. After the concatenation, stacks of Residual blocks \cite{he2016deep} (Res-Blocks) layers have been applied to calculate the output vector (size of 1 in IGL and size of $n$ in EGL). The complete Spline Embedding architecture that includes both average pooling aggregation and non-local blocks is presented in Fig. \ref{fig:spline_net}.

\newpage

\begin{figure}[ht!]
  \centering
  \includegraphics[width=1\linewidth]{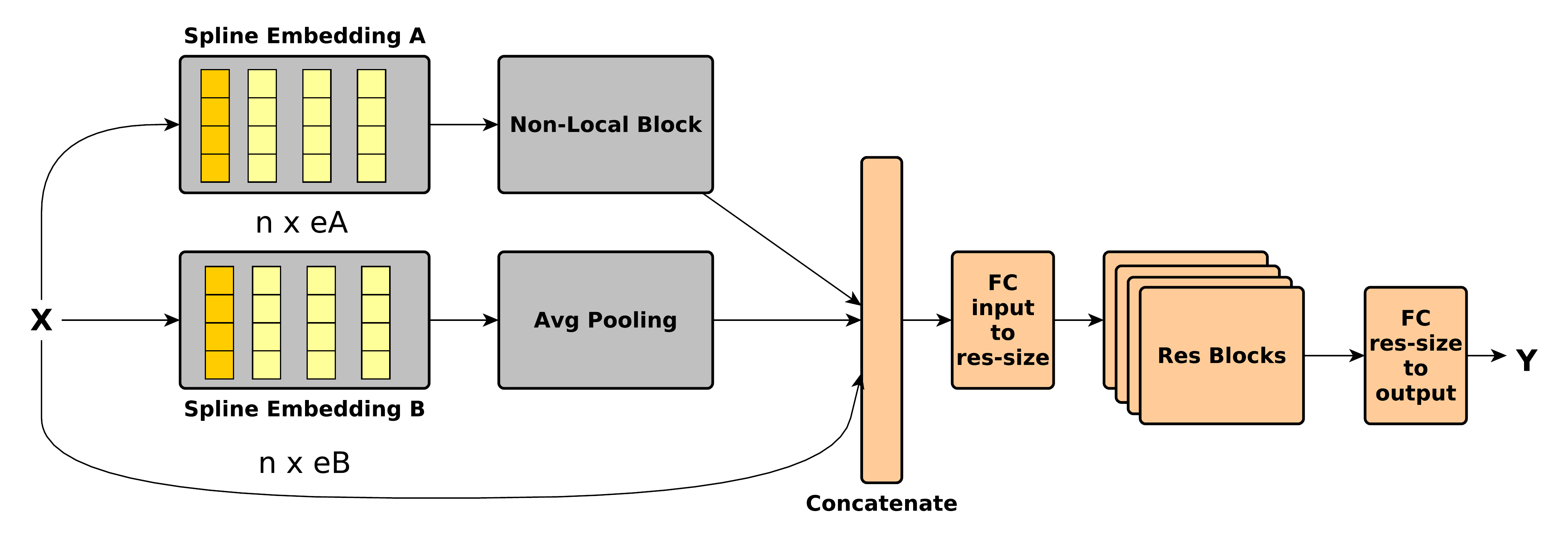}
  %\vspace{-20pt}
  \caption{The Spline Architecture}
\label{fig:spline_net}
\end{figure}

\begin{figure}[ht!]
  \centering
  \includegraphics[width=1\linewidth]{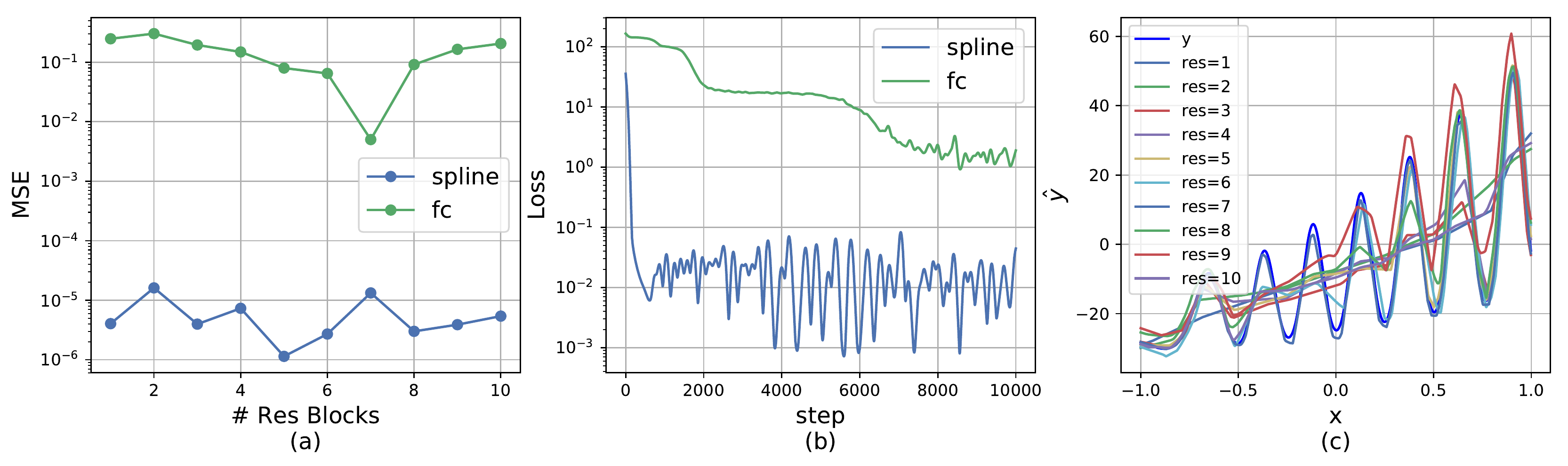}
  %\vspace{-20pt}
  \caption{Comparing Spline fitting vs standard FC fitting.}
\label{fig:spline_vs_fc_p246}
\end{figure}

\begin{figure}[ht!]
  \centering
  \includegraphics[width=1\linewidth]{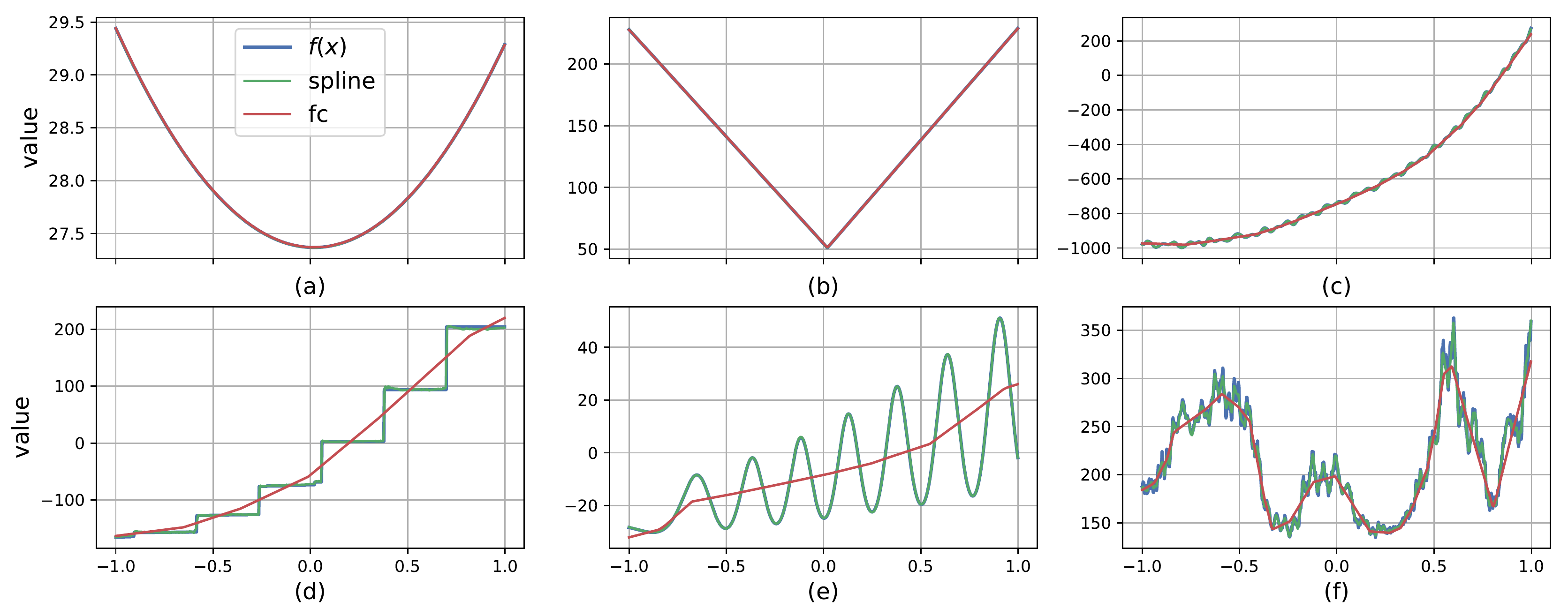}
  %\vspace{-20pt}
  \caption{Comparing Spline fitting vs standard FC fitting.}
\label{fig:spline_vs_fc}
\end{figure}

\newpage

In the next set of experiments, we evaluate the benefit of Spline embedding in 1D COCO problems. We compared the Spline architecture in Fig. \ref{fig:spline_net} to the same architecture without the spline embedding branches ($x$ is directly fed to the FC layer input). We used only $e=8$ splines and a Res-Block layer size of 64. In Fig. \ref{fig:spline_vs_fc_p246}(a), we evaluate the learning of a single problem (246, harmonic decaying function) with different number of Res-Blocks. Here, we used a mini-batch size of 1024 and a total of 1024 mini-batches iteration to learn the function (i.e. a total of $10^6$ samples). We see that Spline embedding obtains much better MSE even for a single Res-Block and maintains its advantage for all the Res-Block sizes which we evaluated. Note that each Res-Block comprises two Fully Connected layers, thus with the additional input and output layers we have $2n+2$ FC layers for $n$ Res-Blocks.

In Fig. \ref{fig:spline_vs_fc_p246}(b) we evaluated the learning process with 2 Res-Blocks for 10240 mini-batches ($10^7$ samples). We see that Spline embedding converges after roughly 500 minibatches while the FC layer learns very slowly. Interestingly, each significant drop in the loss function of the FC net corresponds to a fit of a different ripple in the harmonic decaying function. It seems like the FC architecture converges to local minima that prevent the network from fitting the entire harmonic function. This can be seen in Fig. \ref{fig:spline_vs_fc_p246}(c) where we print the results of the learned FC models for different Res-Block sizes after 1024 mini-batches. The results show that all FC networks fail to fit the harmonic function completely.

To demonstrate expressiveness of Spline embedding, we fit the 4 functions in the 1-D illustrative examples in Sec. \ref{sec:Motivation} and two additional functions: the harmonic decaying function and a noise like function. The results are presented in Fig. \ref{fig:spline_vs_fc}. Remarkably, while we use only $e=8$ splines which sums up to only 680 additional weights ($8\times21$ spline parameters and additional $8\times64$ input weights), we obtain significantly better results than the corresponding FC architecture.\footnote{In 1D problems there is no aggregation step.}

\newpage

\section{Mappings}
\label{sec:Mappings}

By applying the chain rule and the inverse function theorem, we can express the gradient of the original problem $\nabla f$ with the gradient of the scaled problem $\nabla \tilde{f}$:
\begin{equation}
    \frac{\partial f(x)}{\partial x^l}  =  \left(\frac{\partial r_k}{\partial y}\right)^{ -1}  \frac{\partial h_j(x)}{\partial x^l} \frac{\partial\tilde{f}_{jk}(\tilde{x})}{\partial \tilde{x}^l} 
\end{equation}
Here, $\partial x^l$ is the partial derivative with respect to the $l$-th entry of $x$ (we assume that $h$ maps each element independently s.t. the Jacobian of $h$ is diagonal). For strictly linear mappings, it is easy to show that this property also holds for the mean-gradients.

\begin{prop}
Let $h_j:\mathbb{R}^n\to\mathbb{R}^n$ and $r_k:\mathbb{R}\to\mathbb{R}$ be two linear mapping functions s.t., $r_k(y) = \frac{y-\mu_k}{\sigma_k}$ and $h_j^l(x)=a_j^l x + b_j^l$, then the mean-gradient $g_{\varepsilon}$ of $f$ can be recovered from the mean-gradient $\tilde{g}_{\tilde{\varepsilon}}$ of $\tilde{f}$ with
\begin{equation}
\label{linear_mapping}
    g_{\varepsilon}^l(x) =  \frac{a_j^l}{\sigma_k} \tilde{g}_{\tilde{\varepsilon}}^l(\tilde{x})
\end{equation}
where $V_{\varepsilon}$ is the projection $h_j^{-1}(V_{\tilde{\varepsilon}})$ which is bounded by an $n$-ball at $x$ with radius $\varepsilon=\max_{l}\frac{1}{a_j^l}\tilde{\varepsilon}$, i.e. for all $x'\in V_{\varepsilon}(x), \|x'-x\|\leq \max_{l}\frac{1}{a_j^l}\tilde{\varepsilon}$.
\end{prop}

\begin{proof}
Let us write the definition of $g_{\tilde{\varepsilon}}$ with a variable $\tilde{\tau}$ s.t. $\tilde{\tau}\in V_{\tilde{\varepsilon}}(\tilde{x})$ (contrary to the original definition where $\tau$ denoted the difference s.t. $x+\tau\in V_{\varepsilon}(x)$)
\begin{equation}
    g_{\tilde{\varepsilon}}(\tilde{x}) = \arg\min_{g}\int_{\tilde{\tau}\in V_{\tilde{\varepsilon}(\tilde{x})}} |g\cdot(\tilde{\tau} - \tilde{x}) - \tilde{f}(\tilde{\tau}) + \tilde{f}(\tilde{x})|^2 d\tilde{\tau}
\end{equation}
Recall the mapping $\tilde{x}=h(x)$, since it is invertable mapping, there exist $\tau$ s.t. $\tilde{\tau}=h(\tau)$. Substituting $\tilde{\tau}$ with $\tau$, the integral becomes
\begin{equation}
    g_{\tilde{\varepsilon}}(\tilde{x}) = \arg\min_{g}\int_{\tau\in h^{-1}(V_{\tilde{\varepsilon}(\tilde{x})})} |g\cdot(h(\tau) - h(x)) - \tilde{f}(h^{-1}(\tau)) + \tilde{f}(h^{-1}(x))|^2 |\det(Dh(\tau))| d\tau
\end{equation}
where $\det(Dh(\tau))$ denotes the determinant of the Jacobian matrix of the mapping $h$. This determinant is constant for linear mapping so we can ignore it as we search for the arg-min value. We can also multiply the integral by the inverse slope $\frac{1}{\sigma_r}$ and get
\begin{equation}
\begin{aligned}
    g_{\tilde{\varepsilon}}(\tilde{x}) & = \arg\min_{g}\int_{\tau\in h^{-1}(V_{\tilde{\varepsilon}(\tilde{x})})} |\frac{1}{\sigma_r}g\cdot(h(\tau - x)) - \frac{1}{\sigma_r} \tilde{f}(h^{-1}(\tau)) + \frac{1}{\sigma_r} \tilde{f}(h^{-1}(x))|^2 d\tau \\
    & =  \arg\min_{g}\int_{\tau\in h^{-1}(V_{\tilde{\varepsilon}(\tilde{x})})} |\frac{1}{\sigma_r}a_j \odot g \cdot(\tau - x) - r^{-1}(\tilde{f}(h^{-1}(\tau))) + r^{-1}(\tilde{f}(h^{-1}(x)))|^2 d\tau \\
    &= \arg\min_{g}\int_{\tau\in h^{-1}(V_{\tilde{\varepsilon}(\tilde{x})})} |\frac{1}{\sigma_r}a_j \odot g \cdot(\tau - x) - f(\tau) + f(x)|^2 d\tau
\end{aligned}
\end{equation}
Where the last equality holds since $r^{-1}\circ \tilde{f} \circ h = r^{-1}\circ r \circ f \circ h^{-1} \circ h = f$. Since the mapping $g\to\frac{1}{\sigma_r}a_j\odot g$ is bijective, the arg-min can be rephrased as
\begin{equation}
    \frac{1}{\sigma_r}a_j \odot g_{\tilde{\varepsilon}}(\tilde{x}) = \arg\min_{g}\int_{\tau\in h^{-1}(V_{\tilde{\varepsilon}(\tilde{x})})} | g \cdot(\tau - x) - f(\tau) + f(x)|^2 d\tau
\end{equation}
which is exactly the definition for $g_{\varepsilon}$ so we get that $ g_{\varepsilon} = \frac{1}{\sigma_r}a_j \odot g_{\tilde{\varepsilon}}(\tilde{x})$, as requested. Finally, we need to show that for all $\tau\in V_{\varepsilon}(x)$, $\|\tau-x\|\leq \max_{l}\frac{1}{a^l}\tilde{\varepsilon}$.
\begin{equation}
    \|\tau - x\| = \|h^{-1}(\tilde{\tau}) - h^{-1}(\tilde{x})\| = \|h^{-1}(\tilde{\tau} - \tilde{x})\| \\
    \leq  \|h^{-1}\| \|\tilde{\tau} - \tilde{x}\| \leq \max_{l}\frac{1}{a^l}\tilde{\varepsilon}
\end{equation}

\end{proof}

As discussed in Sec. \ref{subsec:DynamicScaling}, the design goals for mappings are twofold: (1) Fix the statistics of the input and output data and; (2) maintain the linearity as much as possible. Following these two goals we explored mappings of the form $y=q(l_{\mathbf{a}}(x))$, where $q$ is an expansion non-linear mapping $\Omega\to\mathbb{R}^n$ for the input mapping and a squash mapping $\mathbb{R}\to\mathbb{R}$ for the output mapping. $l_{\mathbf{a}}$ is a linear mapping that is defined by the $\mathbf{a}$ parameters. For example in the scalar case we can uniquely define the linear function by mapping $x_1$ to $y_1$ and $x_2$ to $y_2$, in this case we denote $\mathbf{a}=[(x_1, y_1), (x_2, y_2)]$.

\subsection{Input Mapping}

Given a candidate solution $x_{j-1}$, we first construct a bounding-box $\Omega_j$ by squeezing the previous region by a factor of $\gamma_{\alpha}$ and placing it s.t. $x_{j-1}$ is in the bounding-box center. For a region $\Omega_j$ such that the upper and lower bounds are found in $[b_{l}, b_{u}]$, we, first, construct a linear mapping of the form $\mathbf{a}=[(b_{l}, -1), (b_{u}, 1)]$. Then, our expansion function is the inverse hyperbolic tangent $\arctanh(x)=\frac{1}{2}\log\left(\frac{1+x}{1-x}\right)$. This function expands $[-1,1]\to\mathbb{R}$ but maintains linearity at the origin. Given that the solution is approximately found in the center of the bounding-box we obtain high linearity except when the solution is found on the edges.   

\subsection{Output Mapping}

For the output mapping we first fix the statistics with a linear mapping $\mathbf{a}=[(Q_{0.1}, -1), (Q_{0.9}, 1)]$ where $Q_{0.1}$ is the $0.1$ quantile in the data and $Q_{0.9}$ is the $0.9$ quantile. This mapping is also termed as robust-scaling as unlike $z$-score $\frac{x - \mu}{\sigma}$, it is resilient to outliers. On the downside it does not necessarily fix the first and second order statistics, but these are at least practically, bounded. The next step, i.e. squash mapping, makes sure that even outliers does not get too high values. For that purpose, we use the squash mapping
\begin{align}
        q(x) = \begin{cases}
            -\log(-x) - 1,   & \qquad x < -1\\
            x, & -1 \leq x < 1\\
            \log(x) + 1, & \qquad\  \ x \geq 1
            \end{cases}
\end{align}

% \begin{prop}
% Let $h_j:\mathbb{R}^n\to\mathbb{R}^n$ and $r:\mathbb{R}\to\mathbb{R}$ be two linear mapping functions s.t., $r(y) = \frac{y-\mu_r}{\sigma_r}$ and $h^l(x)=a^l x + b^l$, then the mean-gradient $g_{\varepsilon}$ of $f$ can be recovered from the mean-gradient $\tilde{g}_{\tilde{\varepsilon}}$ of $\tilde{f}$ with
% \begin{equation}
%     g_{\varepsilon}^l(x) =  \frac{a^l}{\sigma_k} \tilde{g}_{\tilde{\varepsilon}}^l(\tilde{x})
% \end{equation}
% where $V_{\varepsilon}$ is the projection $h^{-1}(V_{\tilde{\varepsilon}})$ which is bounded by a $n$-ball with radius $\varepsilon=\max_{l}\frac{1}{a^l}\tilde{\varepsilon}$.
% \end{prop}

\newpage

\section{The Practical EGL Algorithm}
\label{sec:egl_alg}

\begin{algorithm}[]
\small{
\DontPrintSemicolon
\KwIn{$x_0$, $\Omega$, $\tilde{\alpha}$, $\tilde{\varepsilon}$, $\gamma_{\alpha} < 1$, $\gamma_{\varepsilon} < 1$, $n_{\max}$}
% \SetAlgoLined
\ \ $k=0$ \\
$j=0$ \\
$\Omega_j\leftarrow \Omega$ \\
Map $h_0:\Omega\to\mathbb{R}^n$

\While{budget $C > 0$}{

\SetKwProg{Fn}{Explore:}{}{end}
\Fn{}{
\ \ Generate samples $\mathcal{D}_k = \{\tilde{x}_i\}_{1}^m$, $\tilde{x}_i\in V_{\tilde{\varepsilon}}(\tilde{x}_k)$ \\
Evaluate samples $y_i=f(h_0^{-1}(\tilde{x}_i))$, \qquad $i=1,...,m$ \\
Add samples to the replay buffer $\overline{\mathcal{D}}=\overline{\mathcal{D}}\cup\mathcal{D}_k$
}

\SetKwProg{Fn}{Output Map:}{}{end}
\Fn{}{
\ \ $r_k = squash \circ l_{[Q_{0.1},Q_{0.9}]}$ \\
$\tilde{y}_i = r_k(y_i)$ , \qquad $i=1,...,m$ \\ 
}

\SetKwProg{Fn}{Mean-Gradient learning:}{}{end}
\Fn{}{
\ \ $\theta_k = \arg\min_{\theta}  \sum_{q=0}^{l-1} \sum_{i,j \in \mathcal{D}_{k-q}} |(\tilde{x}_j - \tilde{x}_i) \cdot g_{\theta}(\tilde{x}_i) - \tilde{x}_j + \tilde{x}_i |^2$ \\
}

\SetKwProg{Fn}{Gradient Descent:}{}{end}
\Fn{}{
$x_{k+1} \leftarrow x_k - \tilde{\alpha} g_{\theta_k}(x_k)$ \\
\If{$f(h_j^{-1}(\tilde{x}_{k+1})) > f(h_j^{-1}(\tilde{x}_{k}))$ for $n_{\max}$ times in a row}{
\ \ Generate new trust-region s.t. $|\Omega_{j+1}| = \gamma_{\alpha}|\Omega_j|$ and its center at $x_{best}$\\
Map $h_j:\Omega\to\mathbb{R}^n$ \\
$j \leftarrow j + 1$ \\
$\tilde{\varepsilon} \leftarrow \gamma_{\varepsilon}\tilde{\varepsilon}$ \\
}
\If{$f(h_j^{-1}(\tilde{x}_{k+1})) < f(h_j^{-1}(\tilde{x}_{k}))$ }{
 \ \ $x_{best} = h_j^{-1}(\tilde{x}_k)$ \\
}
}

$k \leftarrow k + 1$ \\
}
}
\KwRet{$x_{best}$}
\caption{Explicit Gradient Learning}
\label{alg:practical_egl}
\end{algorithm}

\newpage

\section{Supplementary details: The COCO experiment} 
\label{sec:coco_exp}

The COCO test suite provides many Black-Box optimization problems on several dimensions (2,3,5,10,20,40). For each dimension, there are 360 distinct problems. The problems are divided into 24 different classes, each contains 15 problems.
To visualize all problem classes, we iterate over the 2D problem set and for each class we present (Fig. \ref{fig:2D_P0}-\ref{fig:2D_P5}) a contour plot, 3D plot and the equivalent 1D problem ($f_{1D}(x)=f_{2D}(x,x)$) combined with the log view of the normalized problem ($\frac{f_{1D}(x) - f_{1D}^min}{f_{1D}^{\max}- f_{1D}^{\min}}$).

To visualize the average convergence rate of each method, we first calculate a scaled distance between the best value at time $t$ and the optimal value $\Delta y_{best}^t = \frac{ \min_{k\leq t} y_k - y^*}{y_0 - y^*}$ where $y^*$ is the minimal value obtained from all the baselines' test-runs. We then average this number, for each $t$, over all runs in the same dimension problem set. This distance is now scaled from zero to one and the results are presented on a log-log scale. The first-column in Fig. \ref{fig:2D_P0}-\ref{fig:2D_P5} presents $\overline{\Delta y}_{best}^t$ on each problem type of the 2D problem set and Fig. \ref{fig:40D} show $\overline{\Delta y}_{best}^t$ in the 40D problem set. 
In table \ref{table:coco_hyperparams}, we present the hyperparameters used for EGL and IGL in all our experiments (besides the ablation tests).

In future work, we will need to design a better mechanism for the $\varepsilon$ scheduling. In problem 19 (Griewank-Rosenbrock F8F2), the $\varepsilon$ scheduling was too slow, and only when we used a smaller initial $\varepsilon$, EGL started to converge to the global minimum (see Fig. \ref{fig:epsilon_40D}(a) where we used initial $\varepsilon = 0.001\times \sqrt{n}$, $\gamma_{\alpha}=0.7$ and $L=1$). On the other hand, in problems, 21 (Gallagher 101 peaks) and 22 (Gallagher 21 peaks) using small $\varepsilon$ ends up in falling to local minima, and the choice of a larger $\varepsilon$ could smooth the gradient which pushes $x_k$ over the local minima (see \ref{fig:epsilon_40D}(b-c) respectively where we used initial $\varepsilon = 0.5\times \sqrt{n}$ and $L=4$). 

In Fig. \ref{fig:hist_P1}-\ref{fig:hist_P6} we present a histogram of the raw and scaled cost value (after the output-mapping) of a 200 samples snapshot from the replay buffer at different periods during the learning process ($t=1K$, $t=10K$, $t=100K$). Typically, we expect that problems with Normal or Uniform distributions should be easier to learn with a NN (e.g. problems 15 (Rastrigin), 18 (Schaffer F7, cond1000), 23 (ats ras)), while problems with skewed distribution or multimodal distribution are much harder (e.g. problems 2 (Ellipsoid separable), 10 (Ellipsoid) and 11 (Discus)). However, simply, mapping from a hard distribution into a Normal distribution is not necessarily a good choice since we lose the mapping linearity s.t. the scaled mean-gradient may not correspond to the true mean-gradient. Thus, the output-mapping must balance between linearity and normalization. In future work, we would like to find better, more robust output-mappings that overcome this problem. Understanding the way that the values are distributed at run-time could also help us define a better mechanism for deciding on $\varepsilon$ and the RB size $L$. If the function outputs are close to each other, large RB could be beneficial, but if the values have high variance, large RB could add unnecessary noise.

\begin{table}
\centering
\caption{The COCO experiment Hyperparameters}
\label{table:coco_hyperparams}
\begin{tabular*}{16cm}[t]{lrl}
\toprule
Parameter & Value & Description \\
\midrule
$n$ & [2,3,5,10,20,40,784] & coco space dimension \\
$m$ & 64 & Exploration points \\
$m_{warmup\_factor}$ & 5 & $m\times m_{warmup\_factor}$ to adjust the network parameters \\
& & around the TR initial point\\
\texttt{batch} & 1024 &  Minibatch of EGL/IGL training\\
$L$ & $32$ &  Number of exploration steps that constitute the replay buffer for EGL/IGL \\ 
& & The replay memory size is: $RB=L\times m $\\
$C$ & $15\times10^4$ & Budget \\
$\alpha$ & $10^{-2}$ & Optimization steps' size \\
\texttt{g\_lr} & $10^{-3}$ & $g_{\theta}$ learning rate \\
$\gamma_{\alpha}$ & $0.9$ & Trust region squeezing factor \\
$\gamma_{\varepsilon}$ & $0.97$ & $\varepsilon$ squeezing factor \\
$\varepsilon$ & $0.1 \times \sqrt{n}$ & Initial exploration size \\
$n_{\max}$ & $10$ & The number of times in a row that \\  
& & $f(h_j^{-1}(\tilde{x}_{k+1})) > f(h_j^{-1}(\tilde{x}_{k}))$ \\
$n_{\min}$ & $40$ & Minimum gradient descent iterations \\
$p$ & $0$ & Perturbation radius \\
$g_{\theta}$ & Spline & Network architecture (SPLINE/FC) \\
\texttt{OM} & log & Output Mapping \\
\texttt{OM\_lr} & 0.1 & Moving average learning rate for the Output Mapping \\
\texttt{N\_minibatches} & 60 & \# of mini-batches for the mean-gradient learning in each $k$ step \\
$V_{\varepsilon}(x)$ & ball-explore & $V_{\varepsilon}(x) = x + \varepsilon \times U[-1,1]$  (see Sec. \ref{sec:GradientGuidedExploration} for details)\\

\bottomrule
\end{tabular*}
\end{table}%

\newpage

\begin{figure*}
   \centering
   \includegraphics[width=1\textwidth]{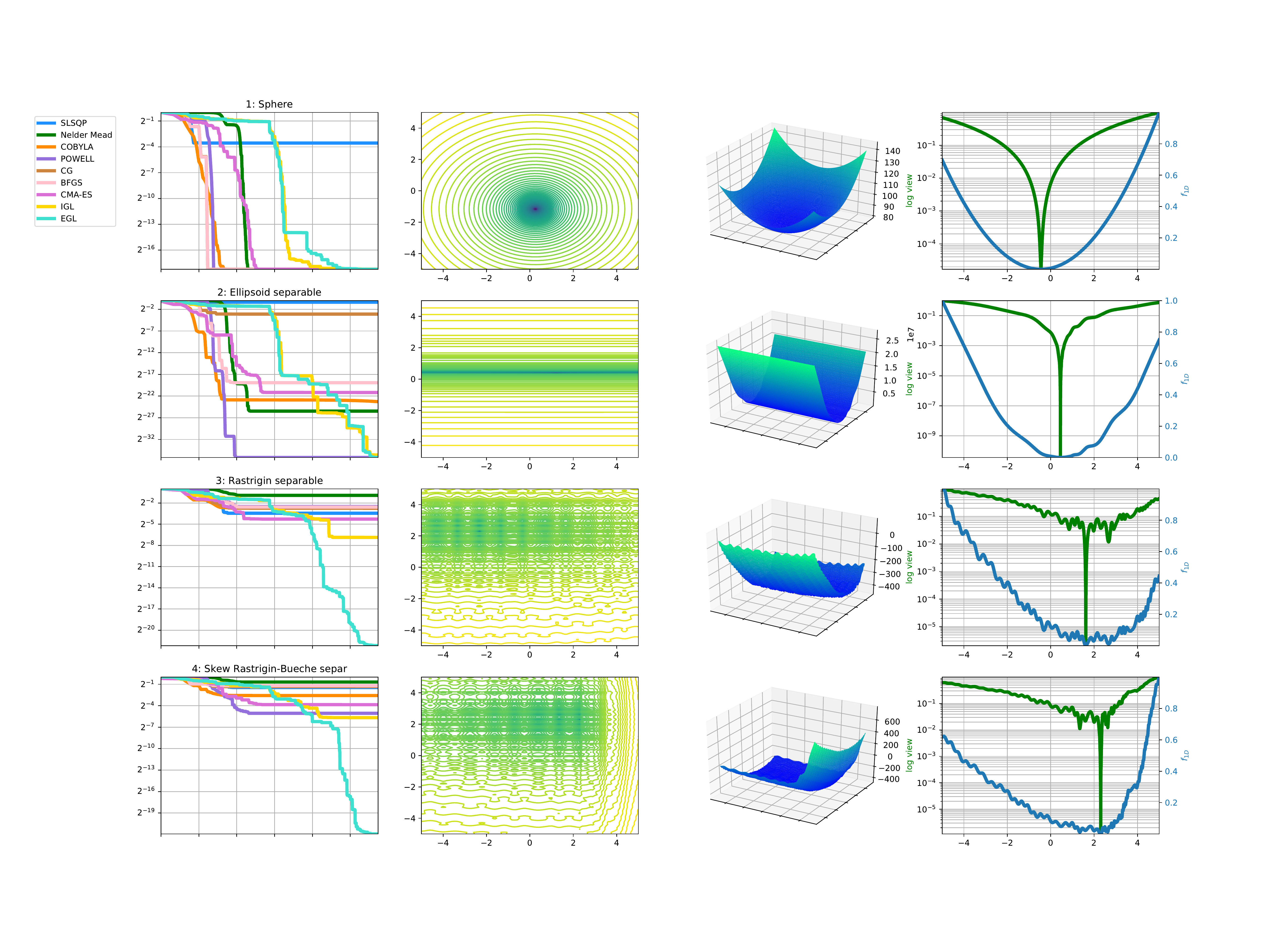}
   \caption{Visualization problems type 1-4 of 2D problems. First column: The scaled distance $\overline{\Delta y}_{best}^t$. Second column: Counter plot. Third column: 3D plot. Forth column: equivalent 1D problem with log view. }
   \label{fig:2D_P0}
\end{figure*}

\begin{figure*}
	   \centering
	   \includegraphics[width=1\textwidth]{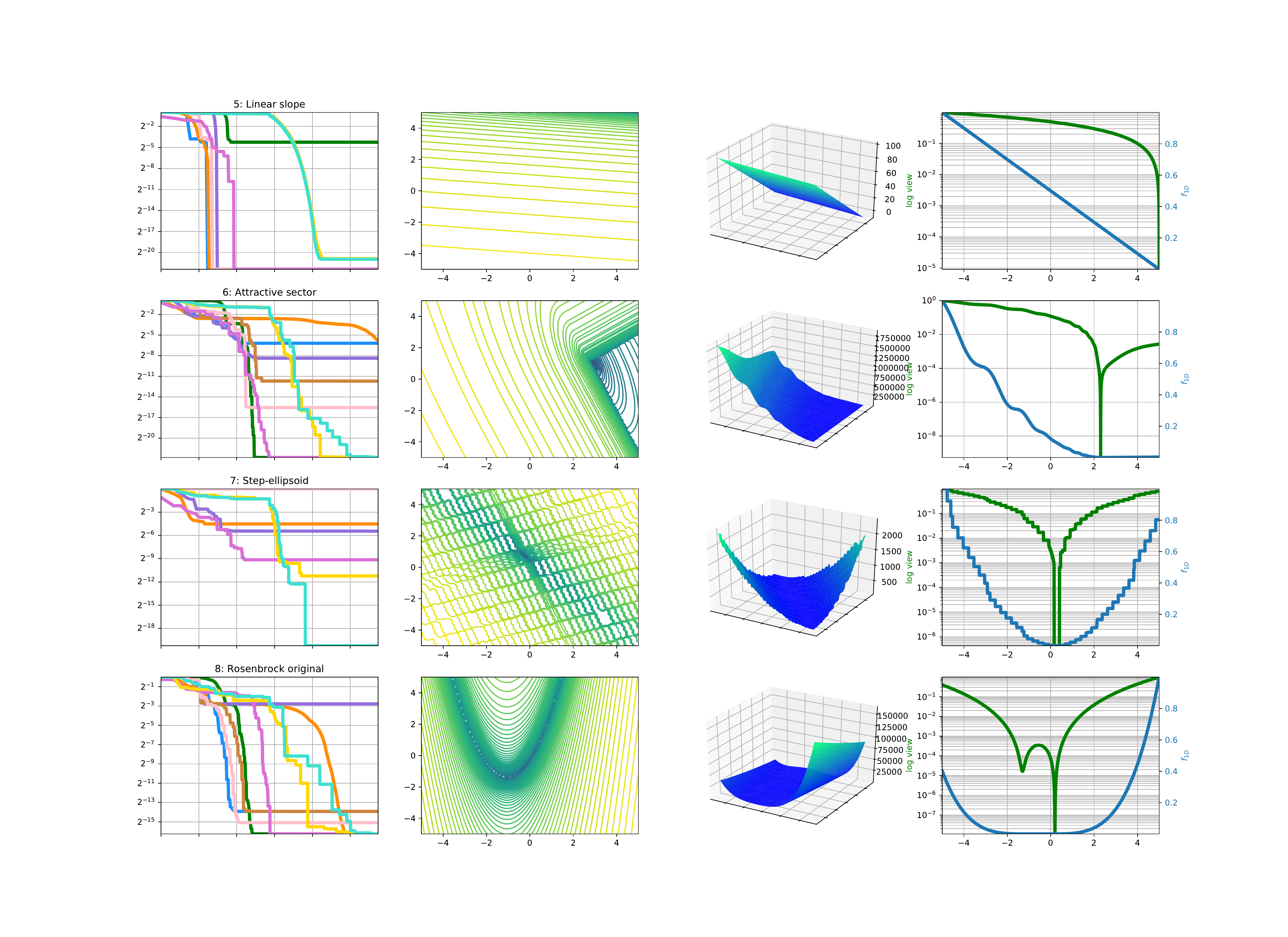}
	   \caption{Visualization problems type 5-8 of 2D problems. First column: The scaled distance $\overline{\Delta y}_{best}^t$. Second column: Counter plot. Third column: 3D plot. Forth column: equivalent 1D problem with log view. }
	   \label{fig:2D_P1}
\end{figure*}

\begin{figure*}
	   \includegraphics[width=1\textwidth]{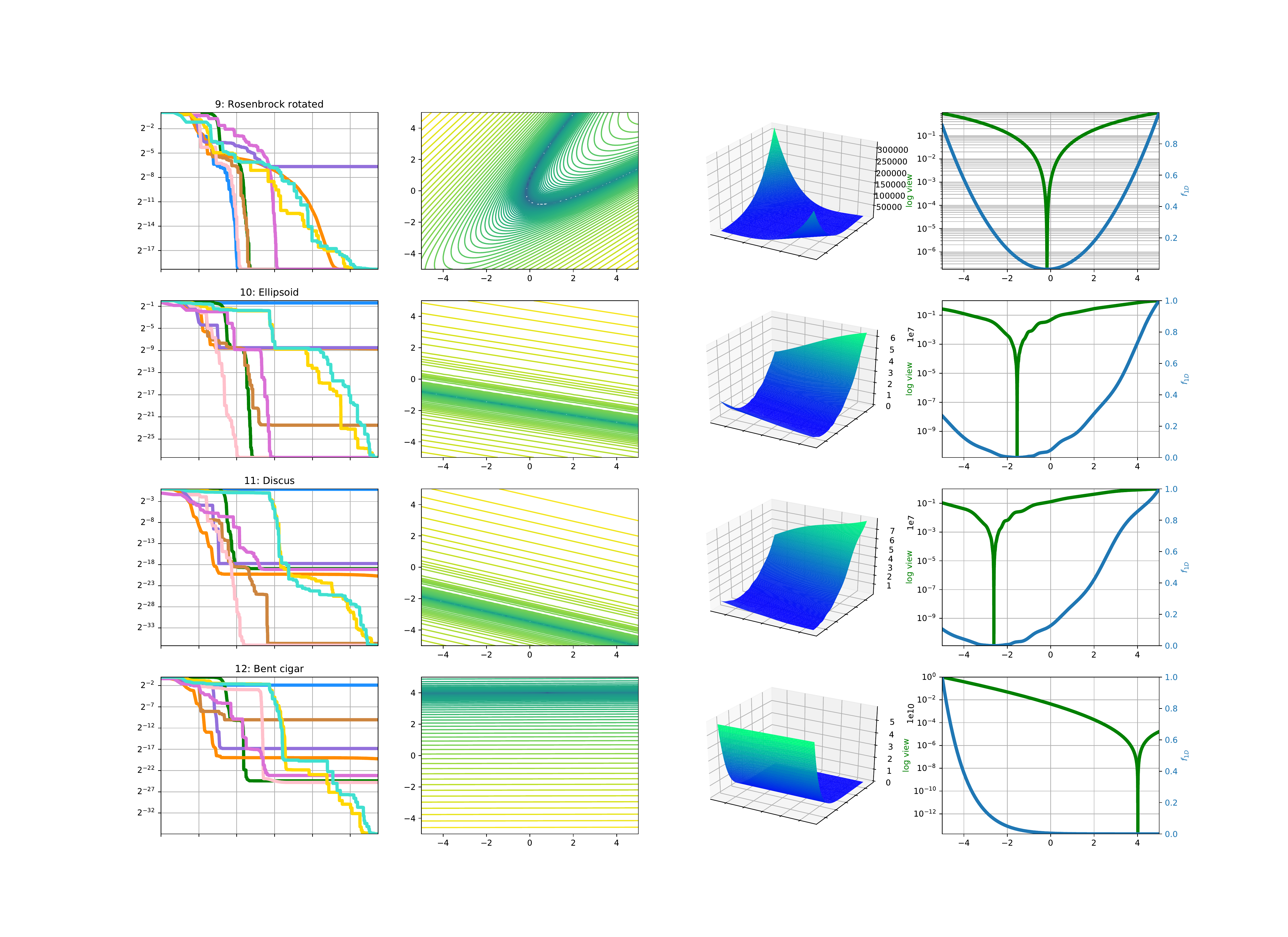}
	   \caption{Visualization problems type 9-12 of 2D problems. First column: The scaled distance $\overline{\Delta y}_{best}^t$. Second column: Counter plot. Third column: 3D plot. Forth column: equivalent 1D problem with log view. }
	   \label{fig:2D_P2}
\end{figure*}

\begin{figure*}
	   \includegraphics[width=1\textwidth]{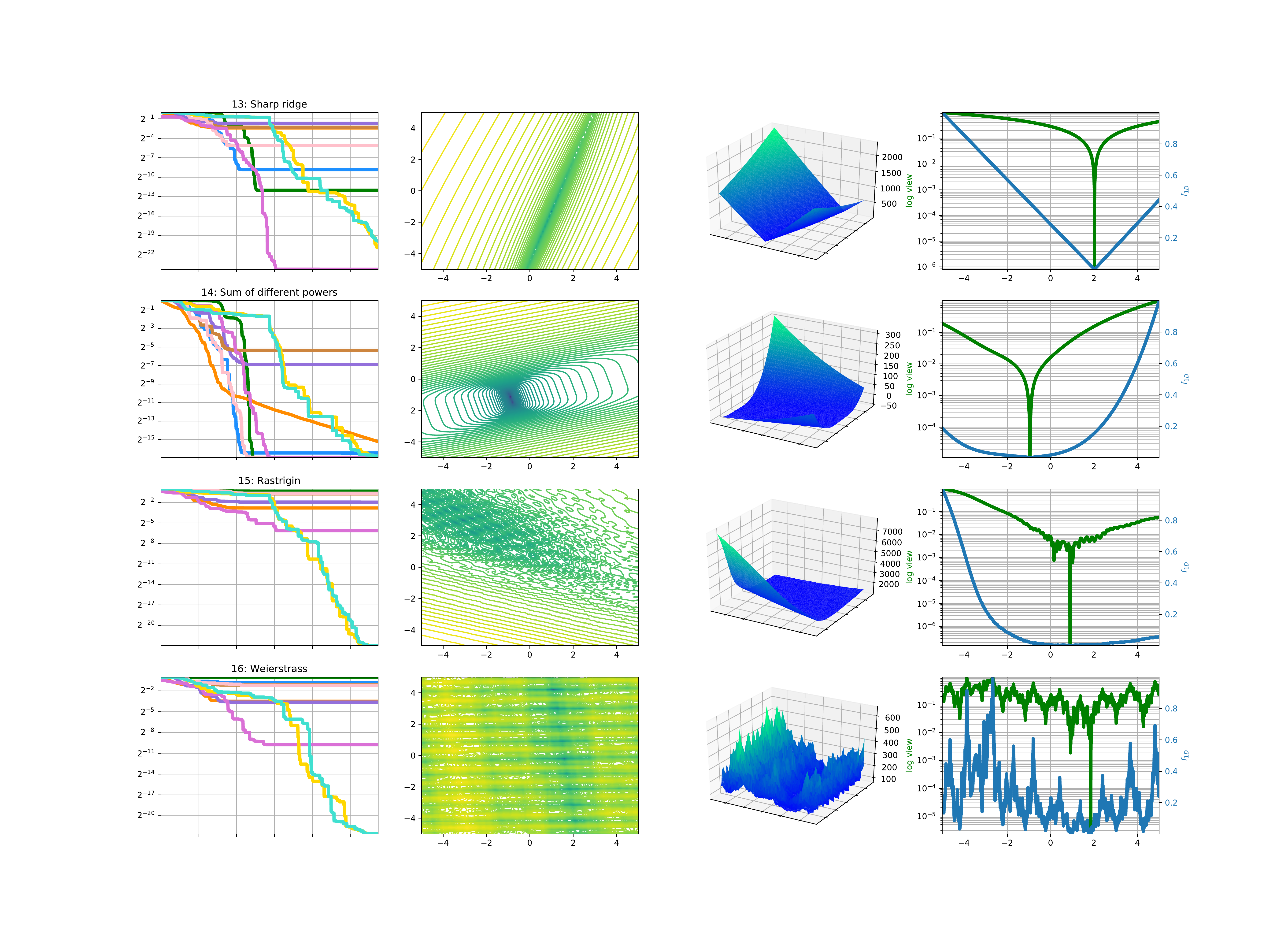}
	   \caption{Visualization problems type 13-16 of 2D problems. First column: The scaled distance $\overline{\Delta y}_{best}^t$. Second column: Counter plot. Third column: 3D plot. Forth column: equivalent 1D problem with log view. }
	   \label{fig:2D_P3}
\end{figure*}

\begin{figure*}
	  \includegraphics[width=1\textwidth]{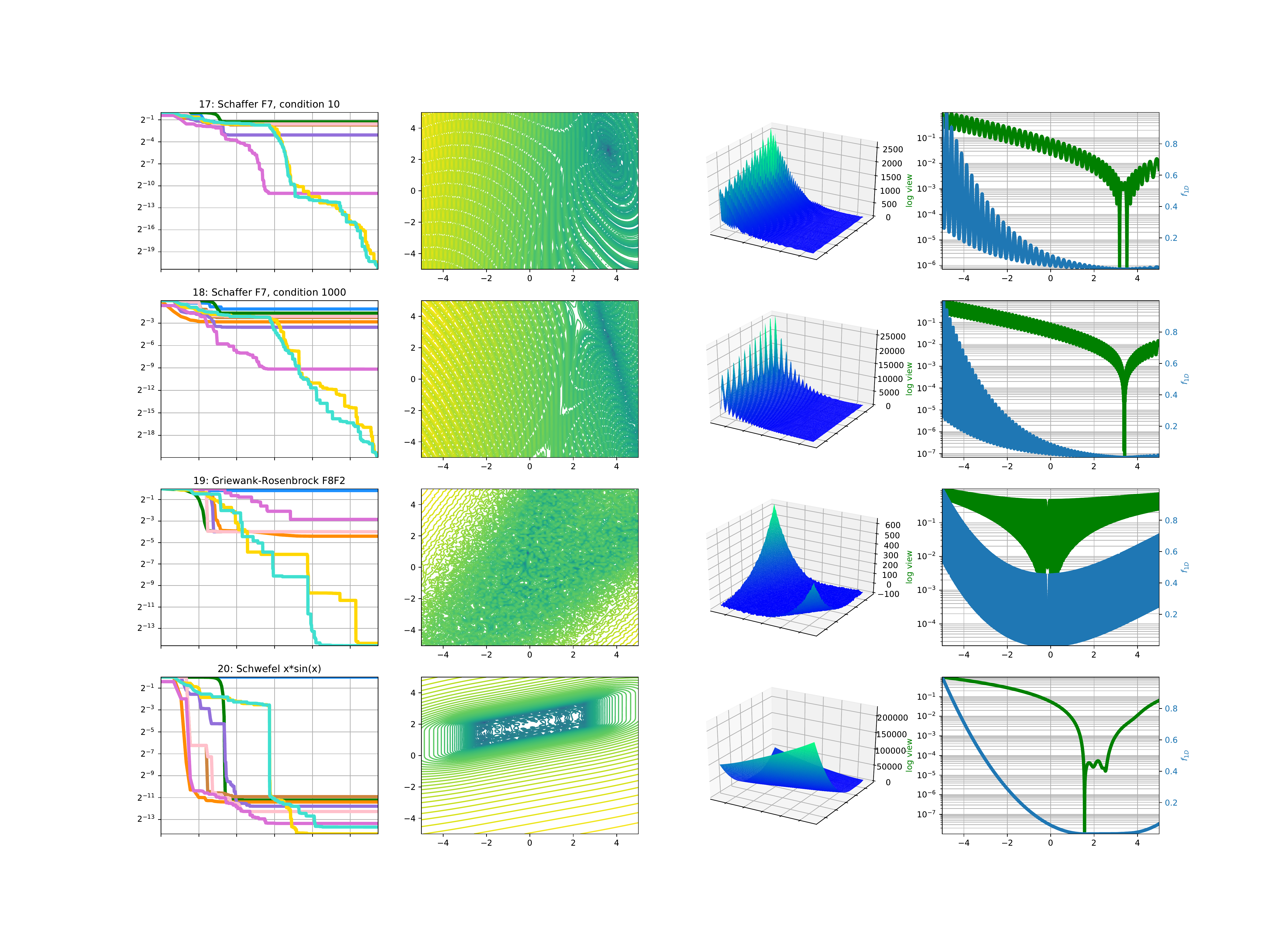}
	  \caption{Visualization problems type 17-20 of 2D problems. First column: The scaled distance $\overline{\Delta y}_{best}^t$. Second column: Counter plot. Third column: 3D plot. Forth column: equivalent 1D problem with log view. }
	  \label{fig:2D_P4}
\end{figure*}

\begin{figure*}
	   \includegraphics[width=1\textwidth]{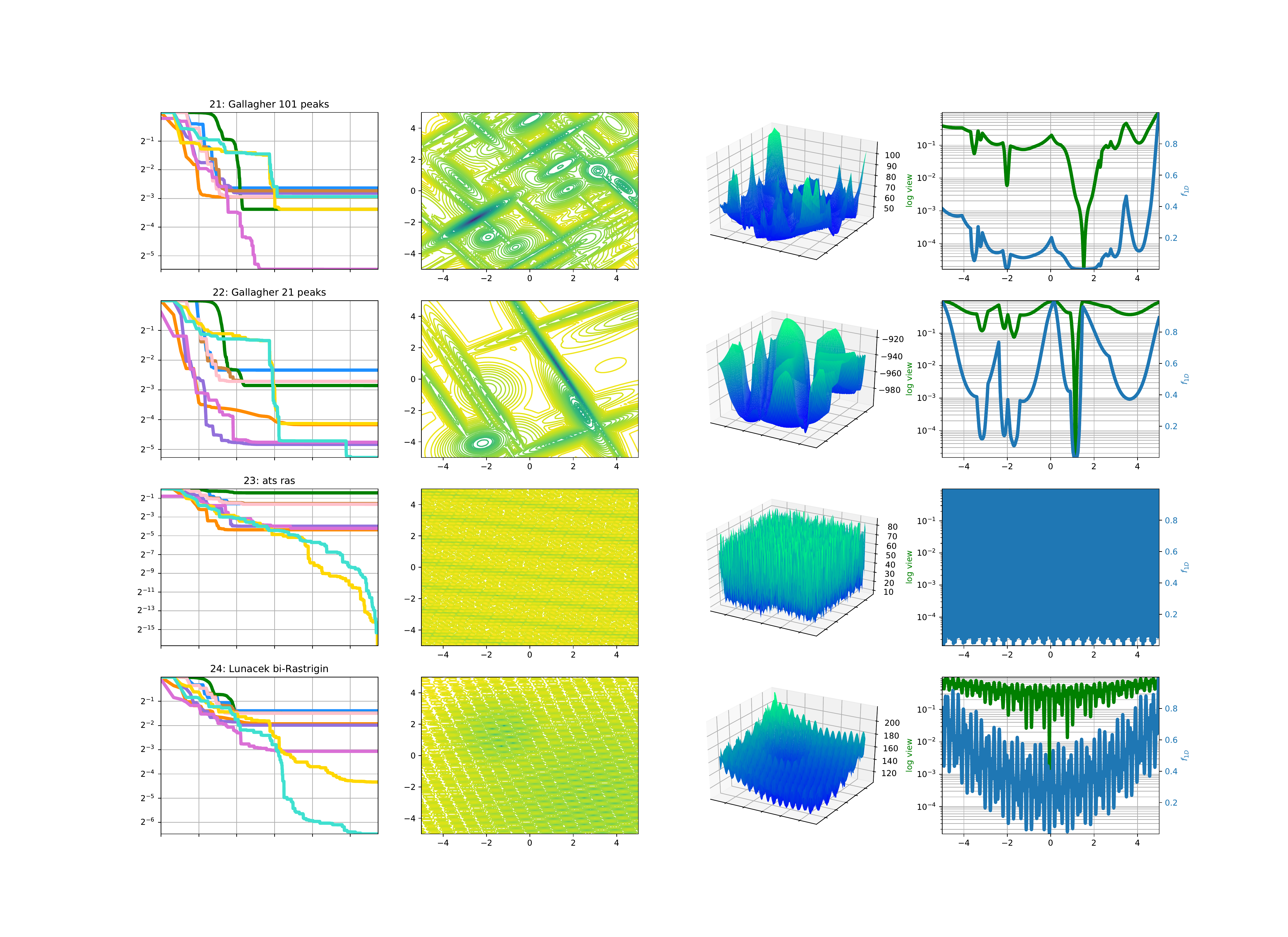}
	   \caption{Visualization problems type 21-24 of 2D problems. First column: The scaled distance $\overline{\Delta y}_{best}^t$. Second column: Counter plot. Third column: 3D plot. Forth column: equivalent 1D problem with log view. }
	   \label{fig:2D_P5}
\end{figure*}

\begin{figure*}
  \centering
  \includegraphics[width=1\linewidth]{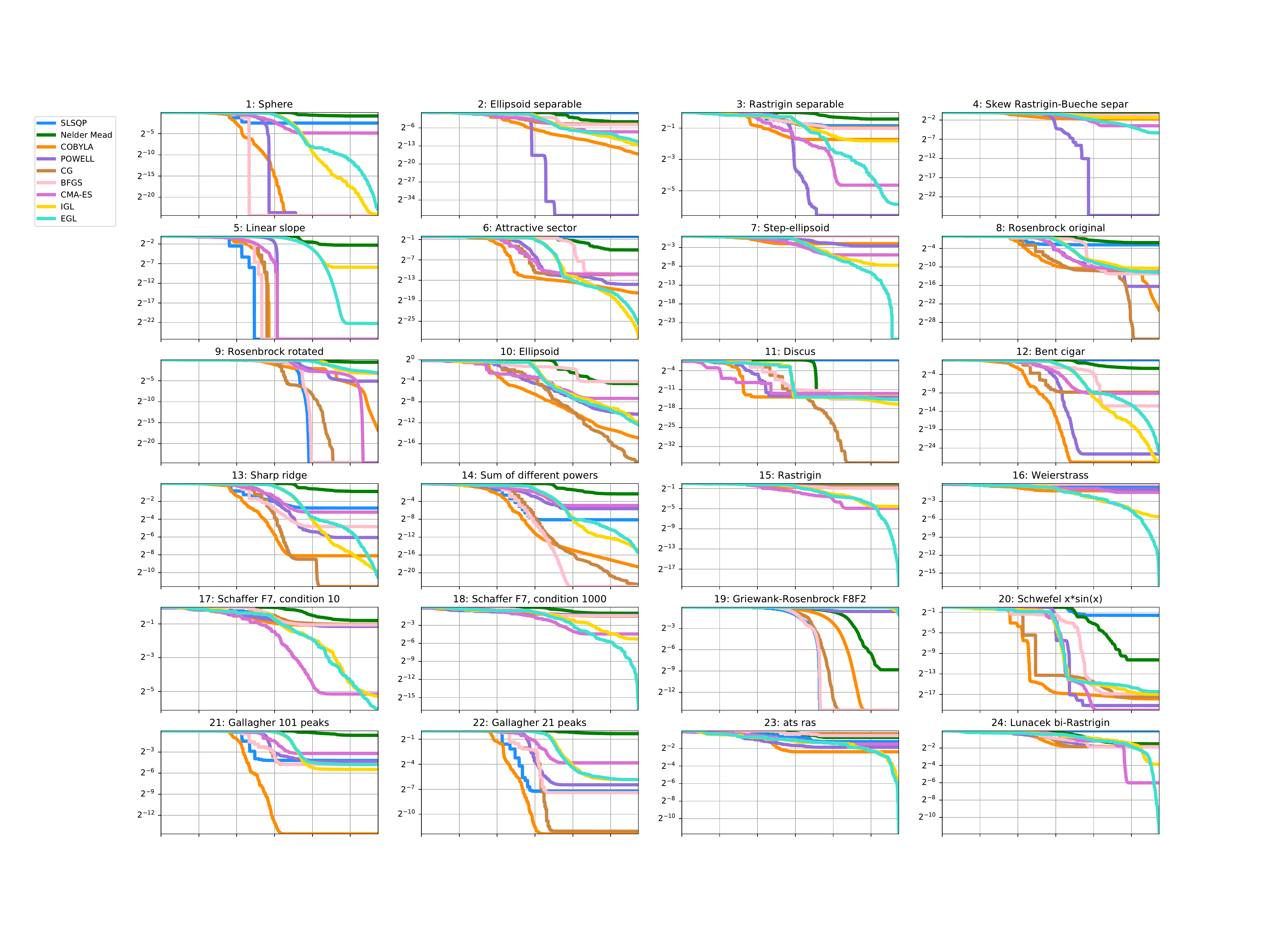}
  %\vspace{-20pt}
  \caption{The scaled distance $\overline{\Delta y}_{best}^t$ per problem type on 40D}
\label{fig:40D}
\end{figure*}

\begin{figure*}
  \centering
  \includegraphics[width=1\linewidth]{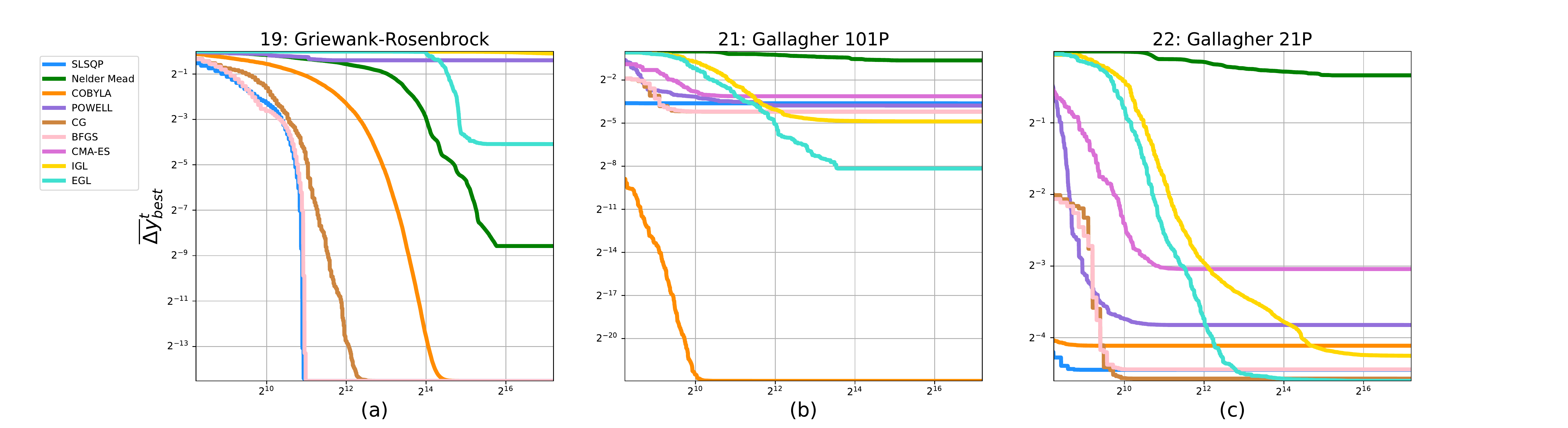}
  %\vspace{-20pt}
  \caption{The scaled distance $\overline{\Delta y}_{best}^t$ with different $\varepsilon$ on 40D. (a) problem type 19, (b) problem type 21, (c) problem type 22}
\label{fig:epsilon_40D}
\end{figure*}

\newpage

\begin{figure*}
	   \centering
	   \includegraphics[width=1\textwidth]{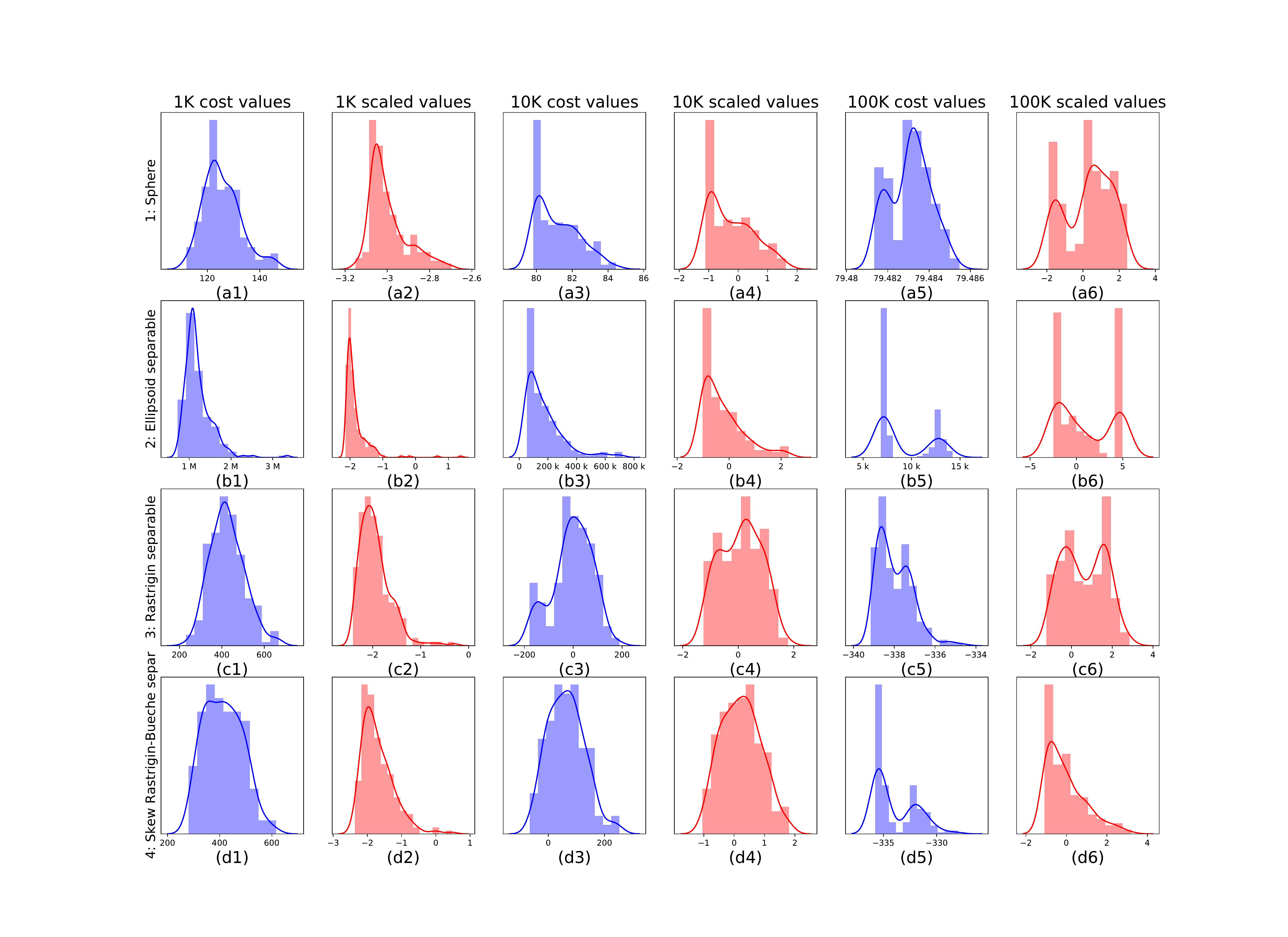}
	   \caption{Histogram plot of a snapshot of 200 samples from the RB around different times ($t=1K$, $t=10K$, $t=100K$) with and without OM for problems 1-4}
	   \label{fig:hist_P1}
\end{figure*}

\begin{figure*}
	   \includegraphics[width=1\textwidth]{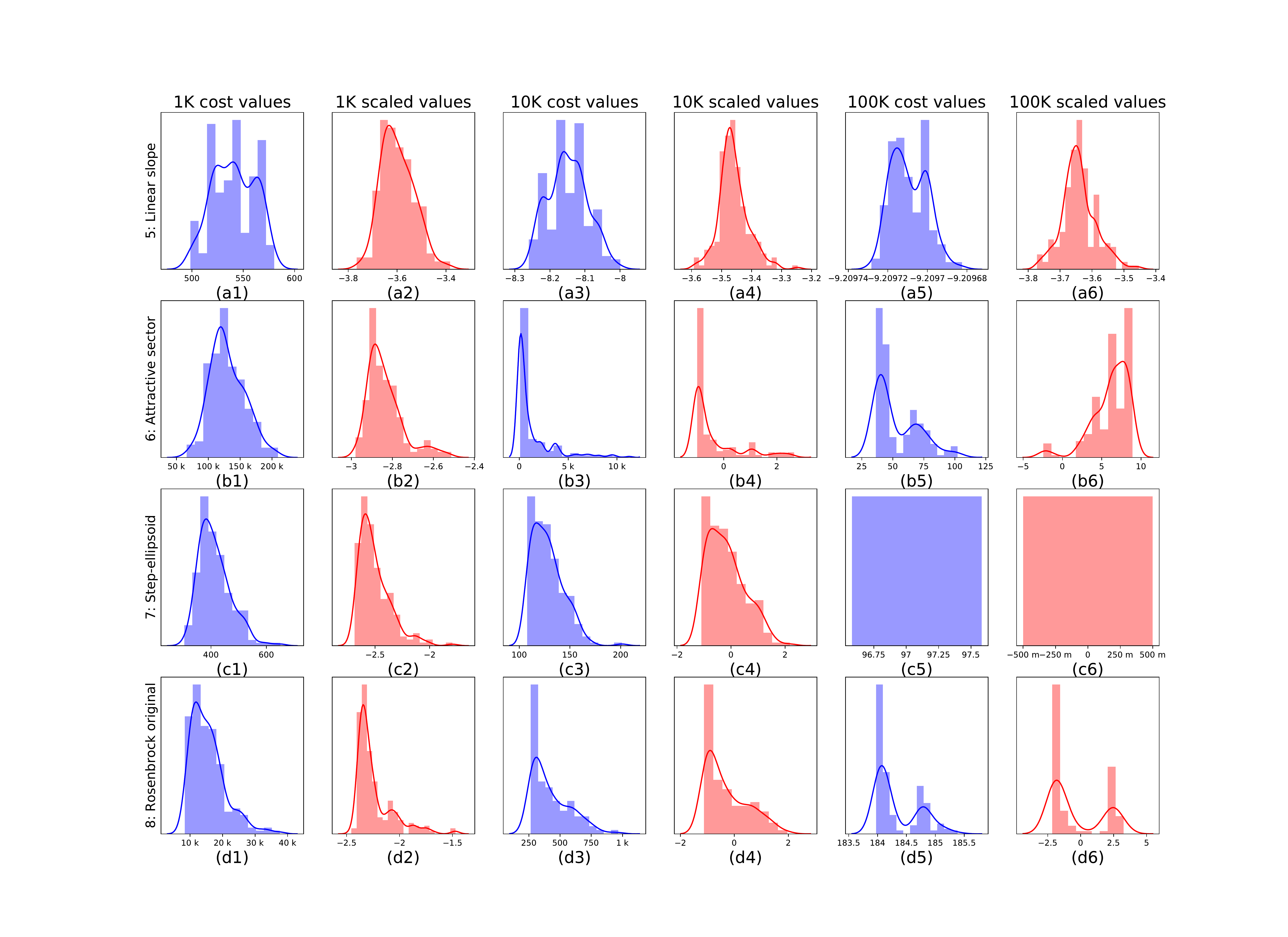}
	   \caption{Histogram plot of a snapshot of 200 samples from the RB around different times ($t=1K$, $t=10K$, $t=100K$) with and without OM for problems 5-8}
	   \label{fig:hist_P2}
\end{figure*}

\begin{figure*}
	   \includegraphics[width=1\textwidth]{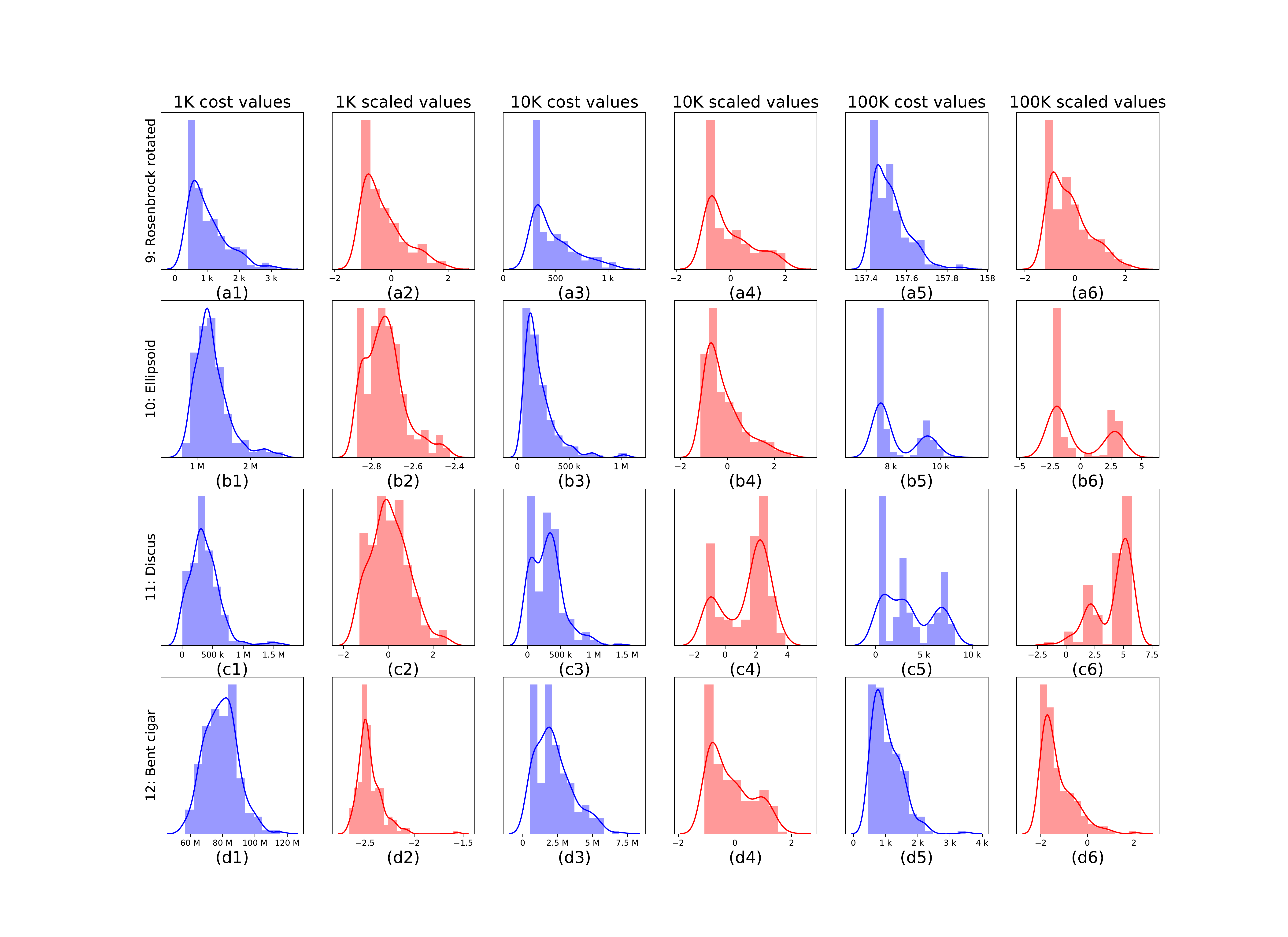}
	   \caption{Histogram plot of a snapshot of 200 samples from the RB around different times ($t=1K$, $t=10K$, $t=100K$) with and without OM for problems 9-12}
	   \label{fig:hist_P3}
\end{figure*}

\begin{figure*}
	  \includegraphics[width=1\textwidth]{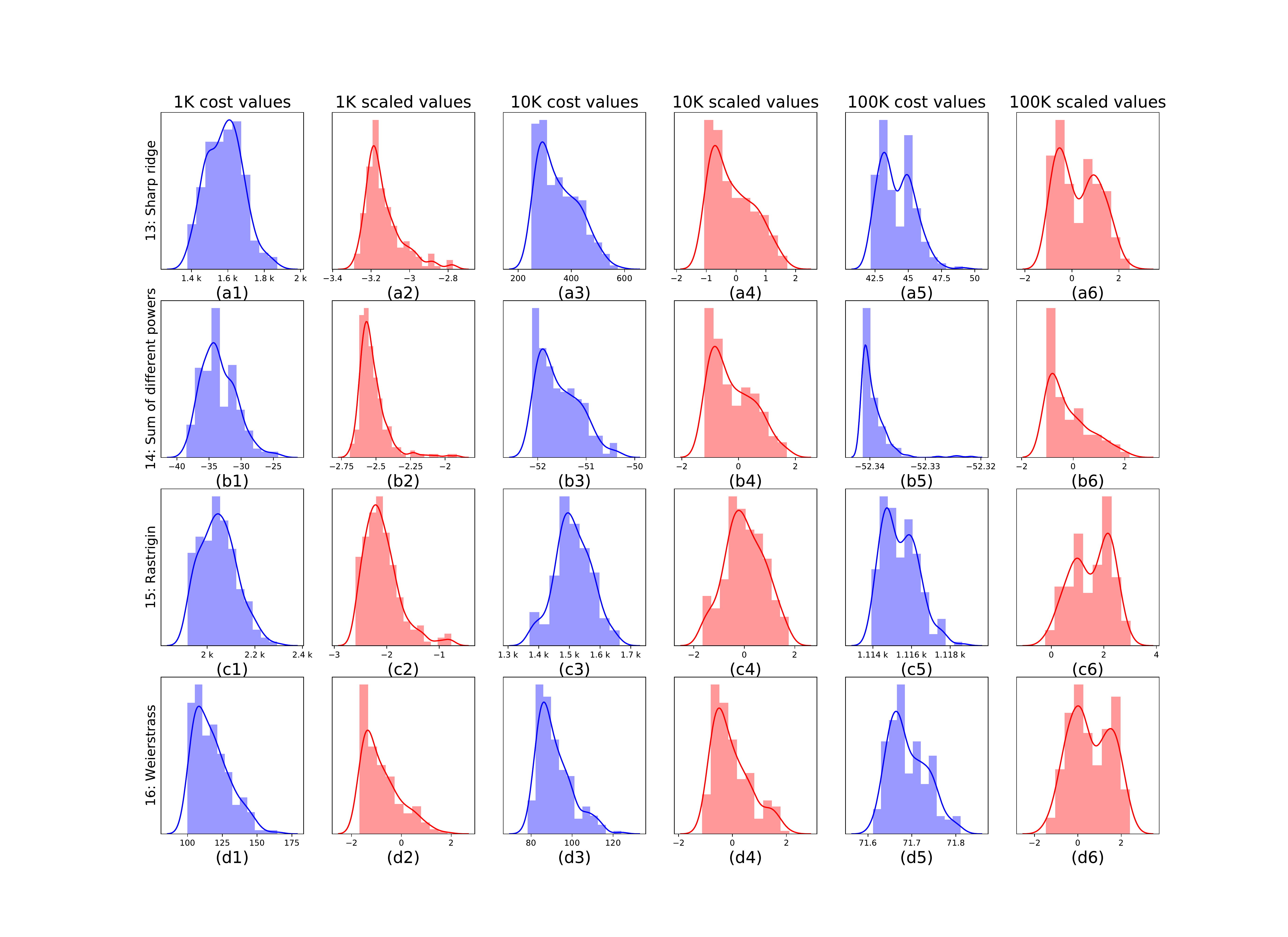}
	  \caption{Histogram plot of a snapshot of 200 samples from the RB around different times ($t=1K$, $t=10K$, $t=100K$) with and without OM for problems 13-16}
	  \label{fig:hist_P4}
\end{figure*}

\begin{figure*}
	   \includegraphics[width=1\textwidth]{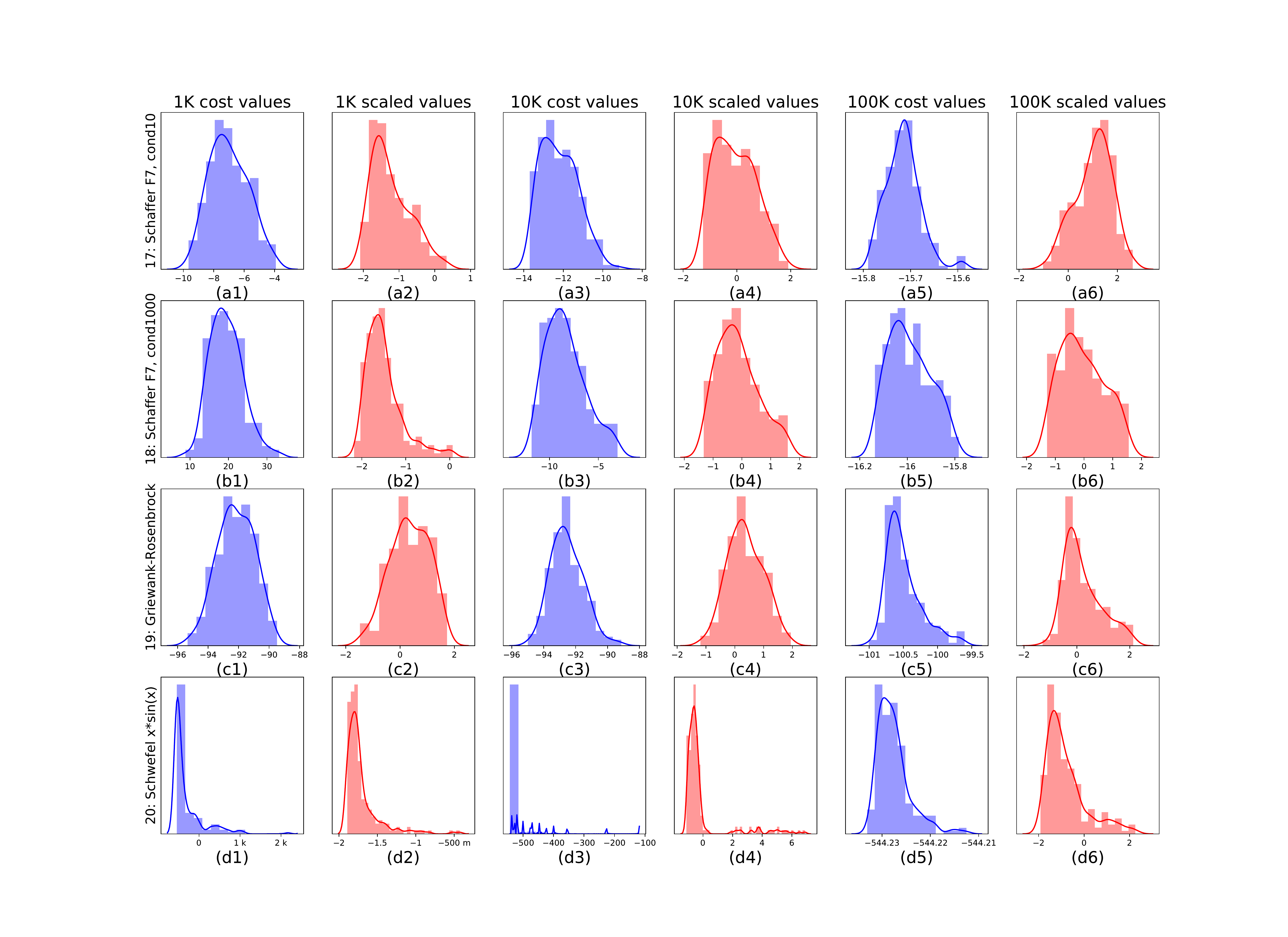}
	   \caption{Histogram plot of a snapshot of 200 samples from the RB around different times ($t=1K$, $t=10K$, $t=100K$) with and without OM for problems 17-20}
	   \label{fig:hist_P5}
\end{figure*}

\begin{figure*}
	   \includegraphics[width=1\textwidth]{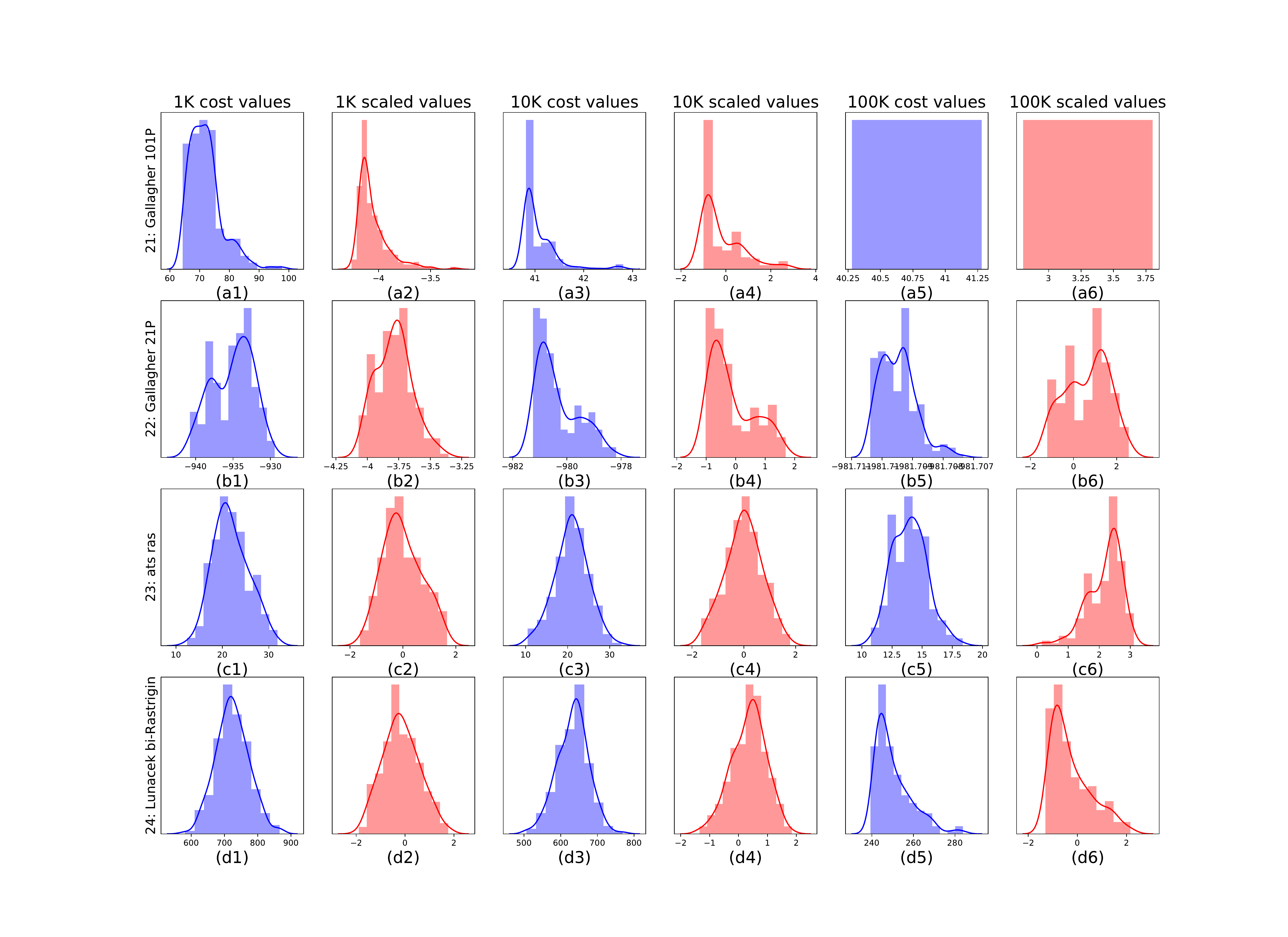}
	   \caption{Histogram plot of a snapshot of 200 samples from the RB around different times ($t=1K$, $t=10K$, $t=100K$) with and without OM for problems 21-24}
	   \label{fig:hist_P6}
\end{figure*}
\clearpage
\newpage

\section{Supplementary details: latent space search}
\label{sec:latent_space_search}

\begin{figure*}[ht!]
  \centering
  \includegraphics[width=.7\linewidth]{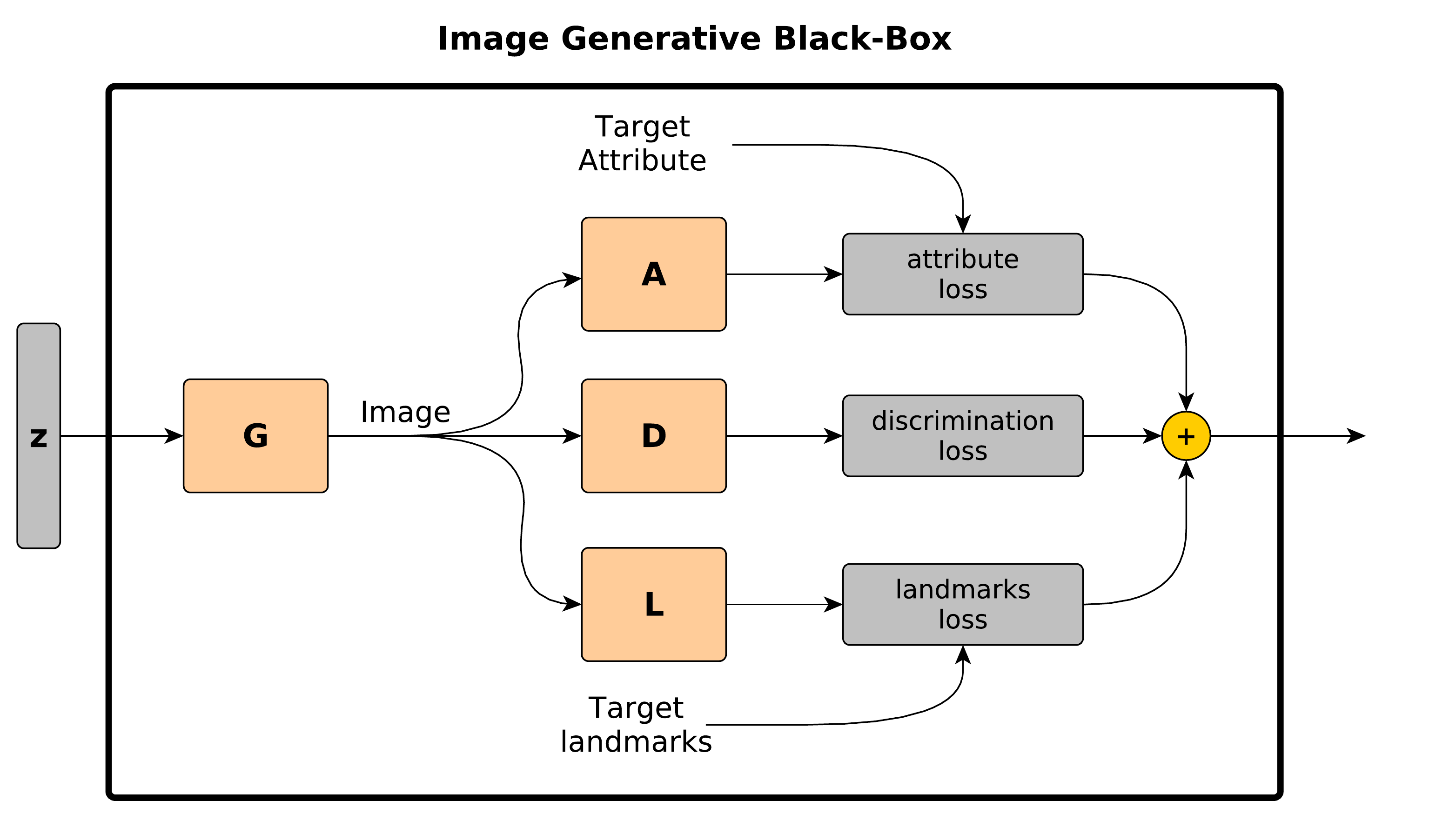}
  %\vspace{-20pt}
  \caption{The image-generative BBO task}
\label{fig:latent_graph}
\end{figure*}

In this experiment, the task is to utilize a pre-trained Black-Box face image generator and generate a realistic face image with target face attributes and face landmark points. Formally, we have 4 Black-Box networks that constitute our BBO problem:
\begin{enumerate}
    \item Generator $G:z\to x$, where $z\sim \mathcal{N}(0, \mathbf{I}_{n})$ and $x$ is an RGB image with $H=218$, $W=178$.
    \item Discriminator $D:x\to \mathbb{R}$ s.t. positive $D(x)$ indicate poor fake images while negative $D(x)$ indicates real or a good fake image.
    \item Attribute Classifier $A:x\to \mathbb{R}^{40}$ where each element in $A(x)$ is the probability of a single attribute (out of 40 different attributes).
    \item Landmark points Estimator $L:x\to \mathbb{R}^{68}$ predicts the location of 68 different landmark points.
\end{enumerate}

In addition, every BBO problem is characterized by two external parameters
\begin{enumerate}
    \item Target attributes $a$ a vector of 40 Booleans.
    \item Target landmark points $l\in\mathbb{R}^{68}$ a vector of 68 landmark points locations.
\end{enumerate}

The overall cost function is defined as 
\begin{equation*}
    f_{al}(z) = \lambda_a \mathcal{L}_a(G(z)) + \lambda_l \mathcal{L}_l(G(z)) + \lambda_g  \tanh(D(G(z)))
\end{equation*}
Where: 
\begin{enumerate}
    \item  $\mathcal{L}_a$ is the Cross-Entropy loss between the generated face attributes as measured by the classifier and the desired set of attributes $a$.
    \item $\mathcal{L}_l$ is the MSE between the generated landmark points and the desired set of landmarks $l$.
    \item $D(G(z))$ is the discriminator output, positive for low-quality images and negative for high-quality images.
\end{enumerate}

The objective is to find $z^*$ that minimizes $f_{al}$. A graphical description of the BBO problem is given in Fig.
\ref{fig:latent_graph}. We used a constant starting point $z_0$, sampled from $\mathcal{N}(0, \mathbf{I}_{n})$ for all our runs. To generate different problems, we sampled a target image $x_T$ from the CelebA dataset. We used its attributes as the target attribute and used its landmark points estimation $l=L(x_T)$ as the target landmarks. Note that the target image $x_T$ was never revealed to the EGL optimizer, only its attributes and landmark points.

We applied the same EGL algorithm as in the COCO experiment with the Spline Embedding network architecture and with two notable changes: (1) a set of slightly modified hyperparameters (See Table \ref{table:latent_hyperparams}); and (2) a modified exploration domain $V_{\varepsilon}$. The main reason for these adjustments was the high computational cost (comparing to the COCO experiment) of each different $z$ vector. This led us to squeeze the budget $C$ to only $10^4$ evaluations and the number of exploration points to only $m=32$. For such a low number of exploration points around each candidate (in a $n=512$ dimension space), we designed an exploration domain $V_{\varepsilon}$, termed {\em cone-explore} which we found out to be more efficient than the uniform exploration inside an $n$-ball with $\varepsilon$ radius that was executed in the COCO experiment. A description of the cone-explore method is given in Sec. \ref{sec:GradientGuidedExploration}.

Additional results of EGL and IGL are given in Fig. \ref{fig:latent_res_2} and Fig. \ref{fig:latent_res_3}. We also evaluated two classical algorithms: GC and CMA-ES, both provided unsatisfying results (see Fig. \ref{fig:latent_cg_cma}). We observed that CG converged to local minima around the initial point $z_0$ while CMA-ES converged to points that have close landmark point but poor discrimination score. We did not investigate into this phenomena but we postulate that it is the result of different landscapes statistics of the two factors in the cost function, i.e. $\mathcal{L}_l(G(z))$ and $\tanh(D(G(z)))$ which fool the CMA-ES algorithm.

\begin{table}
\centering
\caption{Latent-Space Search Hyperparameters}
\label{table:latent_hyperparams}
\begin{tabular*}{16cm}[t]{lrl}
\toprule
Parameter & Value & Description \\
\midrule
$n$ & 512 & Latent space dimension \\
$\lambda_a$ & 1 & Attributes score weight \\
$\lambda_d$ & 2 & Discriminator score weight \\
$\lambda_l$ & 100 & Landmarks score weight \\
$m$ & 32 & Exploration points \\
\texttt{batch} & 1024 &  Minibatch of EGL/IGL training\\
$L$ & $64$ &  Number of exploration steps that constitute the replay buffer for EGL/IGL \\ 
& & The replay memory size is: $RB=L\times m $\\
$C$ & $10^4$ & Budget \\
$\phi$ & $\frac{2}{3}\pi$ & Cone exploration angle \\
$\alpha$ & $0.02$ & Optimization steps' size \\
\texttt{g\_lr} & $10^{-3}$ & $g_{\theta}$ learning rate \\
$\gamma_{\alpha}$ & $0.9$ & Trust region squeezing factor \\
$\gamma_{\varepsilon}$ & $0.95$ & $\varepsilon$ squeezing factor \\
\bottomrule
\end{tabular*}
\end{table}%

\begin{figure*}
  \centering
  \includegraphics[width=.7\linewidth]{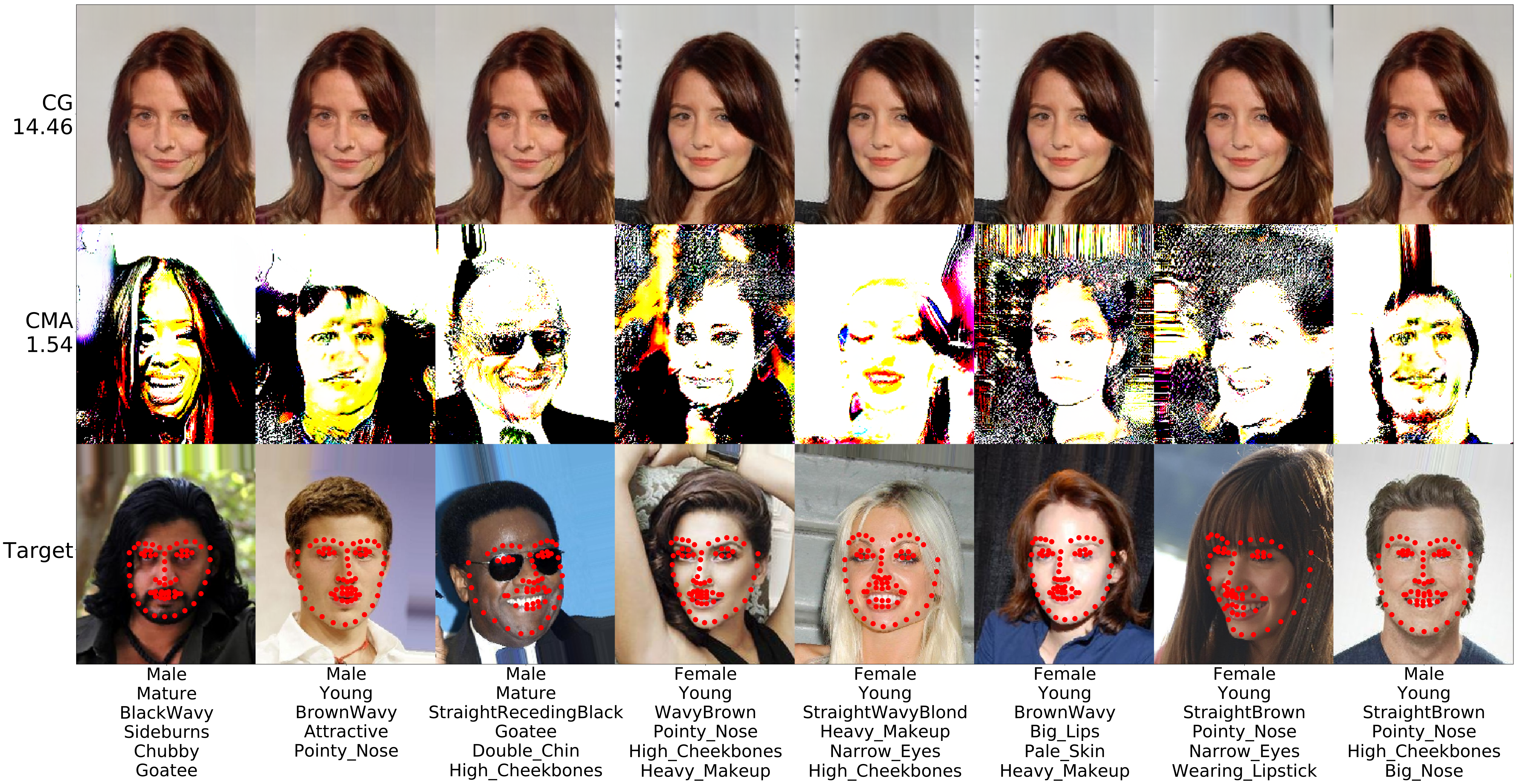}
  %\vspace{-20pt}
  \caption{Baselines results: CG and CMA}
\label{fig:latent_cg_cma}
\end{figure*}

\ifx\hidepics\undefined

\begin{figure*}
  \centering
  \includegraphics[width=1\linewidth]{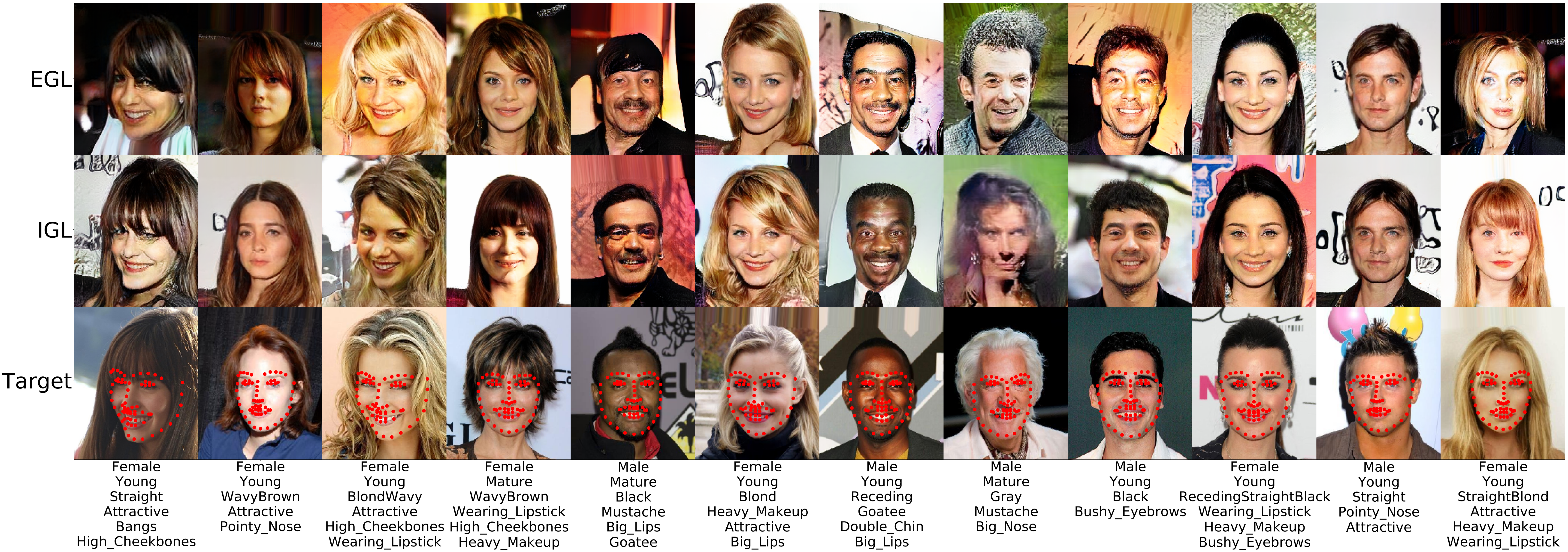}
  %\vspace{-20pt}
  \caption{Searching latent space of generative models with EGL and IGL}
\label{fig:latent_res_2}
\end{figure*}

\else

\begin{figure*}
  \centering
  \includegraphics[width=1\linewidth]{example-image-a}
  %\vspace{-20pt}
  \caption{Searching latent space of generative models with EGL and IGL}
\end{figure*}

\fi

\ifx\hidepics\undefined

\begin{figure*}
  \centering
  \includegraphics[width=1\linewidth]{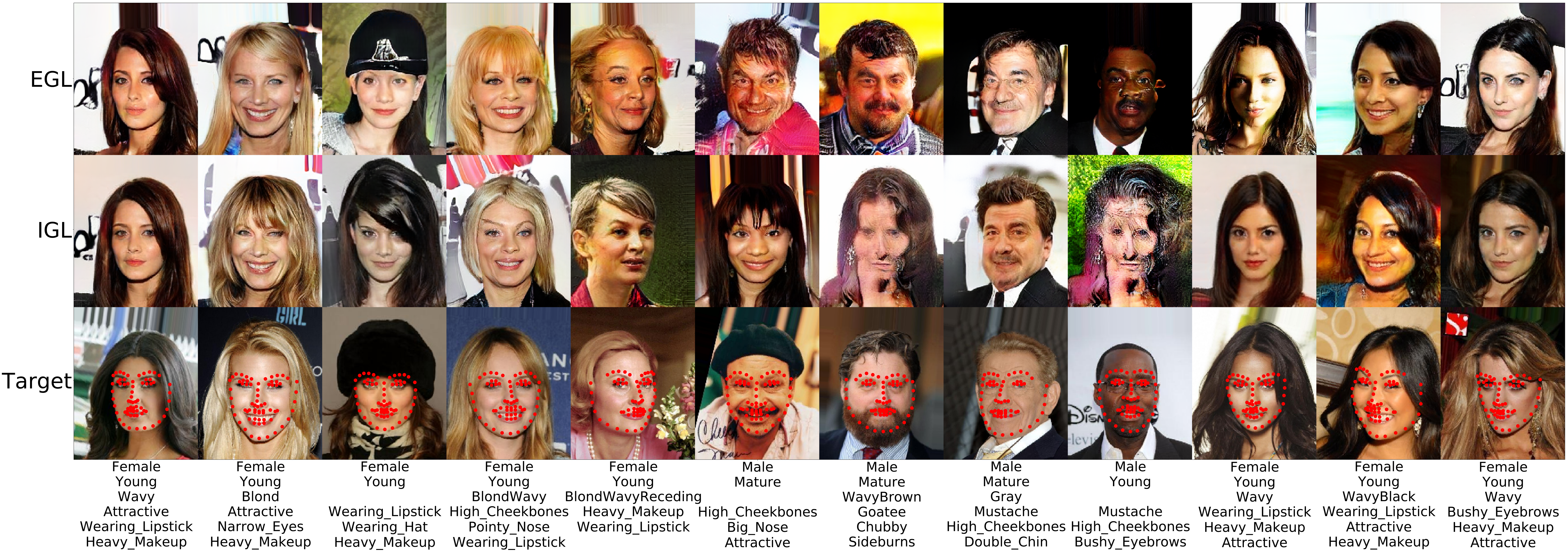}
  %\vspace{-20pt}
  \caption{Searching latent space of generative models with EGL and IGL}
\label{fig:latent_res_3}
\end{figure*}

\else

\begin{figure*}
  \centering
  \includegraphics[width=1\linewidth]{example-image-a}
  %\vspace{-20pt}
  \caption{Searching latent space of generative models with EGL and IGL}
\end{figure*}

\fi

\newpage

\section{Gradient Guided Exploration}\label{sec:GradientGuidedExploration}

In high-dimensional problems and low budgets, sampling $n+1$ exploration points for each new candidate $x_k$ may consume the entire budget too fast without being able to take enough optimization steps. In practice,  EGL works even with $m \ll n+1$ exploration points, however, we observed that one can improve the efficiency when $m \ll n+1$ by sampling the exploration points non uniformly around the candidate $x_k$. Specifically, using the previous estimation of the gradient to determine the search direction. Hence, we term this approach as {\em gradient guided exploration}.

In the COCO experiment (Sec. \ref{sec:Experiment}) we sampled the exploration points uniformly around each candidate. In other words, our $V_{\varepsilon}$ domain was an $n$-ball with $\varepsilon$ radius, we term this method as {\em ball-explore}. Ball-explore does not make any assumptions on the gradient direction at $x_k$ and does not use any a-priori information. However, for continuous gradients, we do have a-priori information on the gradient direction as we have our previous estimator $g_{\theta_{k-1}}$. Since $x_k$ is relatively close to $x_{k-1}$, the learned model $g_{\theta_{k-1}}$ can be used as a first-order approximation for the gradient in $x_k$, i.e. $g_{\theta_{k-1}}(x_k)$. 

If $g_{\theta_{k-1}}(x_k)$ is a good approximation for $g_{\varepsilon}(x_k)$, then sampling points in a perpendicular direction to the mean-gradient, i.e. $x$ s.t. $(x-x_k) \cdot g_{\theta_{k-1}}(x_k) = 0$, adds little information since $f(x) - f(x_k) \approx 0$ so the loss $|(x-x_k) \cdot g_{\theta_{k-1}}(x_k) - f(x) + f(x_k)|^2$ is very small. Therefore, we experimented with sampling points inside the intersection of a $n$-ball $B_{\varepsilon}(x_k)$ and a cone with apex at $x_k$, direction $-g_{\theta_{k-1}}(x_k)$ and some hyperparameter cone-angle of $\phi$. The distance vector $x-x_k$ of a point inside such a cone has high cosine similarity with the mean-gradient and we postulate that this should improve the efficiency of the learning process. We term this alternative exploration method as {\em cone-explore} and denote the cone domain as $C_{\varepsilon}^{\phi}(x,g_{\theta_{k-1}}(x_k))$.

For high dimensions, cone-explore significantly reduces the exploration volume. A simple upper bound for the cone-to-ball volume ratio show that it decays exponentially in $n$
\begin{equation}
    \frac{|C_{\varepsilon}^{\phi}|}{|B_{\varepsilon}|} \leq \frac{\sqrt{\pi}\Gamma(\frac{n+1}{2})}{n\Gamma(\frac{n}{2}+1)} \left(\sin\phi\right)^{n-1}
\end{equation}
Unfortunately, cone-explore is not suitable for non-continuous gradients or too large optimization steps. To take into account gradient discontinuities, we suggest to sample half of the points inside the cone and half inside an $n$-ball. 

In Fig. \ref{fig:cone_vs_ball} we present an ablation test in the 784D COCO problem set of cone-explore and $\frac{1}{2}$-cone-$\frac{1}{2}$-ball explore with respect to the standard ball-explore. In this experiment, we used only $m=32$ exploration points around each candidate. The results show that sampling half of the exploration points inside a cone improved the results by 18\%. We found out that the strategy also improves the results in the latent space search experiment, yet we did not conduct a full ablation test.

On the downside, we found out that if $m\approx n$ then cone-explore hurts the performance. We hypothesize that near local minima, where the exact direction of the gradient is important, the mean-gradient learned with cone-explore has lower accuracy and therefore, ball-explore with sufficient sampling points converges to better solutions.

\begin{figure*}[hb!]
  \centering
  \includegraphics[width=0.5\linewidth]{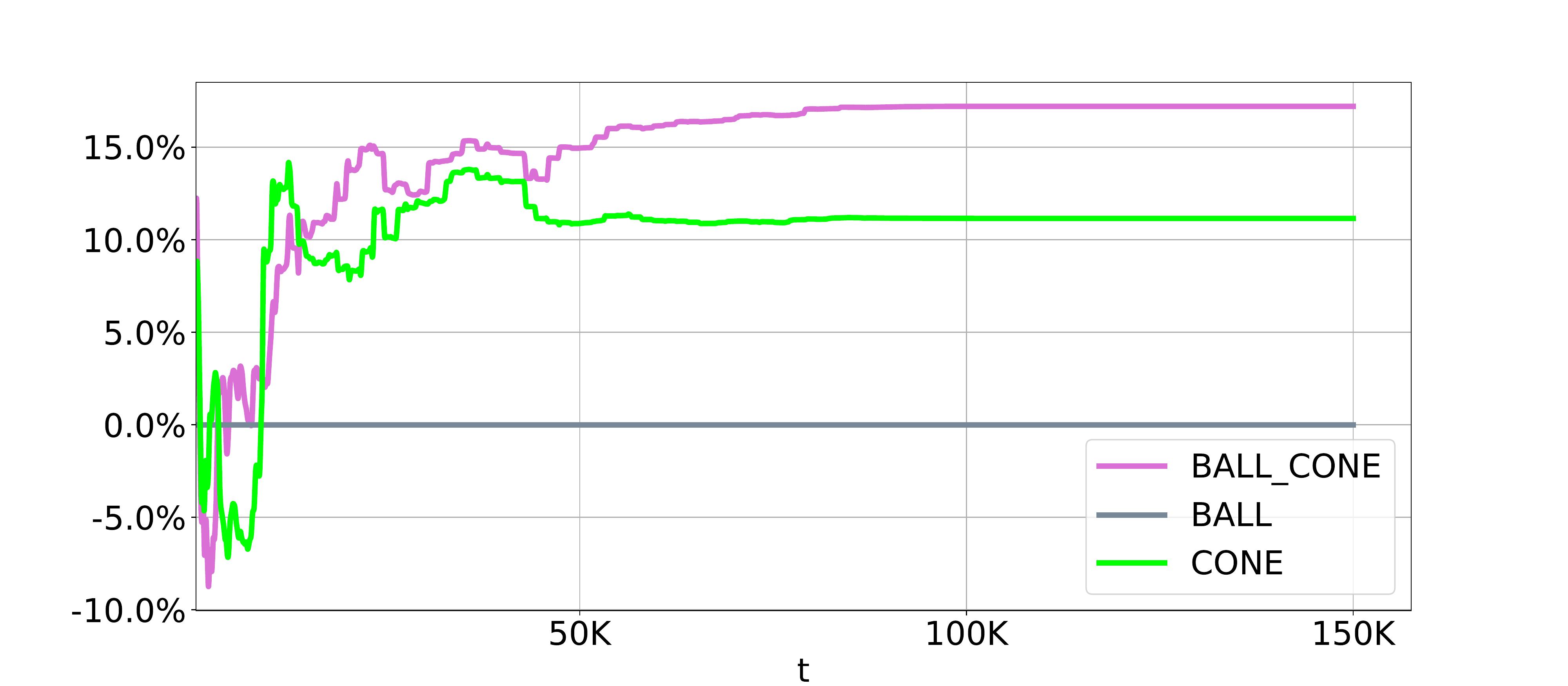}
  \caption{Ablation test on the 784D COCO problem set: Cone-explore vs Ball-explore}
\label{fig:cone_vs_ball}
\end{figure*}

\newpage

\bibliography{mybib}
\bibliographystyle{icml2020}

\end{document}